\documentclass{article}
\usepackage[final,main]{neurips_2026}
\usepackage[utf8]{inputenc}
\usepackage[T1]{fontenc}
\usepackage{float}
\usepackage{booktabs}
\usepackage{longtable}
\usepackage{graphicx}
\usepackage{hyperref}
\hypersetup{colorlinks=true, linkcolor=black, citecolor=black, urlcolor=black}
\usepackage{url}
\DeclareUrlCommand\bpath{\urlstyle{tt}}
\usepackage{amsfonts}
\usepackage{amsmath}
\usepackage{amssymb}
\usepackage{amsthm}
\usepackage{enumitem}
\usepackage{nicefrac}
\usepackage{microtype}
\usepackage{xcolor}
\usepackage{algorithm}
\usepackage{algorithmic}
\usepackage{listings}
\lstdefinestyle{pypseudo}{
  language=Python,
  basicstyle=\ttfamily\footnotesize,
  keywordstyle=\bfseries,
  commentstyle=\itshape\color{gray!70!black},
  stringstyle=\color{black},
  showstringspaces=false,
  columns=fullflexible,
  keepspaces=true,
  breaklines=true,
  breakatwhitespace=true,
  xleftmargin=1em,
  morekeywords={with,as},
}

\newtheorem{proposition}{Proposition}
\newtheorem{definition}{Definition}

\newcommand{\textpkg}{\texttt{ssdiff}}
\newcommand{\inferpkg}{\texttt{plskit}}
\newcommand{\repoanchor}{the replication record~\citep{lenartowicz2026replication}}

\title{Cheap and Powerful Tests for Supervised Subspaces: Per-Component Inference for PLS}

\author{%
  Pawe{\l} Lenartowicz\thanks{ORCID: 0000-0002-6906-7217. \texttt{pawellenartowicz@europe.com}} \\
  Society for Open Science (Stowarzyszenie na Rzecz Otwartej Nauki) \\
  Centre for Brain Research, Jagiellonian University \\
  \And
  Hubert Plisiecki\thanks{ORCID: 0000-0002-5273-1716} \\
  Society for Open Science (Stowarzyszenie na Rzecz Otwartej Nauki) \\
  IDEAS Research Institute \\
}

\begin{document}

\maketitle

\begin{abstract}
Partial Least Squares (PLS) regression extracts a few outcome-aligned
directions in a high-dimensional $\mathbf{X}$ and is widely used across
applied science, but inference on the resulting fit is either
expensive, biased and discouraged, or absent. We reduce inference to
held-out OLS refits of the supervised subspace, a primitive shared by
PLS, supervised PCA, and linear probes, and supply two tests using
held-out correlations: a Nadeau--Bengio corrected asymptotic $t$-test
as a fast approximation, and a permutation test with comparable power,
finite-sample valid under outcome--predictor independence and iid rows.
Held-out predictions are
unchanged under any orthogonal rebasing of the supervised span, so an
interpretable basis such as varimax inherits the joint claim but not a
per-axis $p$-value; per-component claims come from a fixed-sequence
test on the PLS extraction order. We validate on synthetic geometries,
two NIR chemometric datasets, and cross-lingual word-embedding
regressions; the exact test also transfers to supervised PCA and a
ridge probe. The proposed tests have more power than
CV-permutation-$Q^{2}$, at a fraction of its cost. A pre-run check on
$n$ and the spectrum of $\mathbf{X}$ says when the approximation is
safe. We release a Rust library with Python, R, and Julia bindings,
plus a Python text pipeline.
\end{abstract}

\section{Introduction}
\label{sec:intro}

\subsection{Setting}
\label{sec:intro:opening}

Across applied science and machine learning, analysts run supervised
low-rank regression to extract a few $\mathbf{y}$-aligned directions
in $\mathbf{X}$ and ask the same omnibus question: ``does this
$K$-dimensional supervised subspace predict $\mathbf{y}$ better than
chance?'' The question is shared across PLS~\citep{wold2001pls},
supervised PCA~\citep{barshan2011spca}, sparse linear discriminant
analysis~\citep{clemmensen2011sparselda}, linear
probes~\citep{belinkov2022probing}, and sparse-autoencoder concept
directions~\citep{cunningham2024dictionary,templeton2024scaling,hindupur2025projecting},
and any test that operates on held-out OLS-refit predictions of the
supervised subspace applies in principle. None of the existing tests
answer it in a way that is both well-calibrated and cheap.

\subsection{The gap}
\label{sec:intro:gap}

\textbf{General --- omnibus inference.}
Existing PLS inference for the joint $K$-dimensional claim takes
three forms, all with known weaknesses: jackknife
$t$~\citep{mevik2007pls,martens2000jackknife} is flagged as biased by
its own maintainers; cross-validated
permutation-$Q^{2}$~\citep{westerhuis2008assessment} is calibrated but
expensive (\S\ref{sec:exp:powercalib}, App.~\ref{app:nontext}); and
\texttt{sklearn.PLSRegression}~\citep{sklearn_plsregression} ships no
test at all, the corresponding \texttt{statsmodels} feature
request~\citep{statsmodels_pls_issue} has been open since 2019, and
PLS-SEM has settled on bootstrap CIs~\citep{ray2022seminr}. The closest
route to inference is the degrees-of-freedom estimate of
\citet{kramer2011dof}, shipped as \texttt{plsdof}~\citep{plsdof_pkg}:
it supplies a DoF count for model selection and confidence intervals,
its authors do not propose it as a significance test or validate its
level, so benchmarking it as one means building the test for them and
reporting it failing at our choices (we ran that construction anyway;
App.~\ref{app:rule-grid}), and its exact DoF pass is
$O(n^{3})$~\citep{kramer2009lanczos}. No existing tool delivers
calibration and tractable cost together.

\textbf{NLP-specific --- per-axis lift.}
When features are individually human-readable (words, brain
regions, gene labels), analysts also want a calibrated test on
\emph{each} named axis. But PLS components are identified only
up to orthogonal rotation~\citep{kaiser1958varimax}, so a per-axis
claim has to say what survives rebasing, and no existing PLS
inference tool does.

\subsection{Contribution}
\label{sec:intro:contribution}

\paragraph{When the cheap test is valid, and what to run otherwise.}
We give a Nadeau--Bengio corrected resampled
$t$~\citep{nadeau2003inference} on held-out OLS-refit predictions of
the supervised subspace (\S\ref{sec:methods:test}). Its reference is
asymptotic, and we state the condition under which it holds:
$n\!\geq\!25$ and stable rank of the standardized $\mathbf{X}$ at
least 3 (\S\ref{sec:methods:validity}), a pre-run check that flags
every level failure across 25 null cells at 5{,}000 replicates and
sends flagged designs to an exact NB-permutation test that we specify
and price (\S\ref{sec:methods:exactperm}): with the splits held fixed
it refits nothing, and for $p\!\geq\!20{,}000$ it costs about the same
as the asymptotic test ($1.15\times$ at $p\!=\!20{,}000$, $n\!=\!200$,
less as $p$ grows).

\paragraph{More power at lower cost.}
NB is substantially cheaper than CV-permutation-$Q^{2}$ at matched or
better power (\S\ref{sec:exp:powercalib}).

\paragraph{Rotation invariance: what a named axis inherits.}
Prop.~\ref{prop:rotation-invariance} shows held-out OLS-refit
predictions are invariant under any orthogonal rebasing of the
supervised span, varimax included, so analysts can rotate to a
word-readable basis without disturbing inference; its negative half is
that a rotated axis inherits the joint claim and nothing else
(\S\ref{sec:methods:rotation}).

\paragraph{Per-component claims.}
NIPALS deflation and a fixed-sequence stop give a per-component test.
For a single-split variant with train-only deflation we prove strong
FWER control (App.~\ref{app:fwer-proof}): with probability at least
$1-\alpha$, every component the sequence keeps adds signal, not noise,
beyond the components fitted before it. The repeated-split NB test
is assessed empirically (\S\ref{sec:exp:fwer}). The
sequence answers a question selection does not: on Warriner valence
the split-half CV-$R^{2}$ curve is flat past $K\!=\!2$
($0.601 \to 0.619 \to 0.625$), so parsimony selectors stop there,
while NB rejects for component 3 at
$p\!\approx\!10^{-24}$ and component 4 at $p\!\approx\!10^{-11}$, and
the proved single-split variant, run once on a pre-specified split,
keeps all five components it tests.

\paragraph{What transfers.}
The exact test transfers to supervised PCA and a ridge probe; the
asymptotic NB correction is PLS1-scoped (\S\ref{sec:exp:beyond}).

\paragraph{Code and data.}
The \inferpkg{} (Rust core, Python/R/Julia bindings) and \textpkg{}
(Python text pipeline) libraries are on PyPI. End-to-end replication
scripts for every numbered table and figure, and the underlying result
CSVs/JSON/plots at paper settings, are in \repoanchor.

\section{Related work}
\label{sec:related}

\textbf{PLS inference.} Table~\ref{tab:inference-comparison} compares
the inferential targets, cost, calibration, and rotation treatment of
the principal alternatives. The corrected resampled $t$
\citep{nadeau2003inference,dietterich1998,bouckaert2004} underlies our
NB approximation. Related selection procedures use split-half RMSE
\citep{kvalheim2018} or multi-aspect permutation tests
\citep{barzizza2023}; selecting a component count does not establish
significance of each fitted component. We instead use Fisher-transformed
held-out correlations, distinguish the single-split guarantee from
the repeated-split approximation, and preserve joint predictions
under rebasing (Prop.~\ref{prop:rotation-invariance}). We benchmark
against \texttt{pls::jack.test} and permutation-$Q^{2}$ in
\S\ref{sec:exp:powercalib}.

\begin{table}[t]
\centering
\caption{Inference procedures reviewed in \S\ref{sec:intro:gap}. Costs describe the
tested fit or component: $F$ CV folds, $J$ repeated splits, $B$
permutations. Rotation entries concern rebasing a fixed fitted span;
none supplies significance for an individual varimax axis.}
\label{tab:inference-comparison}
\footnotesize
\setlength{\tabcolsep}{3pt}
\begin{tabular}{@{}p{.16\linewidth} p{.21\linewidth} p{.15\linewidth} p{.23\linewidth} p{.17\linewidth}@{}}
\toprule
\textbf{Method} & \textbf{Target} & \textbf{Cost} & \textbf{Calibration} & \textbf{Rotation} \\
\midrule
Jackknife $t$ & Regression coefficients & $F$ CV fits & Approximate; biased variance estimates & Tests coefficients, not latent axes \\
CV-perm.-$Q^2$ & Outcome--predictor independence & $(B+1)F$ fits & Exact under null permutation invariance & Joint predictions unchanged \\
\texttt{plsdof} & Model complexity; selection and CIs & $O(n^3)$ exact DoF pass & No supplied omnibus significance test & Fitted-model target \\
NB (ours) & Held-out predictive association & $J$ fits & Empirical; no proved component-chain FWER & Joint OLS refit invariant \\
Permutation (ours) & Outcome--predictor independence & $J$ cached maps at $K=1$; $B$ draws & Exact under iid rows and independence & Joint OLS refit invariant \\
Single split (ours) & Train-conditional residual association or independence & One training chain; held-out tests & Strong FWER: Gaussian correlation or permutation independence nulls & Claims follow training order \\
\bottomrule
\end{tabular}
\end{table}

\textbf{Supervised projections without calibrated per-axis
tests.} Concept-level interpretability of deep networks (post-hoc
probes \citep{kim2018tcav,belinkov2022probing}, learned feature
decompositions
\citep{koh2020concept,cunningham2024dictionary,templeton2024scaling,hindupur2025projecting})
locates structure inside an existing network rather than fitting a
regression in the embedding space; classical supervised DR (sLDA
\citep{blei2008slda}, supervised PCA \citep{barshan2011spca},
envelope models \citep{cook2010envelope}, reduced-rank regression
\citep{izenman1975rrr}) yields low-rank fits with likelihood-based
joint inference rather than per-axis; PLS itself is well-characterized
as a Krylov supervised-shrinkage method
\citep{helland1988pls,phatak2002krylov,hestenes1952cg,lingjaerde2000shrinkage}.
None packages a cheap, rotation-invariant, per-axis null; TCAV's
random-concept comparison is the closest analogue but targets a
different object.

\textbf{What the test certifies.}
NB calibration is about the predictive content of the held-out
supervised subspace, not the concept-causal identity of any direction
recovered after rotation: a calibrated rejection provides evidence
that the supervised reduction predicts $\mathbf{y}$ better than chance.
FWER control across $K^{*}$ axes additionally requires validity at
each true null; the single-split guarantee and the empirical scope of
NB are distinguished in \S\ref{sec:methods:test} and
\S\ref{sec:exp:fwer}. Concept-level claims
layered on top, such as register-consistency of a varimax axis, are
separate falsifiable hypotheses; calibrated tests can pass while a
probe relies on non-concept features~\citep{kumar2022probing}, so we
report calibration and interpretive readings separately and test the
latter against sharp counter-readings.

\section{Methods}
\label{sec:methods}

\noindent
We distinguish the signal dimension $\widetilde r$ and planted direction
count $r_{\mathrm{plant}}$ from PLS component counts $K$. Two recurring
$K$-symbols: $K_{\mathrm{fit}}$ is the number of PLS components extracted
at fit time, and $K^{*}$ is the data-driven count chosen by a predictive
selector (1-SE on CV-$R^{2}$) or BIC. Full glossary in App.~\ref{app:notation}.

\subsection{Setup: documents, axes, and the baseline pipeline}
\label{sec:methods:representations}
\label{sec:methods:baseline}

Texts are mean-pooled, $\ell_2$-normalized static word embeddings,
$\mathbf{X}\!\in\!\mathbb{R}^{n\times d}$, with
$\mathbf{y}\!\in\!\mathbb{R}^{n}$ the numeric outcome, both mean
centered, and any axis $\mathbf{a}\!\in\!\mathbb{R}^{d}$ is read from
its word loadings $\mathbf{e}_{w}^{\!\top}\mathbf{a}/\|\mathbf{a}\|$.
The 1-D reference point is the OLS-on-embeddings supervised projection
of \citet{plisiecki2025ssd} (OLS on the leading $K$ principal
components of $\mathbf{X}$, $K$ chosen by out-of-sample $R^{2}$;
pooling recipe and baseline in App.~\ref{app:setup}), whose PCA basis
\S\ref{sec:methods:pls} replaces with PLS.

\subsection{PLS as a supervised direction finder}
\label{sec:methods:pls}

\subsubsection{One-component PLS with an OLS score refit}
\label{sec:methods:pls1}

PLS1 (NIPALS;~\citealt{wold2001pls,hoskuldsson1988pls}) upgrades the
PCA+OLS baseline of \S\ref{sec:methods:baseline} to a basis whose
leading direction maximizes sample covariance with $\mathbf{y}$:
\begin{equation}
    \mathbf{w}
    \;=\;\arg\max_{\|\mathbf{w}\|=1}\,
        \operatorname{Cov}\!\bigl(\mathbf{X}\mathbf{w},\,\mathbf{y}\bigr)
    \;\propto\; \mathbf{X}^{\!\top}\mathbf{y},
    \label{eq:pls1-weight}
\end{equation}
with score $\mathbf{t}\!=\!\mathbf{X}\mathbf{w}$. At $K\!=\!1$,
regressing $\mathbf{y}$ on $\mathbf{t}$ gives
$\hat{\boldsymbol\beta}_{\mathrm{PLS},1}\!=\!q\,\mathbf{w}$,
where $q\!=\!\mathbf{t}^{\!\top}\mathbf{y}/\mathbf{t}^{\!\top}\mathbf{t}$.
This is OLS on the PLS score; in general it differs from unrestricted
OLS and from OLS on the leading principal component. Each test fixes
$K\!=\!1$, using outcome covariance to choose its direction.

\subsubsection{\texorpdfstring{$K\!>\!1$}{K>1}: Krylov saturation and calibration axes}
\label{sec:methods:plsK}

For $K\!\geq\!1$, PLS extracts an ordered sequence
$\mathbf{w}_1,\dots,\mathbf{w}_K$ by iterating
Eq.~\eqref{eq:pls1-weight} with NIPALS deflation (full recurrence in
App.~\ref{app:nipals}). The $K$-dim PLS span is the Krylov subspace
generated by $\mathbf{X}^{\!\top}\mathbf{X}$ and
$\mathbf{X}^{\!\top}\mathbf{y}$~\citep{helland1988pls,phatak2002krylov}:
\begin{equation}
    \operatorname{span}(\mathbf{w}_1,\dots,\mathbf{w}_K)
    \;=\;\mathcal{K}_K(\mathbf{X}^{\!\top}\mathbf{X},\mathbf{X}^{\!\top}\mathbf{y}).
    \label{eq:krylov-span}
\end{equation}
This is a span statement, hence preserved under any \emph{invertible}
rebasing $\mathbf{R}\!\in\!GL(K)$.
For a scalar linear response, the conditional mean depends on one
linear index, but correlated predictors can require several PLS
components to recover its coefficient vector. We call convergence to
$\hat{\boldsymbol\beta}_{\mathrm{OLS}}$ \emph{Krylov saturation}:
its depth depends on the predictor spectrum and its alignment with
the outcome covariance, not on the dimension of that linear index.
Later components can improve prediction without identifying separate
latent causes (Def.~\ref{def:calibration-axis};
App.~\ref{app:proof}; demonstrated in \S\ref{sec:exp:synth}).

\subsection{Inference: NB test, FWER, and rotation invariance}
\label{sec:methods:test}

\label{sec:methods:test:tnb}

Per-axis $p_{k}$ run on the \emph{unrotated} NIPALS components and
feed a fixed-sequence FWER stop; rotation enters after inference
as a basis change for verbal interpretation
(\S\ref{sec:methods:rotation}). The joint summary
$\alpha^{*}_{\mathrm{FWER}}$ is defined and shown to be
rotation-invariant in Prop.~\ref{prop:rotation-invariance} below.
Figure~\ref{fig:inference-pipeline} schematizes
the three stages.

\begin{figure}[t]
  \centering
  \includegraphics[width=0.85\textwidth]{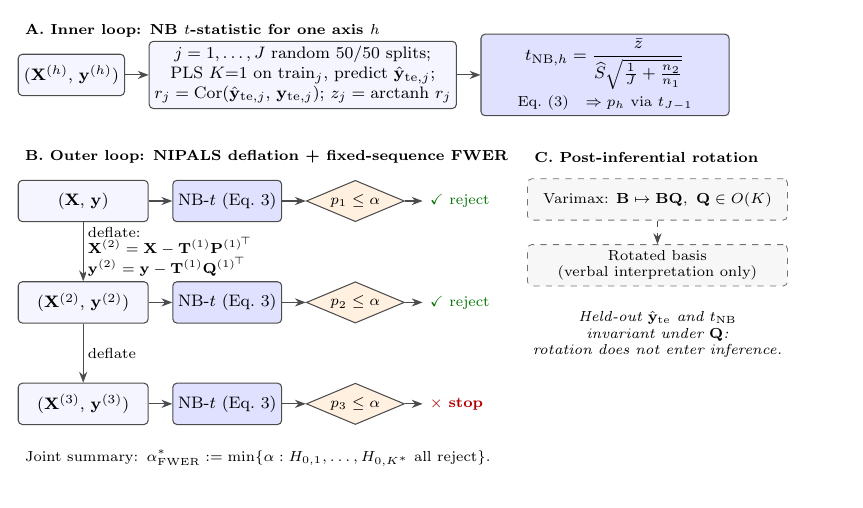}
  \caption{Inference pipeline for \S\ref{sec:methods:test}.
  \textbf{(A)} Per-axis NB $t$-statistic ($t_{\mathrm{NB},h}$,
  Eq.~\eqref{eq:nb}): $J$ random 50/50 splits feed Fisher-$z$
  aggregation with the corrected resampled variance.
  \textbf{(B)} Fixed-sequence FWER over NIPALS-deflated pairs
  $(\mathbf{X}^{(h)}, \mathbf{y}^{(h)})$, $h = 1, \dots, K^{*}$,
  stopping at the first non-rejection; the joint summary
  $\alpha^{*}_{\mathrm{FWER}}$ is the minimum level at which all
  preceding nulls reject.
  \textbf{(C)} Varimax acts on loadings only after inference;
  held-out predictions and $t_{\mathrm{NB}}$ are invariant under any
  $\mathbf{Q} \in O(K)$.}
  \label{fig:inference-pipeline}
\end{figure}

\subsubsection{The corrected resampled \texorpdfstring{$t$}{t}-statistic}
We adapt the corrected resampled $t$-statistic
of~\citet{nadeau2003inference} to a single supervised direction. Let
$J$ be the number of random partitions and, for split $j$, let
$(\mathcal{I}_{j}^{\mathrm{tr}},\mathcal{I}_{j}^{\mathrm{te}})$ be a
random 50/50 partition of $\{1,\dots,n\}$ with sizes $n_{1}$ and
$n_{2}$. On each split we refit PLS at $K\!=\!1$ on
$\mathcal{I}_{j}^{\mathrm{tr}}$, predict on the held-out half via the
rank-1 supervised direction, and record the Pearson correlation
$r_{j}=\operatorname{Cor}\!\bigl(\hat{\mathbf{y}}^{\,\mathrm{te}}_{j},\,\mathbf{y}^{\,\mathrm{te}}_{j}\bigr)$.
After Fisher transformation $z_{j}=\operatorname{arctanh}(r_{j})$, we
form
\begin{equation}
    t_{\mathrm{NB}}
    \;=\;\frac{\bar z}
              {\widehat{S}\,\sqrt{\,\dfrac{1}{J}+\dfrac{n_{2}}{n_{1}}\,}},
    \qquad
    \text{reference under }H_0:\quad t_{J-1},
    \label{eq:nb}
\end{equation}
where $\bar z$ and $\widehat{S}^{2}$ are the sample mean and variance
of $\{z_{j}\}_{j=1}^{J}$. The $(1/J + n_{2}/n_{1})$ correction
approximates the dependence created by overlapping training folds;
the $t$ reference is assessed in \S\ref{sec:methods:validity}.

\subsubsection{Per-axis tests via NIPALS deflation}
For axis $h\!=\!2,\dots,K$, deflate $(\mathbf{X},\mathbf{y})$ by the
leading $h\!-\!1$ full-data NIPALS components (full recurrence in
App.~\ref{app:nipals}) and apply Eq.~\eqref{eq:nb} at $K\!=\!1$ to the
deflated pair $(\mathbf{X}^{(h)},\mathbf{y}^{(h)})$; the omnibus claim
is the $h\!=\!1$ case on the full data.

\subsubsection{Per-axis claims and FWER control}
We test $H_{0,1},H_{0,2},\dots$ at level $\alpha$ in NIPALS extraction
order (step $k$ extracts the $\mathbf{w}_{k}$ maximizing
$|\!\operatorname{cov}(\mathbf{t}_{k},\mathbf{y}_{k-1})|$ on the
$(k\!-\!1)$-fold deflated pair), stopping at the first non-rejection;
the extraction order is not changed in response to test outcomes.
Under the fixed-sequence sequential rejection
principle~\citep{goeman2010sequential,westfall2001optimally}, strong
FWER control at level $\alpha$ holds provided each per-step NB test is
level-$\alpha$ valid for each true null in the ordered family,
regardless of the truth of earlier nulls (the \emph{boundary
property}). Any false rejection requires
rejection of the first true null, so no independence between tests is
needed. Step 1 also supplies the omnibus test of $y\perp\mathbf x$;
no multiplicity correction is needed for that test alone.

\emph{Status.} For repeated-split NB, the boundary property is
assessed empirically, not formally proved (\S\ref{sec:exp:fwer}).
App.~\ref{app:fwer-proof} gives a finite-sample guarantee for a
single independent split with train-only deflation and Gaussian
correlation tests. There, rejecting step $h$ means component $h$ is
positively correlated with what components $1,\dots,h\!-\!1$ leave
unexplained, so adding it with a small enough positive weight lowers
population squared error; with probability at least $1-\alpha$, no
kept component is noise in this sense. The components are those
fitted on the training half, and the gain is relative to the earlier
fitted components, not a new $\mathbf{y}$-direction
(\S\ref{sec:methods:plsK}). That proof does not establish calibration
of the $J$-split NB aggregate, with either train-only or full-data
deflation, so for NB this reading is supported
empirically, not proved.

For the joint summary, with $K^{*}$ fixed by 1-SE on CV-$R^{2}$ (or
BIC),
\begin{equation*}
  \alpha^{*}_{\mathrm{FWER}}
  \;:=\;
  \min\bigl\{\alpha : H_{0,1},\dots,H_{0,K^{*}}\text{ all reject}\bigr\}
  \;=\;\max_{h\le K^{*}}p_h.
\end{equation*}
This is the smallest nominal level at which the entire specified
prefix rejects. Truncating an otherwise unchanged valid sequence at
$K^{*}$ cannot increase its FWER; it does not establish the missing
per-step validity of NB.

\subsubsection{When the asymptotic test is valid}
\label{sec:methods:validity}
Eq.~\eqref{eq:nb} refers $t_{\mathrm{NB}}$ to a $t_{J-1}$ reference
that is asymptotic in two places: the normality of each Fisher-$z$ and
the plug-in $\rho_{0}\!=\!n_{2}/(n_{1}+n_{2})$ for the correlation between
overlapping splits. \textbf{Rule.} Use NB-asymptotic when
$n\!\geq\!25$ and the stable rank
$\|\mathbf{X}\|_F^{2}/\|\mathbf{X}\|_2^{2}$ of the column-standardized
$\mathbf{X}$ is at least 3; otherwise use the exact test of
\S\ref{sec:methods:exactperm}. Stable rank costs one SVD and is the
reciprocal of the first principal component's variance share. On 25
null cells at 5{,}000 replicates ($n$ from 10 to 320;
\S\ref{sec:exp:calibration}, App.~\ref{app:rule-grid}) the rule flags
every cell whose $\alpha\!=\!.05$ or $.10$ level moves, and every
cleared cell holds level or is conservative. The mechanism is the plug-in:
\citet{nadeau2003inference} insert a design-independent $\rho_{0}$
into the variance of the resampled $t$, exact only when the $J$ split
estimates are $J$ replicates at that correlation. On a one-direction
$\mathbf{X}$, PLS1 at $K\!=\!1$ returns only $\pm\mathbf{v}$, with
the sign set by the training half, so the $J$ estimates re-cut one
scalar, $\mathrm{cor}(\mathbf{u}_{1},\mathbf{y})$: their correlation
drops below $\rho_{0}$, and the null of $t_{\mathrm{NB}}$ is narrower
than $t_{J-1}$ in the centre and heavier in the tail, conservative at
$\alpha\!=\!.05,.10$ and inflated at $.01$ (App.~\ref{app:rule-grid}). Two qualifiers: the threshold was set empirically, on a separate
decaying-spectrum sweep, and errs toward flagging; and the $n\!<\!25$
gate covers a different failure, the small-sample Fisher-$z$
asymptotic, which breaks at $n\!=\!10$ even on an iid design with
stable rank 7. The useful range of NB-asymptotic is $\alpha\!=\!.05$
and $.10$ (\S\ref{sec:disc:limitations}).

\subsubsection{Exact NB-permutation}
\label{sec:methods:exactperm}
The fallback uses the same split-level correlations, but permutes
their mean Fisher transform $\bar z$ rather than the studentized
$t_{\mathrm{NB}}$ of Eq.~\eqref{eq:nb}.
(1)~Draw the $J$ split-halves independently of the data and hold them fixed. (2)~On each
split fit PLS at $K\!=\!1$ on the training half, refit OLS, Fisher-transform
the held-out correlation, and average over splits to obtain $\bar z$.
(3)~For each of $B$ independent uniform permutation draws, permute
$\mathbf{y}$ once and recompute $\bar z$ with that same permuted
$\mathbf{y}$ in every split;
$p\!=\!(\#\{\bar z_{b}\!\geq\!\bar z_{\mathrm{obs}}\}+1)/(B+1)$.
(4)~Under $H_0:y\perp\mathbf x$ and iid rows, the distribution of
$\mathbf y$ conditional on $\mathbf X$ and the fixed splits is
invariant under these permutations. Recomputing all outcome-dependent
training steps for each draw therefore gives a finite-sample valid
test, without Gaussian or large-$n$ assumptions. Exchangeability of
the paired rows alone does not justify permuting outcomes relative to
predictors. Using one shared outcome permutation across splits follows
\citet{valente2021permutation}.
\emph{Cost.} With the splits fixed and $K\!=\!1$ the held-out score on
each split is a fixed linear map of the training outcomes, so a
permutation only changes the input vector: the $B$ draws collapse to
two matrix products per split with no refit, exact conditional on the
splits. At $n\!=\!200$ the cost over
NB-asymptotic falls from $20.8\times$ at $p\!=\!75$ to $2.5\times$ at
$p\!=\!1{,}200$ (App.~\ref{app:rule-grid}); for $p\!\geq\!20{,}000$ it
costs about the same ($1.15\times$ at $p\!=\!20{,}000$, less as $p$
grows): exactness is close to free exactly where the rule of
\S\ref{sec:methods:validity} sends the user. The collapse needs a
direction linear in $\mathbf{y}$; a pipeline with a selection step pays
the $B$ refits (\S\ref{sec:exp:beyond}). Exactness covers the omnibus
step only: after full-data deflation the permutation reference is not
exact (\S\ref{sec:exp:fwer}), so on flagged designs per-component
claims should come from the single-split variant of
App.~\ref{app:fwer-proof}, whose permutation form assumes only iid
rows.

\subsubsection{Rotation invariance of the held-out predictions}

\begin{proposition}[Rotation invariance of held-out OLS-refit predictions]
\label{prop:rotation-invariance}
Fix $\mathbf{Q}\!\in\!O(K)$ across splits (Assumption~A2,
App.~\ref{app:proof}). For any split $j$, the per-split OLS refit on
rotated scores
$\widetilde{\mathbf{T}}_{\mathrm{tr},j}\!=\!\mathbf{T}_{\mathrm{tr},j}\mathbf{Q}$
satisfies
$\widetilde{\boldsymbol\beta}_j\!=\!\mathbf{Q}^{\!\top}\boldsymbol\beta_j$,
so held-out predictions are pointwise equal:
$\widetilde{\mathbf{T}}_{\mathrm{te},j}\widetilde{\boldsymbol\beta}_j
=\mathbf{T}_{\mathrm{te},j}\mathbf{Q}\mathbf{Q}^{\!\top}\boldsymbol\beta_j
=\mathbf{T}_{\mathrm{te},j}\boldsymbol\beta_j$.
Any test statistic depending on the splits only through
$\{\hat{\mathbf{y}}_{\mathrm{te},j}\}_{j=1}^{J}$, including
$t_{\mathrm{NB}}$ of~\eqref{eq:nb}, is therefore invariant under
$\mathbf{Q}$.
\end{proposition}

\begin{proof}
OLS predictions are invariant under any invertible reparameterization
of the regressors; $\mathbf{Q}\!\in\!O(K)$ additionally keeps
loading-side varimax stable (full statement in App.~\ref{app:proof}).
\end{proof}

\paragraph{What rotation does not buy.}
A varimax axis inherits the joint subspace's significance and nothing
else; running the fixed sequence on rotated axes is not licensed,
because the sequence is defined on the NIPALS extraction order
(Table~\ref{tab:hypotheses}).

\subsection{Interpretation: varimax as a basis change after inference}
\label{sec:methods:rotation}

PLS identifies the supervised subspace but leaves the basis ambiguous:
components are identified only up to orthogonal
rotation~\citep{kaiser1958varimax}. We rotate the \textbf{loading
matrix}
$\mathbf{L}\!=\!\mathbf{E}_{\!\mathcal{V}_{\!\mathrm{tgt}}}\mathbf{W}\!\in\!\mathbb{R}^{V_{\!\mathrm{tgt}}\times K}$
(rather than $\mathbf{W}$ or document scores) by Kaiser-normalized
varimax against vocabulary target $\mathcal{V}_{\!\mathrm{tgt}}$
(objective, sign/order recipe, and full transformation in
App.~\ref{app:varimax}); $\mathcal{V}_{\!\mathrm{tgt}}\!=\!\mathcal{V}$
gives the unsupervised variant, and a curated content-word lexicon the
\emph{smart-target} variant, the embedding analogue of marker-variable
specification in confirmatory factor analysis. Rotation rebases
\emph{inside} the Krylov span to concentrate the leading supervised
loading on a sparse, describable Dim-1, so the $K\!>\!1$ calibration
axes read as named tweaks of the main axis. Joint
$R^{2}_{\mathrm{full}}$ is rotation-invariant; per-axis $R^{2}$, when
reported, is on the unrotated NIPALS scores (App.~\ref{app:varimax}).
Hard-cardinality sparse-PLS (top-$s$ deflation, no rotation) is an
alternative readability route (App.~\ref{app:sparse-pls}).
Table~\ref{tab:hypotheses} states which null each object carries.

\begin{table}[t]
\centering
\caption{What is tested before and after rotation.}
\label{tab:hypotheses}
\footnotesize
\begin{tabular}{@{}l p{5.0cm} p{5.4cm}@{}}
\toprule
\textbf{Object} & \textbf{Null tested} & \textbf{Guarantee} \\
\midrule
Joint subspace & rank-$K$ fit predicts $\mathbf{y}$ no better than chance & omnibus NB test; exact under NB-permutation \\
Unrotated component $h$ & no positive residual association after $h\!-\!1$ deflations & fixed-sequence; FWER requires per-step validity (\S\ref{sec:exp:fwer}) \\
Varimax axis & none --- no separate test & inherits joint significance only (Prop.~\ref{prop:rotation-invariance}) \\
\bottomrule
\end{tabular}
\end{table}

\section{Experiments}
\label{sec:experiments}
\label{sec:exp:setup}

Inferential claims are made on EN-Warriner GloVe-300 (text) and on
two real NIR designs, Tecator and gasoline; Polish and Spanish appear
only in the $K\!=\!3$ cross-lingual demonstration of
\S\ref{sec:exp:peraxis-glove}. Embedding provenance and vocabulary
cleaning are in App.~\ref{app:embeddings}.

\subsection{Mechanism --- Krylov span and the suppressor fingerprint}
\label{sec:exp:synth}

\begin{figure}[t]
\centering
\includegraphics[width=\linewidth]{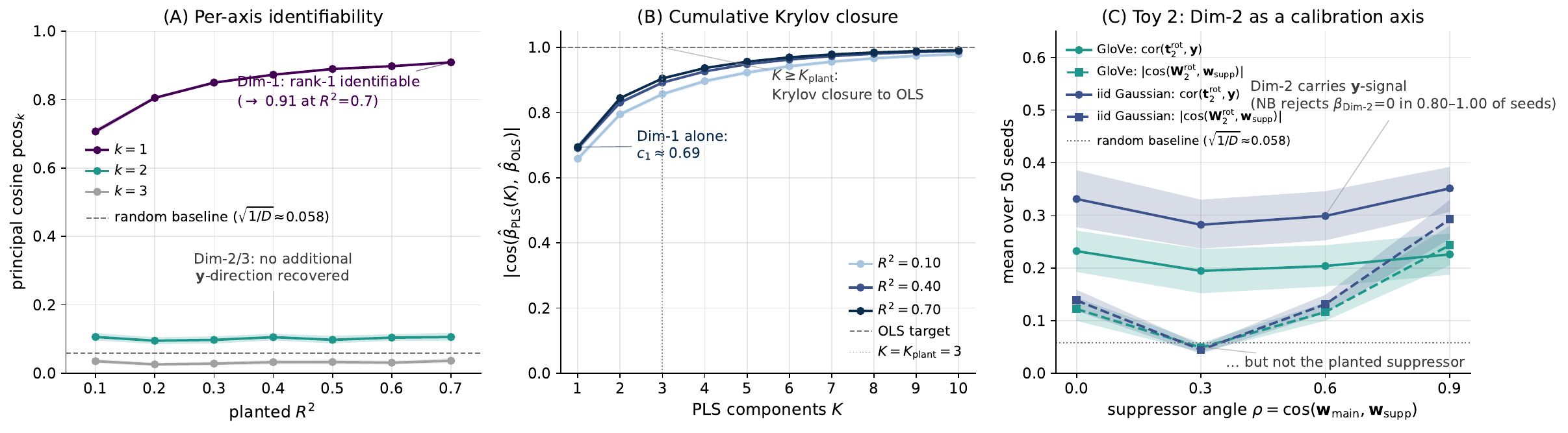}
\caption{\textbf{Krylov saturation of supervised PLS to OLS, and the
calibration-axis fingerprint.}
\textit{(A)} Per-axis principal cosine against the planted
subspace at $r_{\mathrm{plant}}\!=\!3$, across planted $R^2$.
\textit{(B)} Cumulative $\hat{\beta}_{\mathrm{PLS}}(K)$ versus
ridge-stabilized $\hat{\beta}_{\mathrm{OLS}}$, $K\!=\!1\dots10$.
Toy~1 setup: $N\!=\!4000$, $D\!=\!300$, Warriner GloVe, 50 seeds,
$\boldsymbol{\alpha}\!=\![1.0,0.7,0.5]$. iid Gaussian replication
and $r_{\mathrm{plant}}\!=\!2$ in App.~\ref{app:toys}.
\textit{(C)} Toy~2: at $K\!=\!2$ with a planted suppressor at angle
$\rho$ to the main predictor, the varimax-rotated Dim-2 does not
recover the suppressor (dashed, near the random baseline) yet carries
held-out $\mathbf{y}$-signal (solid), on GloVe and iid geometry alike;
50 seeds per cell, bootstrap 95\% bands.}
\label{fig:calibration-axis}
\end{figure}

Two controlled mechanism experiments (Toys 1--2) anchor the
methodology, run on both real Warriner GloVe geometry and iid
Gaussian $\mathbf{X}$ so identifiability and suppressor-fingerprint
claims are evaluated independently of text structure.
\textbf{Toy~1 (Krylov saturation).} Figure~\ref{fig:calibration-axis}
illustrates the calibration-axis reading directly (formal statement in
App.~\ref{app:proof}): $c_K$ rises from $\approx\!0.69$ at $K\!=\!1$
to $\geq\!0.98$ by $K\!=\!10$ across all three signal strengths, so
higher components do real predictive work even though they recover no
new $\mathbf{y}$-direction, and subspace plateau values are matched
across geometries ($\overline{\mathrm{pcos}}\!\in\![0.41,0.45]$ GloVe,
$[0.46,0.51]$ iid at $r_{\mathrm{plant}}\!=\!2$; $[0.28,0.35]$ and
$[0.30,0.37]$ at $r_{\mathrm{plant}}\!=\!3$; full per-axis table in
App.~\ref{app:toys}).
\textbf{Toy~2 (calibration fingerprint at $K\!=\!2$;
Fig.~\ref{fig:calibration-axis}C).} Varimax-rotated Dim-2 does
\emph{not} recover the planted suppressor
($|\!\cos|\!\in\![0.04,0.29]$) yet reproduces the empirical
calibration-axis fingerprint of \S\ref{sec:exp:peraxis-glove}
($\operatorname{cor}(\mathbf{t}_2^{\mathrm{rot}},\mathbf{y})\!\in\![0.19,0.35]$,
NB rejects in 0.80--1.00 of replicates): Dim-2 absorbs
$\mathbf{y}$-aware $\mathbf{X}$-residual variance, flagging which
clusters the rank-1 direction misweights, and the fingerprint is
geometry-independent (matched on GloVe and iid). Toys~3--5 numerically
corroborate functional-form invariance, rotation invariance
(Prop.~\ref{prop:rotation-invariance}), and selector-vs-rank separation
(App.~\ref{app:toys}, App.~\ref{app:overshoot}).

\subsection{Inference: the validity rule and null calibration}
\label{sec:exp:calibration}

Table~\ref{tab:cal-combined} shows the rule of
\S\ref{sec:methods:validity} on eight of its 25 null cells plus
Tecator (all in App.~\ref{app:rule-grid}). The rule flags every cell
whose level moves at $\alpha\!=\!.05$ or $.10$, and clearing is not a
clean bill: five cleared cells are conservative enough that their
Wilson interval excludes nominal, which costs power only.
KS-uniformity is not the discriminator: two cleared block-collinear
cells have KS $p$ near $10^{-6}$ while holding level at both $\alpha$.
The exact test holds level in every cell ($0.042$--$0.057$ at
$\alpha\!=\!.05$, $0.093$--$0.108$ at $.10$, KS never below $0.024$).
Tecator, on which the asymptotic null visibly drifts (KS
$p\!\approx\!3{\times}10^{-6}$ at $n\!=\!40$; PP plot in
App.~\ref{app:nontext}), is essentially rank one, stable rank $1.014$
with the first principal component carrying $98.6\%$ of the variance,
so the rule flags it, and it was not used to set the threshold. On
flagged designs the drift changes the shape of the null rather than
its $\alpha\!=\!.05$ level; the inflated $0.30$ rejection rate on
Tecator belongs to the uncorrected $t$-test (App.~\ref{app:nontext}).

\begin{table}[t]
\centering
\caption{The validity rule on eight null cells (full grid of 25 plus
Tecator in App.~\ref{app:rule-grid}). Stable rank
$\|\mathbf{X}\|_F^{2}/\|\mathbf{X}\|_2^{2}$ of the column-standardized
design; \emph{flag} if $n\!<\!25$ or stable rank $<\!3$. Level:
rejection rate under $H_{0}$ at $\alpha\!=\!.05$ / $.10$; KS:
KS-uniform $p$ of the NB-asymptotic null $p$-values. 5{,}000 replicates
per cell (Wilson half-width $\pm 0.006$ at $.05$, $\pm 0.008$ at $.10$)
except Warriner (2{,}000) and Tecator (NB-asymptotic 400, exact
2{,}000).}
\label{tab:cal-combined}
\footnotesize
\setlength{\tabcolsep}{4pt}
\begin{tabular}{@{}lrrl rr r rr@{}}
\toprule
 & & & & \multicolumn{2}{c}{\textbf{NB-asymptotic}} & & \multicolumn{2}{c}{\textbf{exact}} \\
\cmidrule(lr){5-6}\cmidrule(lr){8-9}
\textbf{Design} & $n$ & \textbf{st.\ rank} & \textbf{rule} & $.05$ & $.10$ & \textbf{KS (NB)} & $.05$ & $.10$ \\
\midrule
iid Gaussian ($p\!=\!300$)                    &  40 & 22.0  & pass & 0.038 & 0.083 & $1.3{\times}10^{-3}$ & 0.048 & 0.096 \\
iid Gaussian                                  & 160 & 55.3  & pass & 0.050 & 0.100 & 0.72                 & 0.052 & 0.106 \\
block-collinear (5 blocks, $r\!=\!.9$)        &  40 &  3.80 & pass & 0.050 & 0.095 & $7.6{\times}10^{-6}$ & 0.050 & 0.103 \\
Warriner GloVe-300 (real)                     &  40 & 17.1  & pass & 0.035 & 0.085 & 0.18                 & 0.042 & 0.100 \\
\midrule
dominant factor ($r\!=\!.9$, all columns)     &  40 &  1.10 & \textbf{flag} & 0.041 & 0.061 & $2{\times}10^{-77}$ & 0.050 & 0.100 \\
NIR gasoline ($60\!\times\!401$, real)        &  60 &  1.39 & \textbf{flag} & 0.049 & 0.077 & $4{\times}10^{-31}$ & 0.050 & 0.098 \\
Tecator NIR ($215\!\times\!100$, real)        &  40 &  1.014 & \textbf{flag} & 0.055 & 0.080 & $2.9{\times}10^{-6}$ & 0.052 & 0.094 \\
iid Gaussian                                  &  10 &  7.08 & \textbf{flag} ($n$) & 0.004 & 0.020 & $10^{-70}$ & 0.056 & 0.108 \\
\bottomrule
\end{tabular}
\end{table}

\subsection{Inference: detection power and cost}
\label{sec:exp:powercalib}

Fig.~\ref{fig:s3-power-curves} sweeps $n$ on the synthetic Warriner geometry with
a planted rank-1 SVD signal at $R^2\!\in\!\{0.05,0.40\}$ and
$r_{\mathrm{plant}}\!\in\!\{1,3\}$; the
NB $\geq$ perm-$Q^2$ ordering holds at every $n$ and is widest in the noisy regime.
Cell-level point checks and the Tecator replication are in App.~\ref{app:nontext}
(Tab.~\ref{tab:s2-power}, Fig.~\ref{fig:t3-power-curves}). On the
Warriner ladder from $n\!=\!400$ to the full $13{,}365$ rows, perm-$Q^{2}$
costs $115$--$125\times$ NB-asymptotic per test and the exact
NB-permutation of \S\ref{sec:methods:exactperm} costs $7.5$--$15\times$
(App.~\ref{app:rule-grid}, Tab.~\ref{tab:cost}; real-geometry
replication of the ordering in App.~\ref{app:embedding-power}).

\begin{figure}[t]
\centering
\includegraphics[width=\linewidth]{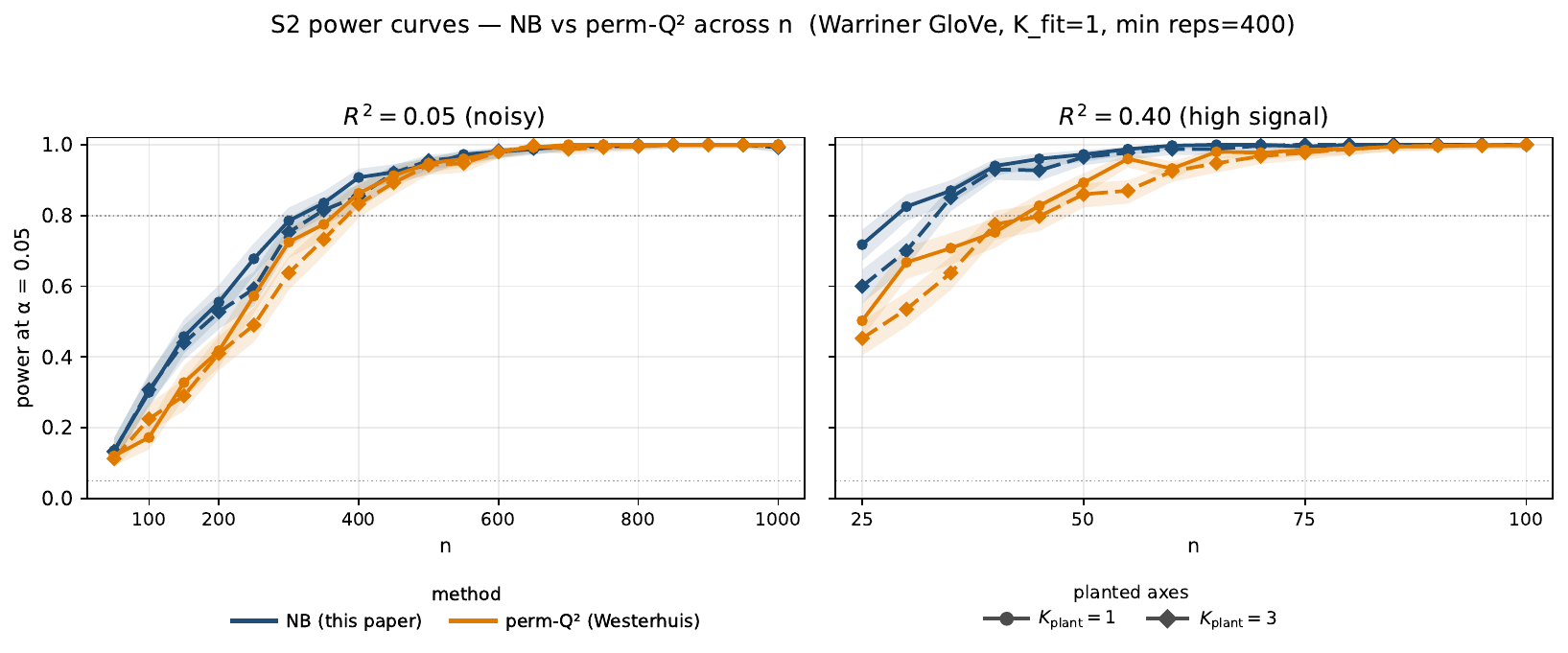}
\caption{Empirical power vs.\ $n$ at $\alpha\!=\!0.05$ on synthetic Warriner
geometry. NB (blue) vs.\ perm-$Q^2$ (orange) at $r_{\mathrm{plant}}\!=\!1$ (solid)
and $r_{\mathrm{plant}}\!=\!3$ (dashed); $R^2\!=\!0.05$ (left) and
$R^2\!=\!0.40$ (right). Matched-power sample sizes in App.~\ref{app:samplered}.}
\label{fig:s3-power-curves}
\end{figure}

\subsection{Inference: per-axis rejection beyond planted signal}
\label{sec:exp:fwer}

The fixed-sequence guarantee depends on validity at the first true
null. App.~\ref{app:fwer-proof} establishes this for a single-split
Gaussian test; validity of the repeated-split NB procedure remains
unproved. We therefore examine how often the NB sequence
rejects beyond $m$ planted signal axes. This is a diagnostic rather
than a direct estimate of FWER for the appendix's conditional
hypotheses: noisy fitted deflation can leave residual signal at
step $m\!+\!1$. At step 1, NB-permutation is finite-sample valid
under iid rows and $y\perp\mathbf x$.

Table~\ref{tab:fwer-boundary} uses strong planted signal to make
step $m\!+\!1$ reachable: four designs, one to three deflations,
5{,}000 replicates. At one deflation the dominant-factor design has
the largest chain rejection rate, close to the nominal thresholds:
$0.049$ $[0.043, 0.055]$ at $\alpha\!=\!.05$ and $0.095$
$[0.087, 0.104]$ at $.10$. Train-only NB on the same generated
datasets gives $0.038$ $[0.033, 0.044]$ and $0.081$ $[0.073, 0.089]$,
respectively.
The other three designs reach that step as often but reject less
frequently. NB-permutation applied naively
after full-data deflation gives $0.059$ $[0.053, 0.066]$ at
$\alpha\!=\!.05$; its exactness argument applies at step 1 only.
At two and three deflations the observed chain rejection rates are
below nominal, but three of the eight design--depth cells never
reach step $m\!+\!1$, so their zeros do not assess that step's test.
These experiments cover at most three deflations; additional grids
on Warriner and Tecator appear in App.~\ref{app:peraxis-fwer}.

\begin{table}[t]
\centering
\caption{Probability that the sequence rejects through step $m\!+\!1$ after $m$ planted signal axes,
$n\!=\!40$, $J\!=\!50$, total planted $R^{2}\!=\!0.5$
spread over the $m$ signal axes, 5{,}000 replicates per cell, at
$\alpha\!=\!.05$ / $.10$. \emph{full-data, NB}: every step deflated on all rows;
\emph{train-only, NB}: each split deflated using its training half.
\emph{full-data, perm.}: permutation reference at the deflated step,
where it is not exact. Step $m\!+\!1$ need not be a true conditional
null (see text). $\dagger$: the chain never reaches that step
(the signal axes below it are not detected), so the zero carries no
information. Setup and the $\alpha\!=\!.01$ column in
App.~\ref{app:peraxis-fwer}.}
\label{tab:fwer-boundary}
\footnotesize
\setlength{\tabcolsep}{4pt}
\begin{tabular}{@{}ll ccc@{}}
\toprule
\textbf{Design} & \textbf{Arm} & $m\!=\!1$ & $m\!=\!2$ & $m\!=\!3$ \\
\midrule
iid Gaussian ($40\!\times\!20$)            & full-data, NB    & 0.000 / 0.000 & 0.000$^{\dagger}$ / 0.000$^{\dagger}$ & 0.000$^{\dagger}$ / 0.000$^{\dagger}$ \\
dominant factor ($40\!\times\!20$)         & full-data, NB    & \textbf{0.049 / 0.095} & 0.000 / 0.000 & 0.000$^{\dagger}$ / 0.000$^{\dagger}$ \\
                                           & train-only, NB   & 0.038 / 0.081 & --- & --- \\
                                           & full-data, perm. & 0.059 / 0.118 & 0.000 / 0.000 & 0.000$^{\dagger}$ / 0.000$^{\dagger}$ \\
decaying spectrum ($40\!\times\!60$)       & full-data, NB    & 0.007 / 0.018 & 0.004 / 0.012 & 0.002 / 0.007 \\
NIR gasoline ($40\!\times\!401$, real)     & full-data, NB    & 0.017 / 0.035 & 0.011 / 0.024 & 0.001 / 0.002 \\
\bottomrule
\end{tabular}
\end{table}

\subsection{Beyond PLS: supervised PCA and a ridge probe}
\label{sec:exp:beyond}

Under iid rows and outcome--predictor independence, the permutation
test covers training pipelines recomputed for every outcome permutation.
We ran it
around supervised PCA~\citep{bair2006spca} and a ridge probe, with
PLS1 as the reference, on Warriner and Tecator at two $n$ each, 2{,}000
null replicates per cell (App.~\ref{app:beyond-pls}). Two findings,
both bounding the paper's claims. The permutation floor transfers:
the exact test holds level in all eight non-PLS null cells at
$\alpha\!=\!.05$ and $.10$ (Wilson intervals cover nominal; KS never
below $0.076$; at $\alpha\!=\!.01$ seven of eight, the exception being
Tecator supervised PCA at $n\!=\!40$ at $0.017$), and its
power is at or above perm-$Q^{2}$'s in all twelve cells. The NB
correction does not: off PLS1 it turns conservative and loses power
(supervised PCA at $0.010$ and $0.015$ where nominal is $.05$, against
$0.044$ and $0.053$ on PLS1).

\subsection{Applied: cross-lingual VAD}
\label{sec:exp:vadfit}
\label{sec:exp:peraxis-glove}

On three (language $\times$ embedding) configurations, English
GloVe-300 with Warriner V/A/D, Polish GloVe-800 with ANPW V/A/D, and
Spanish GloVe-300 with Stadthagen-Gonz\'alez V/A, the $K\!=\!3$ fit is
significant in all eight outcome cells
($\alpha^{*}_{\mathrm{FWER}}\!<\!.001$, per-axis NB on the unrotated
NIPALS extraction; cross-validated $R^{2}$ from $0.35$ to $0.65$), the
leading rotated axis is the conventional affect direction in every
cell, and every cell has at least one $K\!>\!1$ rotated axis carrying
the same technical-register vs.\ emotion-laden contrast (EN
\emph{industrial, containing, metal} against \emph{blessings,
faithful, compassion}), the empirical counterpart of the Toy~2
fingerprint. Configurations, fit statistics, per-pole word lists, the
four alternative readings ruled out, and the sparse-PLS readability
route: App.~\ref{app:vad-words}, \ref{app:dim3-robustness},
\ref{app:sparse-pls}.

\section{Limitations}
\label{sec:discussion}
\label{sec:disc:limitations}

\label{sec:disc:scope}
\textbf{(0) Rotation half is text-only by design.}
Loading-side varimax against a curated vocabulary target carries
interpretive content only where features are individually
human-readable; for spectral wavelengths, voxels, or gene-expression
probes the rotation target is the wrong primitive (sparsity,
spatial-contiguity, or cluster-structure priors are the appropriate
substitutes). \textbf{(a--b) The asymptotic reference is design-dependent and established only at $\alpha\!=\!.05$ and $.10$.}
The Fisher-$z$ normality and $(1/J + n_{2}/n_{1})$ variance correction
of~\citet{nadeau2003inference} are asymptotic: below $n\!=\!25$ the
Fisher-$z$ reference fails whatever the spectrum, and the exact test is
used there. The stable-rank threshold of
\S\ref{sec:methods:validity} is empirical, set on one family of designs,
validated on five others, and errs toward flagging; the failure it flags
is one of shape, not scale, so no corrected plug-in $\rho$ repairs it.
On cleared designs NB-asymptotic holds level at $\alpha\!=\!.05$ and
$.10$; at $\alpha\!=\!.01$ it is unusable on flagged designs
($0.02\times$ to $1.9\times$ nominal) and inflates $1.3$--$1.6\times$
on one cleared pool at $n\!\geq\!80$, so
multi-$\alpha$ sweeps, FDR at small $q$, and anything else that consumes
the full $p$-distribution should use the exact test whatever the design. \textbf{(c) Default $K^{*}$ tracks predictive rank, not $\widetilde{r}$.}
The 1-SE rule on split-half CV-$R^2$ returns the predictively justified
rank, which can sit above $\widetilde{r}$ under correlated $\mathbf{X}$
(Krylov deflation needs $K\!>\!1$ even for rank-1 $\mathbf{y}$;
App.~\ref{app:overshoot}); a shorter prefix preserves FWER whenever
the underlying sequence is valid. Fixing $K$ a priori (we use $K\!=\!3$)
or using BIC chooses a model size; neither establishes
$\widetilde r$. \textbf{(d) Rotation depends on a curated target lexicon.}
The smart-target variant of \S\ref{sec:methods:rotation} is a
researcher degree of freedom: rotation invariance preserves predictive
content, the named Dim-1/Dim-2/Dim-3 split is not, so pre-register the
target and report the unsupervised variant
($\mathcal{V}_{\!\mathrm{tgt}}\!=\!\mathcal{V}$) as sensitivity. \textbf{(e) Scope of extensions.}
Basis-dependent refits (LASSO, elastic-net, anisotropic ridge) and
non-orthogonal rotations break at least one step of the argument
(rotation invariance or calibration-axis derivation) and are out of
scope; isotropic ridge is $O(K)$-invariant and inherits the OLS proof.
Across pipelines the split is measured rather than argued
(\S\ref{sec:exp:beyond}): the exact test transfers to supervised PCA and
a ridge probe under iid rows and outcome--predictor independence; the NB
correction is PLS1-scoped, and sLDA, linear probes on deep features, and
SAE concept directions were not run. \textbf{(f) The NB boundary property is not proved.}
The finite-sample FWER guarantee covers a single independent split
with train-only deflation and Gaussian correlation tests
(App.~\ref{app:fwer-proof}). Repeated-split NB with either deflation
scheme requires a separate calibration argument, and under full-data
deflation that argument must also absorb the train/test coupling the
deflation introduces. The experiments of
\S\ref{sec:exp:fwer} measure rejection beyond planted signal counts,
which need not coincide with the first true conditional null, and
cover at most three deflations. \textbf{(g) Real-data breadth.}
Inferential evidence comes from one text design and two NIR
chemometric designs, the second of which, NIR gasoline, is still
chemometrics. Non-Gaussian outcomes are covered by the exchangeability
argument behind the exact test and by the permuted real outcomes of
the null grids, not by a dedicated study.

\section{Conclusion}
\label{sec:conclusion}

PLS is used widely across applied science, but its inferential
infrastructure has not kept up. Practitioners who want a test that
does not cost a thousand fits, or that returns a calibrated
$p$-value on each component rather than only on the joint
subspace, have had to choose between a tool flagged biased by its
own maintainers, a calibrated but expensive permutation loop, or no
test at all.

We reduced ``does this supervised subspace predict $\mathbf{y}$?'' to
a held-out refit of the subspace, with a corrected asymptotic test, a
pre-run check that says when it holds its level, and an exact test for
when it does not; both reject planted signal at least as often as
permutation-$Q^2$ at a fraction of its cost. Held-out predictions are
invariant to orthogonal rebasing, so the joint claim survives rotation
and per-component claims live on the unrotated sequence. The software
is \inferpkg{}~\citep{plskit2026} and \textpkg{}~\citep{plisiecki2025ssd}.

\begin{ack}
The study was funded by the National Science Centre Poland OPUS grant
(Grant 2024/55/B/HS6/02617). Centre for Brain Research is supported as
a flagship project by Future Society Priority Research Area and Quality
of Life Priority Research Area under the Strategic Program of
Excellence Initiative at the Jagiellonian University. The authors
declare no competing interests.
\end{ack}

\bibliographystyle{plainnat}
\bibliography{refs}


\appendix

\section{Technical appendices and supplementary material}

\paragraph{Contents.}
\begin{itemize}[leftmargin=2.4em,itemsep=2pt]
  \item App.~\ref{app:notation} --- Notation glossary.
  \item App.~\ref{app:proof} --- Rotation invariance of the OLS refit: derivation and assumptions.
  \item App.~\ref{app:nipals} --- NIPALS deflation recurrence.
  \item App.~\ref{app:varimax} --- Loading-side varimax: objective, recipe, score orthogonality.
  \item App.~\ref{app:setup} --- Setup: document vectors, word loadings, and the PCA+OLS baseline.
  \item App.~\ref{app:vad-words} --- Cross-lingual VAD: configurations, aggregate fit, and top words per axis.
  \item App.~\ref{app:dim3-robustness} --- Dim-3 calibration axis: robustness against four counter-hypotheses.
  \item App.~\ref{app:toys} --- Toy 1--5 mechanism experiments.
  \item App.~\ref{app:powercalib-full} --- Calibration and power benchmarks (full).
  \item App.~\ref{app:overshoot} --- $K^{*}$ tracks predictive rank, not $\widetilde{r}$.
  \item App.~\ref{app:embedding-power} --- Detection power on real embedding geometries.
  \item App.~\ref{app:nontext} --- Inference: power and Tecator non-text replication.
  \item App.~\ref{app:rule-grid} --- Validity rule: the full null grid, the \texttt{plsdof} construction, and cost.
  \item App.~\ref{app:fwer-proof} --- Per-component FWER: a single-split guarantee.
  \item App.~\ref{app:peraxis-fwer} --- Per-axis FWER: empirical diagnostics.
  \item App.~\ref{app:beyond-pls} --- Beyond PLS: supervised PCA and a ridge probe.
  \item App.~\ref{app:embeddings} --- Embedding pipelines and vocabulary cleaning.
  \item App.~\ref{app:sparse-pls} --- Sparse-PLS --- per-$(K, s)$ sweep on Warriner-EN valence.
  \item App.~\ref{app:impacts} --- Broader impacts.
  \item App.~\ref{app:repro} --- Reproducibility.
  \item App.~\ref{app:algorithm} --- Algorithm: PLS-$K$ pseudocode for the supervised text pipeline.
\end{itemize}

\subsection{Notation glossary}
\label{app:notation}

We use two letter families across \S\ref{sec:methods}--\S\ref{sec:experiments}:
$r$ for ranks of the supervised signal as observable joint properties of
$(\mathbf{X},\mathbf{y})$, and $K$ for component counts of the PLS estimator.

\begingroup
\setlength{\tabcolsep}{6pt}
\renewcommand{\arraystretch}{1.1}
\noindent\begin{tabular}{@{}l p{0.85\linewidth}@{}}
$\widetilde{r}$ & true signal rank --- the smallest $K$ for which $\mathbb{E}[\mathbf{y}\mid\mathbf{X}]$ factors through a $K$-dimensional linear projection of $\mathbf{X}$. Oracle quantity, unknown in practice. \\
$r_{\mathrm{plant}}$ & planted direction count --- the number of directions used to construct the synthetic response (\S\ref{sec:exp:synth}). In a scalar linear model their weighted sum is one linear index, so this count need not equal $\widetilde r$ or the Krylov dimension. \\
$K_{\mathrm{fit}}$ & number of PLS components extracted at fit time --- a hyperparameter of the estimator. \\
$K^{*}$ & data-driven component count --- the value of $K_{\mathrm{fit}}$ chosen by a predictive selector (1-SE on CV-$R^{2}$) or a parsimony selector (BIC). Neither targets $\widetilde{r}$: in correlated $\mathbf{X}$, predictive $K^{*}\!>\!\widetilde{r}$ generically (App.~\ref{app:overshoot}). \\
$k^{*}$ & per-axis truncation cutoff --- for a single axis, the index of the last PLS direction retained under the per-axis rule of \S\ref{sec:exp:peraxis-glove}. \\
\end{tabular}
\endgroup

\subsection{Rotation invariance of the OLS refit: derivation and assumptions}
\label{app:proof}

We expand the proof of Prop.~\ref{prop:rotation-invariance} (rotation
invariance of the joint $K$-component OLS-refit predictions) into four
numbered points, an explicit assumptions block, a remark on
regularised refits, and a discussion of Krylov dimension. Notation as in
\S\ref{sec:methods:test:tnb}. The per-axis pipeline of
\S\ref{sec:methods:test:tnb} ($t_{\mathrm{NB},h}$ on deflated pairs) is
rotation-invariant operationally rather than mathematically: per-axis
tests are run on the unrotated NIPALS components, and the loading-side
varimax of \S\ref{sec:methods:rotation} is applied post-inferentially,
so the rotation choice does not enter the test inputs.

\paragraph{Assumptions.}
\begin{enumerate}[label=(A\arabic*),leftmargin=2.4em,itemsep=2pt]
  \item Every training-score matrix
        $\mathbf{T}_{\mathrm{tr},j}$ has full column rank $K$, so
        $\mathbf{T}_{\mathrm{tr},j}^{\!\top}\mathbf{T}_{\mathrm{tr},j}$
        is invertible on every split.
  \item $\mathbf{Q}\!\in\!O(K)$ is fixed across splits --- varimax
        on the full-data loading matrix yields a single $\mathbf{Q}$
        that is reused identically inside every split. This matches
        our implementation; the identity also holds for a
        data-dependent $\mathbf Q_j$ applied to both training and test
        scores within split $j$, leaving the statistic unchanged.
  \item $\mathbb{E}[\mathbf{y}^{2}]\!<\!\infty$ (standard for the
        $t$-test reference).
  \item The split schedule is exchangeable: each fold partition is
        drawn iid from the uniform distribution over 50/50
        partitions of $\{1,\dots,n\}$. This is the split design used
        for the NB approximation, whose correction uses
        $\rho_0/(1-\rho_0)=n_2/n_1$ with
        $\rho_0=n_2/(n_1+n_2)$~\citep{nadeau2003inference}.
\end{enumerate}
Only (A1) and the within-split basis change in (A2) are needed for
the algebraic identity below; (A3)--(A4) describe the testing setup
and do not by themselves establish NB calibration.

\paragraph{Per-split derivation.}
Write $\widetilde{\mathbf{T}}\!:=\!\mathbf{T}\mathbf{Q}$ and fix a
split $j$ with training/held-out partition
$(\mathcal{I}_{j}^{\mathrm{tr}},\mathcal{I}_{j}^{\mathrm{te}})$. The
OLS refit on rotated training scores satisfies
\begin{equation*}
\widetilde{\boldsymbol\beta}_{j}
\;=\;
\bigl(\mathbf{Q}^{\!\top}\mathbf{T}_{\mathrm{tr},j}^{\!\top}
      \mathbf{T}_{\mathrm{tr},j}\mathbf{Q}\bigr)^{-1}
\mathbf{Q}^{\!\top}\mathbf{T}_{\mathrm{tr},j}^{\!\top}
\mathbf{y}_{\mathrm{tr},j}
\;=\;
\mathbf{Q}^{\!\top}\boldsymbol\beta_{j},
\end{equation*}
where the second equality uses
$(\mathbf{Q}^{\!\top}\mathbf{A}\mathbf{Q})^{-1}
 \!=\!\mathbf{Q}^{\!\top}\mathbf{A}^{-1}\mathbf{Q}$
for $\mathbf{Q}\!\in\!O(K)$. Held-out predictions are then identical
pointwise,
\begin{equation*}
\widetilde{\hat{\mathbf{y}}}_{\mathrm{te},j}
\;=\;
\widetilde{\mathbf{T}}_{\mathrm{te},j}\widetilde{\boldsymbol\beta}_{j}
\;=\;
\mathbf{T}_{\mathrm{te},j}\mathbf{Q}\mathbf{Q}^{\!\top}\boldsymbol\beta_{j}
\;=\;
\hat{\mathbf{y}}_{\mathrm{te},j},
\end{equation*}
so the per-split test correlation
$r_{j}\!=\!\operatorname{Cor}(\hat{\mathbf{y}}_{\mathrm{te},j},
                              \mathbf{y}_{\mathrm{te},j})$,
its Fisher transform $z_{j}\!=\!\operatorname{arctanh}(r_{j})$, and
the NB statistic in Eq.~\eqref{eq:nb} (a function of
$\{z_{j}\}_{j=1}^{J}$ alone, with a deterministic
$(1/J\!+\!n_{2}/n_{1})$ rescale) are invariant under $\mathbf{Q}$.
\hfill$\square$

\paragraph{Remark (failure mode under regularised refit).}
Rotation invariance is exact for the OLS refit and survives any
regulariser whose penalty is itself $O(K)$-invariant: isotropic
ridge ($\lambda\|\boldsymbol\beta\|_{2}^{2}$) is in this class because
$\mathbf{Q}^{\!\top}\mathbf{A}\mathbf{Q}\!+\!\lambda\mathbf{I}
=\mathbf{Q}^{\!\top}(\mathbf{A}\!+\!\lambda\mathbf{I})\mathbf{Q}$,
which gives
$\widetilde{\boldsymbol\beta}_{j}\!=\!\mathbf{Q}^{\!\top}\boldsymbol\beta_{j}$
and pointwise-equal predictions exactly as in the OLS case.
Invariance breaks once the penalty is basis-dependent: LASSO and
elastic-net use $\|\boldsymbol\beta\|_{1}$, which is not preserved
under $\mathbf{Q}\!\in\!O(K)$ ($\|\mathbf{Q}^{\!\top}\boldsymbol\beta\|_{1}\!\neq\!\|\boldsymbol\beta\|_{1}$
in general), and an anisotropic ridge with penalty matrix
$\boldsymbol\Lambda\!\neq\!\lambda\mathbf{I}$ inherits the same
basis-dependence. For these regularisers the equality
$\widetilde{\boldsymbol\beta}_{j}\!=\!\mathbf{Q}^{\!\top}\boldsymbol\beta_{j}$
breaks and $t_{\mathrm{NB}}$ is only \emph{approximately} invariant
in the small-penalty limit. Practitioners using basis-dependent
regularisers should compute the NB statistic on the unrotated basis
and report the rotated basis as an interpretation device only --- or
use a rotation-equivariant penalty (isotropic ridge, or group
sparsity over rotated columns). The NB-OLS pairing in this paper
avoids the issue by construction.

\paragraph{What invariance establishes.}
The rotated and unrotated statistics are equal for every dataset and
split schedule, so they have the same null distribution. Any
calibration result for one therefore transfers to the other.
This identity does not establish the $t_{J-1}$ approximation or a
convergence rate: NB uses a plug-in dependence correction
\citep{nadeau2003inference}, whose calibration for the present
statistic is assessed in \S\ref{sec:methods:validity}.

\paragraph{Numerical confirmation (Toy~4).}
On a planted three-axis ground truth with $R^{2}\!=\!0.7$ on real
Warriner geometry,
$|p_{\mathrm{unrot}}-p_{\mathrm{varimax}}|\!=\!0$ (machine zero) and
$|p_{\mathrm{unrot}}-p_{\mathrm{rand{-}orth}}|\!\sim\!10^{-78}$
against a baseline $p\!\sim\!10^{-65}$, across 50 random orthogonal
$\mathbf{Q}$ (Tab.~\ref{tab:toy4_rotation}). Per-seed CSVs in
\texttt{toys/results/toy4\_rotation/}.

\paragraph{Calibration axis: definition and Krylov dimension.}
The body summarizes this construct in \S\ref{sec:methods:plsK};
the definition and its relation to Krylov dimension are below.

\begin{definition}[Calibration axis --- working diagnostic]
\label{def:calibration-axis}
Under univariate $\mathbf{y}$, we call a PLS direction $\mathbf{w}_j$
extracted by deflation a \emph{calibration axis} when, on the data at
hand, the cosine
$|\langle\mathbf{w}_j,\hat{\boldsymbol\beta}_{\mathrm{OLS}}\rangle|/\|\hat{\boldsymbol\beta}_{\mathrm{OLS}}\|$
is small relative to that of the leading axis. NIPALS chooses
$\mathbf{w}_j$ from the deflated predictor--outcome covariance;
alignment with a leading singular vector of the residual design is
not guaranteed. The cosine condition is an empirical diagnostic,
which we report on
real configurations (\S\ref{sec:exp:peraxis-glove}).
\end{definition}

\paragraph{Signal dimension and Krylov saturation are distinct.}
For centered variables with finite second moments, let
$\Sigma_{\mathbf X}\!=\!\operatorname{Cov}(\mathbf x)$ be positive
definite and $\mathbf c\!=\!\operatorname{Cov}(\mathbf x,y)\ne0$.
The population Krylov sequence $\mathcal K_K(\Sigma_{\mathbf X},\mathbf c)$
stabilizes at the number of distinct eigenvalues of $\Sigma_{\mathbf X}$
whose eigenspaces have nonzero projection of $\mathbf c$; the resulting
space contains $\boldsymbol\beta_{\mathrm{OLS}}\!=\!\Sigma_{\mathbf X}^{-1}\mathbf c$
\citep{helland1988pls,phatak2002krylov}.
This number need not equal $\widetilde r$. For example, with
$\Sigma_{\mathbf X}=\left(\begin{smallmatrix}1&1/2\\1/2&1\end{smallmatrix}\right)$
and $y=x_1+\varepsilon$ for independent, mean-zero noise,
$\widetilde r=1$, but $\mathbf c=(1,1/2)^{\top}$ and
$\Sigma_{\mathbf X}\mathbf c=(5/4,1)^{\top}$ are linearly independent,
so saturation requires two components. Signal dimension alone therefore
implies neither saturation at $K=\widetilde r$ nor vanishing projection
of subsequent weights onto the OLS direction. The calibration-axis
reading rests on the finite-sample diagnostics in Toy~1
(Fig.~\ref{fig:calibration-axis}).

\subsection{NIPALS deflation recurrence}
\label{app:nipals}

Concretely, the iteration of Eq.~\eqref{eq:pls1-weight} with
deflation. Initialise $\mathbf{X}_{0}=\mathbf{X}$,
$\mathbf{y}_{0}=\mathbf{y}$. For $k=1,\dots,K$,
\begin{align}
    \mathbf{w}_k &\;\propto\; \mathbf{X}_{k-1}^{\!\top}\,\mathbf{y}_{k-1},
        & \|\mathbf{w}_k\|&=1, \nonumber\\
    \mathbf{t}_k &\;=\; \mathbf{X}_{k-1}\mathbf{w}_k,
        & q_k &\;=\; \frac{\mathbf{t}_k^{\!\top}\mathbf{y}_{k-1}}{\mathbf{t}_k^{\!\top}\mathbf{t}_k},
        \label{eq:plsK-step}\\
    \mathbf{p}_k &\;=\; \frac{\mathbf{X}_{k-1}^{\!\top}\mathbf{t}_k}{\mathbf{t}_k^{\!\top}\mathbf{t}_k},
        & \mathbf{X}_{k} &\;=\; \mathbf{X}_{k-1}-\mathbf{t}_k\mathbf{p}_k^{\!\top},
        \nonumber\\
    & & \mathbf{y}_{k} &\;=\; \mathbf{y}_{k-1}-q_k\,\mathbf{t}_k.\nonumber
\end{align}
NIPALS scores are mutually $\mathbf{X}$-orthogonal,
$\mathbf{T}^{\!\top}\mathbf{T}\!=\!\operatorname{diag}\bigl(\|\mathbf{t}_1\|^{2},\dots,\|\mathbf{t}_K\|^{2}\bigr)$,
which makes the per-axis explained outcome variance partition
additively on the unrotated scores~\citep{hoskuldsson1988pls}; this
fact is invoked by the loading-side varimax recipe of
App.~\ref{app:varimax}.

\subsection{Loading-side varimax: objective, recipe, score orthogonality}
\label{app:varimax}

The body (\S\ref{sec:methods:rotation}) summarizes Kaiser-normalized
varimax against vocabulary target $\mathcal{V}_{\!\mathrm{tgt}}$ as a
post-inferential basis change; the explicit objective, sign/sort
recipe, and the score-orthogonality bookkeeping behind the joint vs.\
per-axis $R^{2}$ statement are below.

\paragraph{Objective.} Rows of the loading matrix
$\mathbf{L}\!=\!\mathbf{E}_{\!\mathcal{V}_{\!\mathrm{tgt}}}\mathbf{W}\!\in\!\mathbb{R}^{V_{\!\mathrm{tgt}}\times K}$
are first $\ell_{2}$-normalized to
$\widetilde{\mathbf{L}}_{w}\!=\!\mathbf{L}_{w}/\|\mathbf{L}_{w}\|_{2}$
(with a $10^{-12}$ floor on the norm), and the rotation
\begin{equation}
    \mathbf{R}^{\!\star}
    \;=\;\arg\max_{\mathbf{R}\in O(K)}
    \sum_{k=1}^{K}\!\Bigl[\,
        \tfrac{1}{V_{\!\mathrm{tgt}}}\!\sum_{w}(\widetilde{\mathbf{L}}\mathbf{R})_{wk}^{4}
        \;-\;
        \bigl(\tfrac{1}{V_{\!\mathrm{tgt}}}\!\sum_{w}(\widetilde{\mathbf{L}}\mathbf{R})_{wk}^{2}\bigr)^{\!2}
    \,\Bigr]
    \label{eq:varimax}
\end{equation}
is then applied to the original (un-normalized) $\mathbf{L}$,
$\mathbf{W}$, and $\mathbf{T}$, producing rotated weights
$\mathbf{W}^{\!\star}\!=\!\mathbf{W}\mathbf{R}^{\!\star}$ and word
loadings $\mathbf{L}^{\!\star}\!=\!\mathbf{L}\mathbf{R}^{\!\star}$.
The objective is invariant under signed permutations of columns;
other optima may also exist. For the solution returned by the
optimizer, we fix column order and signs by sorting axes on
$|\!\operatorname{cor}(\mathbf{t}^{\!\star}_{k},\mathbf{y})|$ and
sign-flipping so the correlation is positive. Each rotated axis loads
on a small set of describable tokens, the column the analyst actually
consults, so every dimension is verbally describable.

\paragraph{Score orthogonality is not preserved by rotation.}
Because $\mathbf{R}^{\!\star}\!\in\!O(K)$, the Krylov-span identity
\eqref{eq:krylov-span} survives the rebasing
($\operatorname{span}(\mathbf{W}\mathbf{R}^{\!\star})\!=\!\operatorname{span}(\mathbf{W})$),
so the joint prediction
$\hat{\mathbf y}\!=\!\mathbf{T}\boldsymbol\beta\!=\!\mathbf{T}^{\!\star}\boldsymbol\beta^{\!\star}$
and $R^{2}_{\mathrm{full}}$ are unchanged. Score orthogonality, by
contrast, is \emph{not} preserved: under
$\mathbf{R}^{\!\star}\!\in\!O(K)$,
$(\mathbf{T}^{\!\star})^{\!\top}\mathbf{T}^{\!\star}
\!=\!{\mathbf{R}^{\!\star}}^{\!\top}\!\operatorname{diag}(\|\mathbf t_k\|^{2})\,\mathbf{R}^{\!\star}$
is diagonal precisely when the columns of $\mathbf{R}^{\!\star}$
lie within eigenspaces of the diagonal score-norm matrix. Thus
signed permutations and arbitrary rotations within blocks of equal
score norms preserve orthogonality; mixing unequal-norm blocks
generally does not. The additive partition
\begin{equation}
    R^{2}_{\mathrm{full}}
    \;=\;\sum_{k=1}^{K}\frac{(\mathbf{t}_{k}^{\!\top}\mathbf{y})^{2}}
              {\|\mathbf{t}_{k}\|^{2}\,\|\mathbf{y}\|^{2}}
    \label{eq:r2-partition}
\end{equation}
therefore holds on the \emph{unrotated} NIPALS scores $\mathbf{t}_k$,
where App.~\ref{app:nipals} gives
$\mathbf{T}^{\!\top}\mathbf{T}\!=\!\operatorname{diag}(\|\mathbf t_k\|^{2})$;
on rotated scores the same per-axis terms generically overshoot by
cross-terms. Per-axis $R^{2}$, when reported, is therefore on the
unrotated basis; rotated axes inherit only the joint
$R^{2}_{\mathrm{full}}$.

\subsection{Setup: document vectors, word loadings, and the PCA+OLS baseline}
\label{app:setup}

This appendix carries the representation recipe that
\S\ref{sec:methods:representations} summarizes in three lines.
We represent texts as points in embedding space by mean-pooling token
vectors, so word loadings and document scores share a single metric
and word-level interpretation of any downstream axis is well-defined.
Let $\mathcal{V}=\{w_1,\dots,w_V\}$ be the vocabulary
and $\mathbf{E}\!\in\!\mathbb{R}^{V\times d}$ a static word-embedding
matrix whose rows $\mathbf{e}_w\!\in\!\mathbb{R}^{d}$ encode each
token. For a corpus of $n$ documents indexed by $i$ with token sequence
$T_i\!\subseteq\!\mathcal{V}$, we form a document vector by mean
pooling,
\begin{equation}
    \mathbf{x}_i \;=\; \frac{1}{|T_i|}\!\sum_{w\in T_i}\mathbf{e}_w,
    \qquad \mathbf{x}_i\in\mathbb{R}^{d}.
    \label{eq:doc-vector}
\end{equation}
Stacking
row-wise gives $\mathbf{X}\!\in\!\mathbb{R}^{n\times d}$, and we let
$\mathbf{y}\!\in\!\mathbb{R}^{n}$ collect the numeric outcome.
$\mathbf{X}$ and $\mathbf{y}$ are mean centered prior to fitting, and
embedding rows are $\ell_2$-normalized so that all geometry is
interpretable as cosine similarity.

Given any axis $\mathbf{a}\!\in\!\mathbb{R}^{d}$, we read its content
from word loadings
$\ell_w\!=\!\mathbf{e}_{w}^{\!\top}\mathbf{a}/\|\mathbf{a}\|$
(extreme-token tables, contrast snippets, document scores via the
same inner product). The 1-D baseline applies this to
$\hat{\boldsymbol\beta}_{\mathrm{PCA}}$, the direction obtained by
projecting $\mathbf{X}$ onto its leading $K$ PCs and fitting OLS in
that subspace with $K$ chosen to maximize out-of-sample $R^{2}$. We
use it as the 1-D reference point and apply the same recipe per
rotated axis $\mathbf{w}^{\!\star}_{k}$ recovered in
\S\ref{sec:methods:rotation}; the joint $R^{2}_{\mathrm{full}}$ is
rotation-invariant, and the additive per-axis partition holds on the
unrotated NIPALS scores (App.~\ref{app:varimax}).

\paragraph{Prior work: OLS-on-embeddings supervised projection.}
\citet{plisiecki2025ssd} introduced an OLS-on-embeddings
supervised projection: regress $\mathbf{y}$ on $\mathbf{X}$ in the
leading-$K$ PC subspace of $\mathbf{X}$ and read the resulting axis
$\hat{\boldsymbol\beta}_{\mathrm{PCA}}$ via word loadings
$\mathbf{e}_{w}^{\!\top}\hat{\boldsymbol\beta}_{\mathrm{PCA}}/\|\hat{\boldsymbol\beta}_{\mathrm{PCA}}\|$.
The construction is lexicon-free (no seed-word anchors) and outcome-supervised, but
its basis is variance-ordered rather than covariance-ordered with $\mathbf{y}$.
\S\ref{sec:methods:pls} replaces the PCA basis with PLS.

\subsection{Cross-lingual VAD: configurations, aggregate fit, and top words per axis}
\label{app:vad-words}

\S\ref{sec:exp:peraxis-glove} reports the cross-lingual result in one
paragraph; this appendix carries the configurations, the fit statistics,
and the word lists behind it.

\paragraph{Configurations and aggregate fit.}
We fit three (language $\times$ embedding) configurations
(Table~\ref{tab:configs}): English GloVe-300 with Warriner V/A/D
ratings, Polish GloVe-800 with ANPW V/A/D, and Spanish GloVe-300 with
Stadthagen-Gonz\'alez V/A. Embedding provenance and vocabulary cleaning
are in App.~\ref{app:embeddings}. Aggregate fit is in
Table~\ref{tab:vadfit}: $\alpha^{*}_{\mathrm{FWER}}\!<\!.001$ in every
cell (per-axis NB on the unrotated $K\!=\!3$ NIPALS extraction;
rotation-invariant per Prop.~\ref{prop:rotation-invariance}), and the
leading rotated axis is a clean affect direction throughout. Spanish
has no rated dominance dimension, so the ES analysis is restricted to V
and A.

\begin{table}[H]
\centering
\caption{Three (language $\times$ embedding) configurations.
Config~1 uses Warriner V/A/D, Config~2 uses ANPW (Polish
counterpart), and Config~3 uses the
Stadthagen-Gonz\'alez~\citep{stadthagen2017spanishnorms} valence/arousal
norms (Spanish has no rated dominance dimension). All fits use
$K\!=\!3$ with varimax on the loading matrix; the rotation target
is the full embedding vocabulary intersected with the rated lexicon
(\S\ref{sec:methods:rotation}, App.~\ref{app:embeddings}).}
\label{tab:configs}
\small
\begin{tabular}{@{}clllr@{}}
\toprule
\textbf{Config} & \textbf{Language} & \textbf{Embedding} & \textbf{Outcome} & $n$ \\
\midrule
1 & English & GloVe-300 (Common Crawl 42B, normalized) & Warriner V/A/D            & 13{,}365 \\
2 & Polish  & GloVe-800 (Dadas, normalized)            & ANPW V/A/D                & 4{,}526  \\
3 & Spanish & GloVe-300 (SBWCE, normalized)            & Stadthagen-Gonz\'alez V/A & 12{,}639 \\
\bottomrule
\end{tabular}
\end{table}

\begin{table}[H]
\centering
\caption{Aggregate fit on the three configurations, $K\!=\!3$, varimax on
loadings. For each rated dimension the first figure is the in-sample
(resubstitution) $r^{2}$ of the full-data fit and the second, in
parentheses, the split-half cross-validated $R^{2}$ at $K\!=\!3$
(mean over the $J\!=\!50$ NB splits;
\texttt{kselect/selection\_sequence.csv}). The largest gap is PL
dominance, 0.611 in-sample against 0.557 cross-validated.
$\alpha^{*}_{\mathrm{FWER}}\!<\!.001$ in every cell (per-axis NB
on the unrotated NIPALS extraction, $J\!=\!50$;
Prop.~\ref{prop:rotation-invariance}). FWER control for this NB
summary remains conditional on per-step validity
(\S\ref{sec:methods:test}); the single-split guarantee in
App.~\ref{app:fwer-proof} does not cover this repeated-split analysis.
Spanish dominance is not reported
because the Stadthagen-Gonz\'alez norms cover valence and arousal only.}
\label{tab:vadfit}
\small
\begin{tabular}{@{}lccc@{}}
\toprule
\textbf{Config}    & \textbf{Valence $r^2$ (CV)} & \textbf{Arousal $r^2$ (CV)} & \textbf{Dominance $r^2$ (CV)} \\
\midrule
EN GloVe-300       & 0.631 (0.619) & 0.370 (0.350) & 0.486 (0.472) \\
PL GloVe-800       & 0.692 (0.652) & 0.617 (0.567) & 0.611 (0.557) \\
ES GloVe-300       & 0.551 (0.539) & 0.447 (0.433) & ---           \\
\bottomrule
\end{tabular}
\end{table}

\paragraph{What the rotated axes load on.}
Rotated Dim-1 recovers the conventional 1-D supervised affect
direction in every V/A/D cell of Configs~1--2 (e.g.\ EN valence
anchored by \texttt{elegant, lovely, fabulous}) and in the V/A
cells of Config~3 (ES valence: \emph{felicidad, feliz, sonrisa,
alegr\'ia, hermosa}; ES arousal: \emph{agresi\'on, explosi\'on,
atentado, masacre}); top-10 lists per pole and per dimension are
in Tables~\ref{tab:vad-combined}--\ref{tab:vad-dim-3} below.
What the calibration axes load on is consistent across languages.
In every cell, at least one $K\!>\!1$ rotated axis carries the same
register contrast: one pole concentrates technical, inanimate, or
descriptive content (EN \emph{industrial, containing, metal}), the
other emotion-laden, morally charged, or declarative content (EN
\emph{blessings, faithful, compassion}). Mechanism, PL/ES top-loading
lists, the four alternative readings ruled out, and the
$k^{*}\!=\!2$ truncation note are in App.~\ref{app:dim3-robustness};
Toy~2 (Fig.~\ref{fig:calibration-axis}C) reproduces the synthetic
counterpart. The interpretive frame is taken up in
\S\ref{sec:disc:scope}.

\paragraph{Sparser axes via direct cardinality.}
Hard-cardinality sparse-PLS (top-$s$ deflation, no rotation) is a
parallel route to readability that yields axes literally supported
on $\le\!s$ features. App.~\ref{app:sparse-pls} reports a
per-$(K, s)$ sweep on Warriner-EN valence; the sparse leading
axis at $s\!=\!100$ ranks ``wonderful, elegant, lovely, \dots''
vs.\ ``homicidal, hateful, inhumane, \dots'', read directly off
$\mathbf{w}_1$.

\paragraph{Top words per axis.}
Table~\ref{tab:vad-combined} gives the top-10 tokens per pole
\emph{across} the three rotated axes (``Combined'' loadings) for
every (rated dimension $\times$ language) cell of Configs~1, 2,
and~3 side-by-side --- a single omnibus leaderboard. The per-axis
breakouts at Dim-1, Dim-2, Dim-3 follow in
Tabs.~\ref{tab:vad-dim-1}--\ref{tab:vad-dim-3} with the same row
layout. EN columns are Config~1 (Warriner GloVe-300); PL columns
are Config~2 (ANPW GloVe-800); ES columns are Config~3
(Stadthagen-Gonz\'alez + SBWCE GloVe-300). Polish and Spanish
cells use a uniform \emph{english (native)} format; English
glosses for Polish are taken from the \texttt{word\_en} field of
\citet{imbir2016anpw} when available and otherwise from Google
Translate (cached in \texttt{applied/\_pl\_en\_translations.json});
Spanish glosses come from Google Translate (cached in
\texttt{applied/\_es\_en\_translations.json}). The
Stadthagen-Gonz\'alez ES norms~\citep{stadthagen2017spanishnorms}
do not include a dominance dimension, so the Dominance group has
EN and PL only. Dim-1 is the conventional 1-D supervised affect
direction in every cell; the cross-lingual Dim-3 calibration story
discussed in \S\ref{sec:exp:peraxis-glove} is visible in the
Valence column group of Tab.~\ref{tab:vad-dim-3}, where one pole
loads on technical/inanimate vocabulary (\emph{industrial,
containing, multiple, storage} on EN; \emph{millimeter, diameter,
marked} on PL; \emph{tubo, roca, madera, vidrio} on ES) and the
other pole on emotion- or commitment-laden vocabulary across all
three languages.

For Config~4 (EN XLM-R type-level) only the leading rotated axis
reads as an affect direction; Dim-2 and Dim-3 surface stable
topical / part-of-speech subspaces (a food/sensory cluster and an
abstract-evaluative-vs-concrete-noun cluster) that recur across
V/A/D with sign flips and that survive rank-1 ABTT. We do not include
the per-axis XLM-R top-word tables in the printed appendix because,
beyond Dim-1, the per-dimension decomposition recovers topical /
part-of-speech structure rather than additional affect axes; the
underlying CSVs ship with the replication package
(\texttt{results/xlmr\_warriner\_rotated/csvs/}).


\subsubsection*{Combined leaderboard (across all three rotated axes)}

\begin{table}[H]
  \centering
  \scriptsize
  \setlength{\tabcolsep}{3pt}
  \caption{Top 10 words per side at the \textbf{Combined} loadings for every (rated dimension $\times$ language) cell of Configs 1, 2, 3 (varimax-rotated PLS, $k=3$). Polish and Spanish cells are formatted uniformly as \emph{english (native)}; Polish glosses come from the \texttt{word\_en} field of \citet{imbir2016anpw} when available and from Google Translate otherwise, Spanish glosses from Google Translate (caches under \texttt{applied/\_*\_en\_translations.json}). The Stadthagen-Gonz\'alez ES norms do not include a dominance dimension, so Dominance has EN/PL only.}
  \label{tab:vad-combined}
  \resizebox{\textwidth}{!}{%
  \begin{tabular}{rl c c c c c c c c}
    \toprule
     &  & \multicolumn{3}{c}{\textbf{Valence}} & \multicolumn{3}{c}{\textbf{Arousal}} & \multicolumn{2}{c}{\textbf{Dominance}} \\
    \cmidrule(lr){3-5} \cmidrule(lr){6-8} \cmidrule(lr){9-10}
    Rank & Side & EN & PL & ES & EN & PL & ES & EN & PL \\
    \midrule
    1 & pos & wonderful & terrific (wspaniały) & enjoy (disfrutar) & violent & brutal / violent (brutalny) & aggression (agresión) & strive & terrific (wspaniały) \\
    2 & pos & enjoy & admire (podziwiać) & pleasant (agradable) & deadly & lust (żądza) & murders (asesinatos) & appreciate & bestow (obdarzyć) \\
    3 & pos & appreciation & wonderful (cudowny) & creativity (creatividad) & brutal & frenzy (szał) & murder (asesinato) & appreciated & admire (podziwiać) \\
    4 & pos & delightful & excellent (wspaniale) & gift (regalo) & insane & terrible (straszliwy) & deadly (mortal) & exceptional & excellent (wspaniale) \\
    5 & pos & fantastic & delight (zachwycić) & beautiful (hermosa) & explosion & cruelty (okrucieństwo) & violent (violento) & dedication & excellent (znakomity) \\
    6 & pos & enjoying & joyful (radosny) & happiness (felicidad) & revenge & rape (gwałt) & attack (ataque) & excellent & mighty / powerful (potężny) \\
    7 & pos & fabulous & praise (pochwalić) & talent (talento) & assault & monstrous (potworny) & attack (atentado) & excellence & skill (umiejętność) \\
    8 & pos & lovely & grace (wdzięk) & happy (feliz) & orgasm & furious (wściekły) & violent (violenta) & enjoy & praise (pochwalić) \\
    9 & pos & enjoyed & thank (dziękować) & beauty (belleza) & shocking & hatred (nienawiść) & assaults (agresiones) & enhance & unusual (niezwykły) \\
    10 & pos & exceptional & beautiful (piękny) & welcome (bienvenida) & hardcore & murder (mord) & attacks (ataques) & wonderful & great (świetny) \\
    \midrule
    1 & neg & worse & horrible (okropny) & serious (graves) & gray & arrange (ułożyć) & gardens (jardines) & caused & pitiful (żałosny) \\
    2 & neg & threatening & monstrous (potworny) & assaults (agresiones) & cottage & sit down (usiąść) & elegant (elegante) & deaths & depressed (przygnębiony) \\
    3 & neg & violent & indictment (oskarżenie) & tortures (torturas) & circular & bowl (miska) & prayer (oración) & severe & helpless (bezradny) \\
    4 & neg & horrible & hideous (ohydny) & murders (asesinatos) & cloth & lay out (wyłożyć) & soft (suave) & threatening & hopeless (beznadziejny) \\
    5 & neg & ugly & impotent (bezsilny) & serious (grave) & container & round (okrągły) & terrace (terraza) & fatal & wretched (nędzny) \\
    6 & neg & severe & accuse (oskarżać) & reported (denunciado) & shelf & arranged (ułożony) & harmony (armonía) & worse & pitifully (żałośnie) \\
    7 & neg & racist & disgusting (obrzydliwy) & executions (ejecuciones) & beside & flat (płaski) & garden (jardín) & causing & misery / hardship (nędza) \\
    8 & neg & nasty & disgusting (wstrętny) & insufficiency (insuficiencia) & shade & transparent (przejrzysty) & leaf (hoja) & horrible & depression (depresja) \\
    9 & neg & abusive & misery / hardship (nędza) & murder (asesinato) & sits & warm (ciepły) & landscape (paisaje) & nightmare & stinking (cuchnący) \\
    10 & neg & criminals & stench (smród) & harassment (acoso) & rows & moss (mech) & seated (sentado) & trapped & drag (wlec) \\
    \bottomrule
  \end{tabular}}
\end{table}

\subsubsection*{Per-axis detail --- Dim-1}

\begin{table}[H]
  \centering
  \scriptsize
  \setlength{\tabcolsep}{3pt}
  \caption{Top 10 words per side at the \textbf{Dim-1} loadings, same conventions as Tab.~\ref{tab:vad-combined}.}
  \label{tab:vad-dim-1}
  \resizebox{\textwidth}{!}{%
  \begin{tabular}{rl c c c c c c c c}
    \toprule
     &  & \multicolumn{3}{c}{\textbf{Valence}} & \multicolumn{3}{c}{\textbf{Arousal}} & \multicolumn{2}{c}{\textbf{Dominance}} \\
    \cmidrule(lr){3-5} \cmidrule(lr){6-8} \cmidrule(lr){9-10}
    Rank & Side & EN & PL & ES & EN & PL & ES & EN & PL \\
    \midrule
    1 & pos & elegant & happily (radośnie) & happiness (felicidad) & violent & drug (narkotyk) & aggression (agresión) & elegant & excellent (wspaniale) \\
    2 & pos & lovely & wonderfully (cudownie) & baby (bebé) & brutal & sexual (seksualny) & explosion (explosión) & wonderful & bestow (obdarzyć) \\
    3 & pos & fabulous & excellent (wspaniale) & happy (feliz) & deadly & sex/ intercourse (seks) & attack (atentado) & fabulous & mighty / powerful (potężny) \\
    4 & pos & delightful & wonderful (cudowny) & smile (sonrisa) & revenge & explosive (wybuchowy) & murders (asesinatos) & lovely & terrific (wspaniały) \\
    5 & pos & exquisite & to squirt (tryskać) & enjoy (disfrutar) & rage & mass (masowy) & assaults (agresiones) & delightful & capable (zdolny) \\
    6 & pos & wonderful & warmth (ciepło) & mum (mamá) & terror & explosion (eksplozja) & infection (infección) & stylish & praise (pochwalić) \\
    7 & pos & gorgeous & thanks (dzięki) & joy (alegría) & insane & alcohol (alkohol) & massacre (masacre) & elegance & appreciate (docenić) \\
    8 & pos & elegance & cuddle / hug (przytulić) & beautiful (hermosa) & violence & brutal / violent (brutalny) & slaughter (matanza) & appreciate & truly (prawdziwie) \\
    9 & pos & beautiful & terrific (wspaniały) & gift (regalo) & assault & popularity (popularność) & attacks (ataques) & superb & power (potęga) \\
    10 & pos & stylish & joyful (radosny) & pleasant (agradable) & killing & explode (wybuchać) & theft (robo) & fantastic & lust (żądza) \\
    \midrule
    1 & neg & violations & depressed (przygnębiony) & executions (ejecuciones) & cottage & honest (poczciwy) & calm (tranquila) & deaths & pitiful (żałosny) \\
    2 & neg & preventing & to suffer (cierpieć) & abandonment (abandono) & tray & patient (cierpliwy) & seated (sentado) & caused & pitifully (żałośnie) \\
    3 & neg & causing & reason (powód) & murders (asesinatos) & linen & patiently (cierpliwie) & calm (tranquilo) & severe & helpless (bezradny) \\
    4 & neg & injuries & bereavement (żałoba) & assaults (agresiones) & container & reverie (zaduma) & calm (calma) & suspected & depressed (przygnębiony) \\
    5 & neg & violation & death (zgon) & serious (graves) & decorative & sit (sieść) & prayer (oración) & fatal & wretched (nędzny) \\
    6 & neg & caused & condemnation (potępienie) & deterioration (deterioro) & shelf & sit down (usiąść) & tranquility (tranquilidad) & infections & drag (wlec) \\
    7 & neg & suspected & misery / hardship (nędza) & thefts (robos) & drawer & thoughtful (zamyślony) & bed (cama) & causing & derelict (opuszczony) \\
    8 & neg & prevent & indictment (oskarżenie) & delays (retrasos) & furnished & lazily (leniwie) & soft (suave) & chronic & insecure (niepewny) \\
    9 & neg & damage & cause (przyczyna) & insufficiency (insuficiencia) & hardwood & politely (grzecznie) & smile (sonrisa) & infection & weak (marny) \\
    10 & neg & damages & complaint (skarga) & tortures (torturas) & circular & speak (przemówić) & grandma (abuela) & toxic & sad (smutny) \\
    \bottomrule
  \end{tabular}}
\end{table}

\subsubsection*{Per-axis detail --- Dim-2}

\begin{table}[H]
  \centering
  \scriptsize
  \setlength{\tabcolsep}{3pt}
  \caption{Top 10 words per side at the \textbf{Dim-2} loadings, same conventions as Tab.~\ref{tab:vad-combined}.}
  \label{tab:vad-dim-2}
  \resizebox{\textwidth}{!}{%
  \begin{tabular}{rl c c c c c c c c}
    \toprule
     &  & \multicolumn{3}{c}{\textbf{Valence}} & \multicolumn{3}{c}{\textbf{Arousal}} & \multicolumn{2}{c}{\textbf{Dominance}} \\
    \cmidrule(lr){3-5} \cmidrule(lr){6-8} \cmidrule(lr){9-10}
    Rank & Side & EN & PL & ES & EN & PL & ES & EN & PL \\
    \midrule
    1 & pos & provide & modest (skromny) & decor (decoración) & ignorant & show (okazywać) & terrible (terrible) & monitoring & decorate (ozdobić) \\
    2 & pos & provided & beautiful (piękny) & design (diseño) & jealous & justify (usprawiedliwić) & deadly (mortal) & component & silver (srebrny) \\
    3 & pos & ensure & beauty (uroda) & ceramic (cerámica) & forgive & willing (skłonny) & guilt (culpa) & components & slender (smukły) \\
    4 & pos & provides & to like (cenić) & architecture (arquitectura) & cried & feelings (uczucia) & kill (matar) & evaluation & gold (złocisty) \\
    5 & pos & providing & admire (podziwiać) & designs (diseños) & rude & cruelty (okrucieństwo) & revenge (venganza) & provides & adorn (zdobić) \\
    6 & pos & optimal & enjoy (cieszyć) & style (estilo) & thou & contempt / scorn (pogarda) & pain (dolor) & provide & admire (podziwiać) \\
    7 & pos & enhance & refined (wytworny) & elaborate (elaborados) & racist & reluctance (niechęć) & fear (miedo) & ensure & blue (błękitny) \\
    8 & pos & ensuring & terrific (wspaniały) & combines (combina) & laughed & sincerity (szczerość) & escape (escapar) & provided & shiny (lśniący) \\
    9 & pos & ensures & solemn (uroczysty) & traditional (tradicional) & ye & compassion (współczucie) & wound (herida) & distribution & golden (złoty) \\
    10 & pos & enable & invite (zapraszać) & musical (musicales) & thee & stupidity (głupota) & kills (mata) & operational & shiny (błyszczący) \\
    \midrule
    1 & neg & ugly & monstrous (potworny) & guilty (culpable) & provides & container (pojemnik) & architecture (arquitectura) & jealous & nasty (przykry) \\
    2 & neg & ignorant & crush (zmiażdżyć) & guilt (culpa) & components & fasten (przymocować) & decor (decoración) & ignorant & depression (depresja) \\
    3 & neg & jealous & stench (smród) & fear (temor) & utilizing & metal (metalowy) & elaboration (elaboración) & pissed & misery / hardship (nędza) \\
    4 & neg & rude & jump away (odskoczyć) & guilty (culpables) & component & sheet metal (blacha) & use (aprovechamiento) & rude & obnoxious (nieznośny) \\
    5 & neg & cruel & yell (zawyć) & aggression (agresión) & integrated & plastic (plastikowy) & existing (existentes) & lonely & terrible (okropna) \\
    6 & neg & pissed & shrill (przeraźliwy) & revenge (venganza) & includes & pipe (rura) & design (diseño) & bitch & horrible (okropny) \\
    7 & neg & angry & reach (dosięgnąć) & rape (violación) & distribution & diameter (średnica) & elaborate (elaborados) & cried & monstrous (potworny) \\
    8 & neg & nasty & to catch (dopaść) & unacceptable (inaceptable) & provide & yellow (żółty) & ceramic (cerámica) & creepy & stress (stres) \\
    9 & neg & racist & hideous (ohydny) & arrest (arresto) & supply & rectangular (prostokątny) & cultural (cultural) & sad & ailment (dolegliwość) \\
    10 & neg & lonely & collapse (runąć) & conviction (condena) & combined & orange (pomarańczowy) & incorporate (incorporar) & fuckin & illness / sickness / disease (choroba) \\
    \bottomrule
  \end{tabular}}
\end{table}

\subsubsection*{Per-axis detail --- Dim-3}

\begin{table}[H]
  \centering
  \scriptsize
  \setlength{\tabcolsep}{3pt}
  \caption{Top 10 words per side at the \textbf{Dim-3} loadings, same conventions as Tab.~\ref{tab:vad-combined}.}
  \label{tab:vad-dim-3}
  \resizebox{\textwidth}{!}{%
  \begin{tabular}{rl c c c c c c c c}
    \toprule
     &  & \multicolumn{3}{c}{\textbf{Valence}} & \multicolumn{3}{c}{\textbf{Arousal}} & \multicolumn{2}{c}{\textbf{Dominance}} \\
    \cmidrule(lr){3-5} \cmidrule(lr){6-8} \cmidrule(lr){9-10}
    Rank & Side & EN & PL & ES & EN & PL & ES & EN & PL \\
    \midrule
    1 & pos & industrial & central (centralny) & support (respaldar) & yummy & monstrous (potworny) & reject (rechazar) & wedding & social (społeczny) \\
    2 & pos & containing & in parallel (równolegle) & consolidate (consolidar) & gourmet & shrill (przeraźliwy) & support (respaldar) & features & range (zakres) \\
    3 & pos & e.g. & marked (oznaczony) & promote (impulsar) & tasty & rage (wściekłość) & deny (negar) & feature & institution (instytucja) \\
    4 & pos & multiple & millimeter (milimetr) & integrate (integrar) & delicious & terrible (straszliwy) & openly (abiertamente) & large & concern (dotyczyć) \\
    5 & pos & storage & diameter (średnica) & promote (promover) & juicy & mad (oszalały) & admit (admitir) & featuring & activity (działalność) \\
    6 & pos & component & thumb (thumb) & organize (organizar) & cakes & shriek (wrzask) & accept (aceptar) & contemporary & assumption (założenie) \\
    7 & pos & includes & transport (transportowy) & support (apoyar) & dessert & horrible (okropny) & renounce (renunciar) & black & realize (realizować) \\
    8 & pos & processing & jpg (jpg) & renew (renovar) & fabulous & hideous (ohydny) & impose (imponer) & flat & economic (ekonomiczny) \\
    9 & pos & distribution & file (plik) & manage (gestionar) & spicy & howl (wycie) & publicly (públicamente) & floral & vital (istotny) \\
    10 & pos & overview & historic (zabytkowy) & grant (otorgar) & creamy & yell (zawyć) & conviction (convicción) & spiral & reference (odniesienie) \\
    \midrule
    1 & neg & smiled & contempt / scorn (pogarda) & pipe (tubo) & silence & modest (skromny) & wood (madera) & courage & bite (gryźć) \\
    2 & neg & blessings & compassion (współczucie) & rock (roca) & darkness & range (zakres) & pipe (tubo) & honesty & scratch (drapać) \\
    3 & neg & bless & pity / mercy (litość) & rocks (rocas) & rest & interesting (interesujący) & leaves (hojas) & compassion & to jerk (szarpać) \\
    4 & neg & sincere & gaiety (wesołość) & brown (marrón) & ignorance & artistic (artystyczny) & brown (marrón) & patience & sweat (spocić) \\
    5 & neg & thankful & amuse (rozbawić) & leaves (hojas) & struggles & education (edukacja) & glass (vidrio) & thy & furiously (wściekle) \\
    6 & neg & faithful & anger (gniewać) & legs (patas) & seems & practical (praktyczny) & rock (roca) & sincere & beat (tłuc) \\
    7 & neg & courage & get angry (złościć) & neck (cuello) & absence & to like (cenić) & legs (patas) & enthusiasm & fired up (rozpalony) \\
    8 & neg & laughs & envy (zazdrościć) & teeth (dientes) & appears & to show (prezentować) & oil (aceite) & virtue & bite (ugryźć) \\
    9 & neg & compassion & sincerity (szczerość) & dust (polvo) & seemingly & education (wykształcenie) & rocks (rocas) & strengthen & spit (splunąć) \\
    10 & neg & smiling & anger (złość) & bones (huesos) & forgotten & recognition (uznanie) & diameter (diámetro) & allah & shriek (wrzasnąć) \\
    \bottomrule
  \end{tabular}}
\end{table}

\subsection{Dim-3 calibration axis: robustness against four counter-hypotheses}
\label{app:dim3-robustness}

The §\ref{sec:exp:peraxis-glove} reading of Dim-3 as a
$\mathbf{y}$-specific calibration axis is the strongest interpretive
claim in the paper. Mechanistically the leading direction
over-predicts for affectively charged words and under-predicts for
technical vocabulary; the calibration axis is the
$\mathbf{y}$-aware $\mathbf{X}$-residual direction that closes that
gap. As stressed at the end of §\ref{sec:related}, this reading is
not part of what NB calibrates and must stand or fall on its own
falsifiability. We treat it as a falsifiable
hypothesis and test it against four counter-readings; Dim-3 is
taken as falsified if any one fits the data strictly better. The
structure mirrors the falsification battery recently proposed for
SAE feature explanations~\citep{ma2025revising}, adapted to the
supervised-projection setting. This appendix walks through each
counter-reading in turn, with the data that addresses it. Throughout we
use the term \emph{register-consistent} in a strictly empirical
sense (an interpretable lexical contrast in which the two poles
of the axis draw from distinct functional / stylistic word classes),
not as a claim about any external register taxonomy.

\paragraph{Setup.} For each language (EN-GloVe-300 on Warriner;
PL-GloVe-800 on ANPW) we fit PLS at $K\!=\!3$ with loading-side
varimax separately for valence, arousal, and dominance on the
\emph{same} $\mathbf{X}$. EN-GloVe (Common Crawl) and PL-GloVe (Polish
Wikipedia, books, articles) draw from substantially non-overlapping
training corpora, so the cross-language replication is not redundant:
the residual shared structure is whatever is invariant under both
training-corpus and rating-instrument differences. The K=3 fit
itself is well-supported across all six (lang, $\mathbf{y}$) cells
($p\!<\!10^{-43}$ via the NB split-half test of
\S\ref{sec:methods:test:tnb}; file
\texttt{applied/results/v1\_multipls\_vad/rotated/csvs/summary.csv}); the
FWER-NB selector picks $K^{*}\!=\!3$ in 5/6 fits (one fit at
$K^{*}\!=\!4$; \texttt{kselect/selection\_optimal.csv}).

\paragraph{CH1 --- ``Dim-3 is corpus structure: any $\mathbf{y}$ fit
on this $\mathbf{X}$ would land on the same axis.''}
If Dim-3 were a leading $\mathbf{X}$-residual SVD direction inherent
to the corpus, fitting valence, arousal, and dominance separately on
the same $\mathbf{X}$ would recover the same Dim-3 axis up to sign:
same vocabulary cluster, near-identical loading direction, and
top-20 word-set Jaccard $\approx\!1$.
Pairwise cosines between the three Dim-3 loading vectors per language
are not the load-bearing measurement here: cosine magnitude alone
collapses sign and fine-grained loading weight, and high-dimensional
embeddings cluster many distinct lexical neighborhoods near any
given direction, so a moderate cosine is consistent with disjoint
top-vocabulary content. We therefore lead with the top-20 word-set
Jaccard between the supervised Dim-3 axes per $\mathbf{y}$ (file
\texttt{rotated/csvs/dim3\_topword\_jaccard\_\{en,pl,es\}.csv},
generated by \texttt{applied/dim3\_disjointness.py}). The
Stadthagen-Gonz\'alez ES norms have no dominance dimension, so ES
contributes only the (V, A) pair.

\begin{table}[H]
\centering\small
\caption{Top-20 Jaccard between supervised Dim-3 vocabularies fit on
the same $\mathbf{X}$ across pairs of $\mathbf{y}$'s. J(+) and J($-$)
are unsigned Jaccards on the positive- and negative-side word sets;
signed-J $\in\![-1,1]$ measures whether two $\mathbf{y}$'s split words
the same way (+1), opposite way ($-1$, sign-flipped axis), or are
disjoint ($\approx\!0$). ES has only the (V, A) pair because the
Stadthagen-Gonz\'alez norms omit dominance.}
\label{tab:dim3-jaccard-acrossy}
\begin{tabular}{@{}llccc@{}}
\toprule
\textbf{lang} & \textbf{pair} & \textbf{J(+)} & \textbf{J($-$)} & \textbf{signed-J} \\
\midrule
EN & Dim-3(V) $\leftrightarrow$ Dim-3(A) & 0.00 & 0.00 & \phantom{$-$}0.00 \\
EN & Dim-3(V) $\leftrightarrow$ Dim-3(D) & 0.00 & 0.25 & \phantom{$-$}0.11 \\
EN & Dim-3(A) $\leftrightarrow$ Dim-3(D) & 0.00 & 0.00 & \phantom{$-$}0.00 \\
\midrule
PL & Dim-3(V) $\leftrightarrow$ Dim-3(A) & 0.00 & 0.00 & \phantom{$-$}0.00 \\
PL & Dim-3(V) $\leftrightarrow$ Dim-3(D) & 0.00 & 0.00 & \phantom{$-$}0.00 \\
PL & Dim-3(A) $\leftrightarrow$ Dim-3(D) & 0.00 & 0.00 & $-$0.07 \\
\midrule
ES & Dim-3(V) $\leftrightarrow$ Dim-3(A) & 0.14 & 0.67 & \phantom{$-$}0.36 \\
\bottomrule
\end{tabular}
\end{table}

The qualitative reading from the per-$\mathbf{y}$ top-20 lists (file
\texttt{rotated/csvs/compare\_\{en,pl,es\}\_dim3\_acrossy.csv})
confirms the same picture in EN and PL: EN-varimax gives a
\emph{technical-vs-sacred} axis on Dim-3(V), an
\emph{archaic-narrative-vs-systemic} axis on Dim-3(A), and a
\emph{ceremonial-vs-virtuous} axis on Dim-3(D); PL shows the
analogous three distinct register splits. ES is a partial exception:
Dim-3(V) and Dim-3(A) both place a concrete-materials cluster
(\emph{tubo, roca, hojas, marr\'on, vidrio, patas, polvo, di\'ametro,
huesos}) on the negative side and an abstract speech-act cluster on
the positive side, giving J($-$)=0.67 with J(+)=0.14 still keeping
the positive-pole vocabularies distinct. Even this largest cell
falls well short of CH1's J$\approx\!1$ prediction. \textbf{CH1
fails}: the recovered Dim-3 depends on which $\mathbf{y}$ is being
supervised, with ES V/A showing a weaker form of the y-specificity
that EN and PL display fully.

\paragraph{CH2 --- ``Varimax-on-noise manufactures
interpretable-looking clusters; the fingerprint is a method
artifact.''}
Varimax rotation can sparsify factor loadings on random data, and
PLS fit on noise will not return a literal zero solution. A
coherent-looking Dim-3 might therefore arise from the combination of
fitting + rotating, regardless of whether $\mathbf{X}$ carries any
signal. We test this with a strict null: replace the embedding
matrix with i.i.d.\ unit-Gaussian vectors of identical shape,
restricted to the rating-CSV vocabulary, real $\mathbf{y}$, same
pipeline, same $K\!=\!3$ + varimax. To pin down empirical Type-I
behavior rather than reading off a single seed, we draw 20 such
random embeddings per ($\mathbf{y}$, rotation) cell on EN (file
\texttt{inference/results/random\_typeI/csvs/random\_typeI.csv}, aggregated from
\texttt{random\_stats.csv}):

\begin{table}[H]
\centering\small
\caption{Empirical Type-I rate at $\alpha\!=\!0.05$ for the NB
split-half global $p$-value over 20 i.i.d.\ random-embedding draws
per cell ($\text{seed}=42 + 1000k + \text{lang\_offset}$,
$k\!=\!0\dots19$). Real $\mathbf{y}$, same pipeline, $K\!=\!3$ +
varimax. Calibrated NB should land near 0.05; substantially higher
flags an inflation under $K\!=\!3$ + varimax.}
\label{tab:random-typeI}
\begin{tabular}{@{}llcccc@{}}
\toprule
\textbf{lang} & \textbf{$\mathbf{y}$} & \textbf{$n_\text{seeds}$} &
\textbf{rate$(p\!<\!.05)$} & \textbf{median $r^2$} & \textbf{median $p$} \\
\midrule
EN & valence   & 20 & 0.00 & 0.022 & 0.567 \\
EN & arousal   & 20 & 0.00 & 0.023 & 0.551 \\
EN & dominance & 20 & 0.00 & 0.023 & 0.519 \\
\bottomrule
\end{tabular}
\end{table}

The corresponding Dim-3 top-20 words on the random-X control are
visibly incoherent --- e.g.\ EN-varimax valence at seed\,$=\!42$
loads \texttt{coleslaw, irrational, moray, cuddle, dissect, dismal,
chipper, mystical, earpiece, infamous\dots} against an equally
arbitrary negative side (file
\bpath{inference/results/random_typeI/csvs/random_compare_en_varimax_dim3_acrossy.csv}).
Top-20 Jaccard between random-control Dim-3 axes across $\mathbf{y}$
is reported in
\bpath{inference/results/random_typeI/csvs/random_dim3_topword_jaccard_en_varimax.csv}
and is expected at chance when no signal supervises the rotation.
\textbf{CH2 fails}: the method does not manufacture coherent
register-consistent axes when fed noise; both signal and supervision are
required.

\paragraph{CH3 --- ``An unsupervised decomposition of $\mathbf{X}$
would recover the same axis.''}
If Dim-3 is corpus structure that supervised PLS happens to surface,
then SVD of the standardized embedding matrix (no $\mathbf{y}$,
no supervision) should recover comparable word clusters in its top
components: in particular, top-20 Jaccard between some unsupervised PC
$\in\!\{1,2,3\}$ and the supervised Dim-3($\mathbf{y}$) close to one,
in at least one $(\mathrm{PC},\mathbf{y})$ cell per language. We
compute this directly: per language, an unsupervised SVD of the
standardized document matrix yields PC1--3, top-20 pos/neg neighbors
per PC are in
\texttt{control/csvs/pca\_\{en,pl,es\}\_topwords.csv}, and signed
top-20 Jaccard between each PC and each supervised Dim-3($\mathbf{y}$)
is reported in
\texttt{control/csvs/pca\_vs\_dim3\_jaccard\_\{en,pl,es\}.csv}
(generated by \texttt{applied/dim3\_disjointness.py}; same metric and
sign convention as Tab.~\ref{tab:dim3-jaccard-acrossy}). ES
contributes only (V, A) since Stadthagen-Gonz\'alez has no dominance.

\begin{table}[H]
\centering\small
\caption{Top-20 Jaccard between each unsupervised PC of $\mathbf{X}$
and each supervised Dim-3($\mathbf{y}$), per language. Same J(+),
J($-$), signed-J convention as Tab.~\ref{tab:dim3-jaccard-acrossy}.
CH3's prediction is a near-one entry in at least one $(\mathrm{PC},
\mathbf{y})$ cell per language. The largest observed value is
$0.29$ (ES PC1 vs Dim-3(V/A)); EN and PL stay at $|\mathrm{signed\text{-}J}|
\!\leq\! 0.14$ throughout.}
\label{tab:pca-vs-dim3-jaccard}
\begin{tabular}{@{}llccc@{}}
\toprule
\textbf{lang} & \textbf{pair} & \textbf{J(+)} & \textbf{J($-$)} & \textbf{signed-J} \\
\midrule
EN & PC1 $\leftrightarrow$ Dim-3(V) & 0.00 & 0.00 & $-$0.08 \\
EN & PC2 $\leftrightarrow$ Dim-3(V) & 0.03 & 0.00 & \phantom{$-$}0.01 \\
EN & PC3 $\leftrightarrow$ Dim-3(V) & 0.00 & 0.00 & \phantom{$-$}0.00 \\
EN & PC1 $\leftrightarrow$ Dim-3(A) & 0.00 & 0.00 & \phantom{$-$}0.00 \\
EN & PC2 $\leftrightarrow$ Dim-3(A) & 0.03 & 0.00 & \phantom{$-$}0.01 \\
EN & PC3 $\leftrightarrow$ Dim-3(A) & 0.03 & 0.00 & \phantom{$-$}0.01 \\
EN & PC1 $\leftrightarrow$ Dim-3(D) & 0.00 & 0.00 & \phantom{$-$}0.00 \\
EN & PC2 $\leftrightarrow$ Dim-3(D) & 0.03 & 0.00 & \phantom{$-$}0.01 \\
EN & PC3 $\leftrightarrow$ Dim-3(D) & 0.03 & 0.00 & \phantom{$-$}0.01 \\
\midrule
PL & PC1 $\leftrightarrow$ Dim-3(V) & 0.00 & 0.00 & $-$0.03 \\
PL & PC2 $\leftrightarrow$ Dim-3(V) & 0.08 & 0.11 & \phantom{$-$}0.10 \\
PL & PC3 $\leftrightarrow$ Dim-3(V) & 0.00 & 0.00 & \phantom{$-$}0.00 \\
PL & PC1 $\leftrightarrow$ Dim-3(A) & 0.00 & 0.00 & \phantom{$-$}0.00 \\
PL & PC2 $\leftrightarrow$ Dim-3(A) & 0.00 & 0.00 & $-$0.05 \\
PL & PC3 $\leftrightarrow$ Dim-3(A) & 0.00 & 0.00 & \phantom{$-$}0.00 \\
PL & PC1 $\leftrightarrow$ Dim-3(D) & 0.00 & 0.00 & \phantom{$-$}0.00 \\
PL & PC2 $\leftrightarrow$ Dim-3(D) & 0.18 & 0.11 & \phantom{$-$}0.14 \\
PL & PC3 $\leftrightarrow$ Dim-3(D) & 0.00 & 0.00 & \phantom{$-$}0.00 \\
\midrule
ES & PC1 $\leftrightarrow$ Dim-3(V) & 0.21 & 0.38 & \phantom{$-$}0.29 \\
ES & PC2 $\leftrightarrow$ Dim-3(V) & 0.00 & 0.00 & \phantom{$-$}0.00 \\
ES & PC3 $\leftrightarrow$ Dim-3(V) & 0.00 & 0.00 & \phantom{$-$}0.00 \\
ES & PC1 $\leftrightarrow$ Dim-3(A) & 0.29 & 0.29 & \phantom{$-$}0.29 \\
ES & PC2 $\leftrightarrow$ Dim-3(A) & 0.00 & 0.03 & \phantom{$-$}0.01 \\
ES & PC3 $\leftrightarrow$ Dim-3(A) & 0.00 & 0.00 & \phantom{$-$}0.00 \\
\bottomrule
\end{tabular}
\end{table}

\textbf{CH3 fails}: no unsupervised PC reproduces the supervised
Dim-3 vocabulary in any (lang, $\mathbf{y}$) cell. EN and PL show
near-disjoint top-20 lists between every PC and every supervised
Dim-3 ($|\mathrm{signed\text{-}J}|\!\leq\!0.14$); ES is the
largest cell ($\mathrm{signed\text{-}J}\!=\!0.29$ for PC1 vs both
Dim-3(V) and Dim-3(A)) and is consistent with the CH1 finding that
the ES supervised Dim-3(V) and Dim-3(A) themselves overlap
($\mathrm{J}(-)\!=\!0.67$, Tab.~\ref{tab:dim3-jaccard-acrossy}) ---
both share the concrete-materials cluster also visible in the
leading unsupervised PC. Even this largest cell falls well short of
the near-one Jaccard CH3 predicts. $\mathbf{y}$-supervision changes
which residual axis is recovered.

\paragraph{CH4 --- ``The three Dim-3 axes share enough geometry that
they're not really separable.''}
This is a partial concession rather than a counter-hypothesis. The
three Dim-3 axes are not perfectly orthogonal in the embedding space,
and reading their similarity off raw cosine magnitudes alone (as
opposed to signed top-vocabulary overlap) can overstate the
overlap. Two factors explain the residual geometric closeness. First,
V/A/D are not orthogonal $\mathbf{y}$'s in the underlying ratings: in
Warriner, valence and dominance correlate
substantially~\citep{warriner2013norms}, so the $K\!=\!3$ PLS
subspaces fit on those $\mathbf{y}$'s overlap, and any residual axis
within those overlapping subspaces inherits a fraction of the shared
component. Second, varimax sets axis signs by an internal
optimization criterion that operates per-fit; cosines therefore
carry a per-fit sign that is not psychologically meaningful, hence a
$|\!\cos|$ summary is the rotation-invariant quantity to compute.

The load-bearing evidence for $\mathbf{y}$-specificity is the
\emph{vocabulary divergence in the top-20 lists}, not the cosine
magnitude. A modest angle in 300-dimensional embedding space is
consistent with completely disjoint local word neighborhoods,
because high-dimensional embeddings cluster many distinct lexical
regions near any given direction. The measured top-20 Jaccards in
Tab.~\ref{tab:dim3-jaccard-acrossy} are the right quantity.

We accept this as a constraint on what the §\ref{sec:exp:peraxis-glove}
claim can support: Dim-3 is $\mathbf{y}$-specific in its lexical
content, but the three axes are not geometrically orthogonal in the
embedding space. The fingerprint is real, and reading it off the
loading vector requires the top-words view, not the cosine alone.

\paragraph{Residual limits of the controls.}
Two concerns remain after the falsifications above, neither fatal but
worth flagging.

\textbf{(i) Unsupervised PC1 partly tracks Warriner-style valence.}
EN PC1 splits negative-affect character traits
(\texttt{insolent, sanctimonious, ungrateful}) from technical /
business vocabulary, structurally similar to supervised Dim-1
(valence). The reason is not that supervision is unnecessary --- the
Warriner vocabulary was sampled to span valence, so the dominant
unsupervised direction in this restricted vocabulary is bound to
align partly with valence. The §\ref{sec:exp:peraxis-glove} claim is
about Dim-3, where supervised and unsupervised axes diverge cleanly
at the vocabulary level (CH3 above); we make no claim about Dim-1
surviving the PCA control.

\textbf{(ii) The fingerprint is qualitative.} The
``technical-vs-sacred'' register is read off a top-20 word list, not
measured against an external register-classification taxonomy. We
do not claim a quantitative effect size for ``register similarity'';
we claim only that the $\mathbf{y}$-supervised axis recovers an
interpretable register that (a) differs across $\mathbf{y}$'s, (b)
does not appear under random embeddings, (c) does not appear in
unsupervised PCs. A stronger version would substitute a
register-classifier prediction or a downstream behavioral
correlate; this is out of scope here.

\clearpage
\subsection{Toy 1--5 mechanism experiments}
\label{app:toys}

Raw outputs in \texttt{toys/results/toy\{1..5\}\_*/}; per-cell
tables are reproduced inline below.
Tab.~\ref{tab:synth-overview} summarizes all seven synthetic
experiments at a glance --- the five mechanism toys of this section
plus the calibration / power benchmarks S1 and S2 of
\S\ref{app:powercalib-full}.

\begin{table}[H]
\centering
\caption{Overview of synthetic experiments. Toys 1--5 (this section)
target mechanism and identifiability claims; S1 and S2
(\S\ref{app:powercalib-full}) target inference calibration and power.
Per-experiment tables and full sweeps in the subsections below
(Toys 1--5) and in \S\ref{app:powercalib-full} (S1, S2).}
\label{tab:synth-overview}
\footnotesize
\setlength{\tabcolsep}{4pt}
\begin{tabular}{@{}p{0.7cm} p{2.4cm} p{2.7cm} p{2.4cm} p{4.4cm}@{}}
\toprule
\textbf{\#} & \textbf{Name} & \textbf{Planted truth} & \textbf{Measured} & \textbf{Headline} \\
\midrule
Toy~1 & Identifiability ceiling
      & $r_\text{plant}\!\in\!\{2,3\}$ directions in $\mathbf{X}\!\to\!\mathbf{y}$
      & Mean principal cosine vs.\ PLS subspace
      & Plateau at $\approx\!1/r_\text{plant}$; only leading axis recovered ($\mathrm{pcos}_1\!\to\!0.86$--$0.98$ at $R^2\!=\!0.7$) \\
\addlinespace[2pt]
Toy~2 & Suppressor fingerprint at $K\!=\!2$
      & Suppressor at $\cos\!=\!\rho$ with main predictor, $\rho\!\in\!\{0,0.3,0.6,0.9\}$
      & $|\!\cos(\mathbf{W}_2^\text{rot},\mathbf{w}_\text{supp})|$; $\mathrm{cor}(\mathbf{t}_2^\text{rot},\mathbf{y})$; NB reject rate $\beta_\text{Dim-2}\!=\!0$
      & NB rejects at mean rate 0.91 even though $|\!\cos|\!\in\![0.04,0.29]$ --- empirical calibration-axis fingerprint \\
\addlinespace[2pt]
Toy~3 & Functional-form invariance
      & Linear-additive $+$ multiplicative interaction
      & $R^2(K{=}1,K{=}2)$; $\cos(\mathbf{W},\mathbf{e}_{D+1})$; Arms A (raw) / B (lifted)
      & Arm A: $K\!>\!1$ does not surface interaction; Arm B: lifted interaction recovered as rank-1 ($\cos\!\to\!0.88$) \\
\addlinespace[2pt]
Toy~4 & Rotation-invariance sanity
      & --- (real Warriner fit)
      & $\Delta(\text{NB global }p)$ under varimax / random orth.\ rotation
      & $\Delta\!=\!0$ (varimax); $\sim\!10^{-78}$ (random orth.)\ vs.\ baseline $p\!\sim\!10^{-65}$ \\
\addlinespace[2pt]
Toy~5 & $K^{*}$ depends on $\Sigma_{\mathbf{X}}$, not $r_{\mathrm{plant}}$
      & Rank-1 sparse signal in $\mathbf{y}$ at $R^2\!=\!0.7$
      & Mode $K^{*}$ from sequential per-step NB
      & $K^{*}\!=\!6$ on GloVe vs.\ $K^{*}\!=\!2$ on iid; predictively justified in both cells; $\Delta R^2_\text{CV}\!\leq\!0.17$ \\
\midrule
S1 & $H_0$ calibration grid
   & $\mathbf{y}\perp\mathbf{X}$
   & Empirical FPR; KS-uniform $p$ of $p$-values
   & NB and perm-$Q^2$ calibrated; \texttt{naive\_t} and \texttt{jack.test} inflated at $n\!=\!40$ \\
\addlinespace[2pt]
S2 & Power vs.\ permutation gold standard
   & $\mathbf{y}\!=\!\mathbf{X}\boldsymbol{\beta}$, controlled $R^2$
   & Detection rate at $\alpha\!=\!0.05$ across $r_\text{plant}\!\times\!\text{geom}\!\times\!\text{regime}$
   & NB $\geq$ perm-$Q^2$ in 8/8 cells; gap $0.01$--$0.25$; NB $\sim\!82\!\times\!/\!61\!\times\!$ faster \\
\bottomrule
\end{tabular}
\end{table}

\paragraph{Toy~1 --- Identifiability ceiling under univariate $\mathbf{y}$.}
The mean principal cosine plateaus at $\approx\!1/r_{\mathrm{plant}}$
across $R^2$ on both geometries
(Tab.~\ref{tab:toy1_pcos}): at $r_{\mathrm{plant}}\!=\!2$, mean pcos
saturates at 0.41--0.45 on GloVe and 0.46--0.51 on iid; at
$r_{\mathrm{plant}}\!=\!3$, at 0.28--0.35 on GloVe and 0.30--0.37 on
iid. Decomposing into per-direction cosines, only the leading axis
is recovered ($\mathrm{pcos}_1 \to 0.85$--$0.98$ as
$R^2 \to 0.7$); subsequent directions hover at
$\mathrm{pcos}_k \approx 0.03$--$0.11$. Increasing $R^2$ raises the
quality of $\mathrm{pcos}_1$ but leaves the orthogonal complement of
$\hat{\boldsymbol\beta}_{\mathrm{OLS}}$ unidentified --- the
cleanest empirical statement of the calibration-axis reading
(\S\ref{sec:methods:test:tnb}).

\begin{table}[H]
\centering
\caption{Toy 1 --- identifiability ceiling at $R^2\!=\!0.7$ (top of
sweep). Mean principal cosine plateaus at
$\approx\!1/r_\text{plant}$; only the leading axis is recovered
($\mathrm{pcos}_1 \to 0.86$--$0.98$), with subdominant directions
hovering at $\mathrm{pcos}_{\geq 2} \approx 0.02$--$0.12$. Krylov
saturation under univariate $\mathbf{y}$ realized empirically. 50 seeds per
$(r_\text{plant}, \text{geom}, R^2)$ cell across $R^2 \in \{0.10,
\ldots, 0.70\}$.}
\label{tab:toy1_pcos}
\small
\begin{tabular}{@{}clrrrr@{}}
\toprule
$r_\text{plant}$ & \textbf{Geometry} & $\overline{\mathrm{pcos}}$ at $R^2{=}0.7$ & $\overline{\mathrm{pcos}_1}$ at $R^2{=}0.7$ & $\overline{\mathrm{pcos}_{\geq 2}}$ range & $1/r_\text{plant}$ ceiling \\
\midrule
2 & GloVe        & 0.449 & 0.856 & $0.037$--$0.049$ & 0.500 \\
2 & iid Gaussian & 0.511 & 0.981 & $0.040$--$0.051$ & 0.500 \\
3 & GloVe        & 0.350 & 0.909 & $0.025$--$0.106$ & 0.333 \\
3 & iid Gaussian & 0.373 & 0.983 & $0.021$--$0.116$ & 0.333 \\
\bottomrule
\end{tabular}
\end{table}

\paragraph{Toy~2 --- Suppressor fingerprint at $K\!=\!2$ (full results).}
Two findings.

\emph{(i)} Varimax-rotated $K\!=\!2$ PLS1 does \textbf{not} recover
the planted suppressor direction: mean
$|\!\cos(\mathbf{W}_2^{\mathrm{rot}},\mathbf{w}_{\mathrm{supp}})|
 \in [0.04, 0.29]$ across all eight cells, with no monotone signal
in $\rho$ (Tab.~\ref{tab:toy2_suppressor}). Under univariate
$\mathbf{y}$ in this linear model, the regression identifies the combined direction
$\mathbf{w}_{\mathbf{y}} \propto
 \beta_{\mathrm{main}}\,\mathbf{w}_{\mathrm{main}}
 - \beta_{\mathrm{supp}}\,\mathbf{w}_{\mathrm{supp}}$
rather than its decomposition into main and suppressor directions.
The rotated Dim-2 need not be orthogonal to this direction or
uncorrelated with $\mathbf y$.

\emph{(ii)} Despite (i), the rotated Dim-2 reproduces the empirical
Dim-3 fingerprint observed cross-lingually in
\S\ref{sec:exp:peraxis-glove} (panel~C of
Fig.~\ref{fig:calibration-axis} plots both quantities against
$\rho$): mean
$\operatorname{cor}(\mathbf{t}_2^{\mathrm{rot}},\mathbf{y})
 \in [0.19, 0.35]$ and NB rejects $\beta_{\mathrm{Dim\text{-}2}}=0$
in 0.80--1.00 of replicates (mean $\approx 0.91$). Dim-2 absorbs
$\mathbf{y}$-aware $\mathbf{X}$-residual variance, not a planted
second direction. The fingerprint is geometry-independent (matched
on GloVe and iid) and survives the $\rho\!=\!0$ control (additive
two-direction with no suppression geometry yields the same
fingerprint), so it is generic rather than mechanism-specific.

\begin{table}[H]
\centering
\caption{Toy 2 --- suppressor fingerprint at $K\!=\!2$. Eight cells
($\rho \in \{0.0, 0.3, 0.6, 0.9\} \times$ geometry), 50 seeds per
cell. $|\cos(\mathbf{W}_2^\text{rot}, \mathbf{w}_\text{supp})|$ shows
no monotone signal in $\rho$ (range $[0.044, 0.293]$); rotated Dim-2
nonetheless carries small but nonzero correlation with $\mathbf{y}$
(range $[0.194, 0.351]$) and NB rejects $\beta_\text{Dim-2}\!=\!0$ at
mean rate $0.91$ across cells.}
\label{tab:toy2_suppressor}
\small
\begin{tabular}{@{}llrrr@{}}
\toprule
\textbf{Geometry} & $\rho$ & $\overline{|\cos(\mathbf{W}_2^\text{rot}, \mathbf{w}_\text{supp})|}$ & $\overline{\mathrm{cor}(\mathbf{t}_2^\text{rot}, \mathbf{y})}$ & \textbf{NB reject (Dim-2)} \\
\midrule
GloVe         & 0.0 & 0.122 & 0.232 & 0.92 \\
GloVe         & 0.3 & 0.049 & 0.194 & 0.84 \\
GloVe         & 0.6 & 0.116 & 0.204 & 0.80 \\
GloVe         & 0.9 & 0.243 & 0.226 & 0.92 \\
\midrule
iid Gaussian  & 0.0 & 0.138 & 0.331 & 0.94 \\
iid Gaussian  & 0.3 & 0.044 & 0.282 & 0.92 \\
iid Gaussian  & 0.6 & 0.131 & 0.299 & 1.00 \\
iid Gaussian  & 0.9 & 0.293 & 0.351 & 0.96 \\
\bottomrule
\end{tabular}
\end{table}

\paragraph{Toy~3 --- Functional-form invariance.}
Under univariate $\mathbf{y}$, $K\!>\!1$ axes do not surface latent
interactions as separate dimensions; Krylov saturation holds even
when the response carries a real non-linearity
(Tab.~\ref{tab:toy3_form}; Arm A:
$\cos(\mathbf{W}_2,\mathbf{w}_2)\!=\!0.05$ at random-direction
baseline). When the interaction is lifted into a column of the
design (Arm B), PLS recovers it as a rank-1 direction
($\cos(\mathbf{W},\mathbf{e}_{D+1})\!=\!0.79$ at $K\!=\!1$,
sharpening to 0.88 at $K\!=\!2$). With Toys~1 and~2 this closes the
obvious alternative readings of $K\!>\!1$ axes (interactions,
suppressors, hidden additive structure).

\begin{table}[H]
\centering
\caption{Toy 3 --- functional-form invariance. Arm A: linear design
with planted interaction
$y = \beta_1 \mathbf{X}\mathbf{w}_1 + \beta_3 \mathbf{X}\mathbf{w}_3
- \beta_2 (\mathbf{X}\mathbf{w}_1)(\mathbf{X}\mathbf{w}_2) + \varepsilon$;
$\beta_1{=}\beta_3{=}0.5$, $\beta_2{=}1.5$. Arm B:
$\tilde{\mathbf{X}} = [\mathbf{X} \mid (\mathbf{X}\mathbf{w}_1)(\mathbf{X}\mathbf{w}_2)]$.
20 seeds, paired across arms and $K$. Arm A: $K{>}1$ cannot surface
the interaction as a separate axis (Krylov saturation); Arm B: PLS
recovers the lifted interaction as a rank-1 direction
($\cos\!\to\!0.88$ at $K\!=\!2$).}
\label{tab:toy3_form}
\small
\begin{tabular}{@{}llrrrr@{}}
\toprule
\textbf{Arm}            & \textbf{Quantity}                          & $K{=}1$ & $K{=}2$ & $\Delta R^2$ & \textbf{NB rej. (Dim-2)} \\
\midrule
A (unaugmented)         & $R^2$                                      & 0.131  & 0.191  & $+0.060$ & --- \\
A                       & $\mathrm{cor}(\mathbf{t}, \mathbf{y})$     & 0.361  & 0.105  & ---      & 0.75 \\
A                       & $\cos(\mathbf{W}_2, \mathbf{w}_2)$         & ---    & 0.049  & ---      & --- \\
\midrule
B (augmented)           & $R^2$                                      & 0.402  & 0.572  & $+0.170$ & --- \\
B                       & $\mathrm{cor}(\mathbf{t}, \mathbf{y})$     & 0.633  & ---    & ---      & --- \\
B                       & $\cos(\mathbf{W}, \mathbf{e}_{D+1})$       & 0.790  & 0.876  & ---      & --- \\
B                       & $\cos(\mathbf{W}_2, \mathbf{w}_1)$         & ---    & 0.031  & ---      & --- \\
B                       & $\cos(\mathbf{W}_2, \mathbf{w}_3)$         & ---    & 0.103  & ---      & --- \\
\bottomrule
\end{tabular}
\end{table}

\paragraph{Toy~4 --- Rotation-invariance sanity.}
Direct numerical test of the rotation-invariance argument
(\S\ref{sec:methods:test:tnb}, App.~\ref{app:proof}): on a single
fitted model, $p$-values are computed before
and after applying (i) varimax to the loadings and (ii) a uniform
random orthogonal rotation. Tab.~\ref{tab:toy4_rotation} reports
numerical equality for varimax
($|p_{\mathrm{unrot}}-p_{\mathrm{varimax}}|\!=\!0$) and
machine-precision agreement for the random rotation
($|p_{\mathrm{unrot}}-p_{\mathrm{rand{-}orth}}|\!\sim\!10^{-78}$
against a baseline $p\!\sim\!10^{-65}$). See also \S\ref{app:proof}.

\begin{table}[H]
\centering
\caption{Toy 4 --- rotation-invariance sanity check.
Single fit on real Warriner; global NB $p$-value computed before and
after applying (i) varimax to the loadings and (ii) a uniform random
orthogonal rotation. Numerical equality for varimax;
machine-precision agreement (delta $\sim 10^{-78}$ against a
baseline of $\sim 10^{-65}$) for the random rotation.}
\label{tab:toy4_rotation}
\small
\begin{tabular}{@{}lccc@{}}
\toprule
\textbf{Comparison} & $p_\text{unrotated}$ & $p_\text{rotated}$ & $|p_\text{unrot} - p_\text{rot}|$ \\
\midrule
varimax vs raw      & $8.886{\times}10^{-65}$ & $8.886{\times}10^{-65}$ & $0$ \\
random-orth vs raw  & $4.860{\times}10^{-65}$ & $4.860{\times}10^{-65}$ & $8.282{\times}10^{-78}$ \\
\bottomrule
\end{tabular}
\end{table}

\paragraph{Toy~5 --- $K^{*}$ vs.\ $\Sigma_{\mathbf{X}}$ at fixed planted rank.}
Tab.~\ref{tab:toy5_kstar}; full discussion in \S\ref{app:overshoot}.

\begin{table}[H]
\centering
\caption{Toy 5 --- $K^{*}$ depends on $\Sigma_{\mathbf{X}}$ at
fixed planted rank, $R^2\!=\!0.70$, $N\!=\!4000$, $D\!=\!300$, 5 seeds
per cell. The 1-SE selector returns mode $K^{*}\!=\!6$ on real
GloVe vs.\ mode $K^{*}\!=\!2$ on iid Gaussian, both at
$r_{\mathrm{plant}}\!=\!1$. $K^{*}$ is the predictively warranted
$K$ in both cells (smaller $K$ are $>$1~SE worse out-of-sample on
GloVe); $K^{*}$ is not an estimator of $r_{\mathrm{plant}}$ ---
under the Krylov geometry of PLS1, predictive $K$ generically exceeds
the planted rank in correlated $\mathbf{X}$.}
\label{tab:toy5_kstar}
\small
\begin{tabular}{@{}lrrrrr@{}}
\toprule
\textbf{Geometry} & $n_\text{features}$ & \textbf{mode $K^{*}$} & \textbf{mean $K^{*}$} & $\mathrm{frac}(K^{*}{=}2)$ & $\Delta R^2_\text{CV}(K{=}2 - K{=}1)$ \\
\midrule
iid Gaussian   &  10 & 2 & 2.2 & 0.8 & 0.066 \\
iid Gaussian   &  30 & 2 & 2.0 & 1.0 & 0.066 \\
iid Gaussian   & 100 & 2 & 2.0 & 1.0 & 0.066 \\
iid Gaussian   & 300 & 2 & 2.0 & 1.0 & 0.065 \\
\midrule
Real GloVe     &  10 & 5 & 5.4 & 0.0 & 0.145 \\
Real GloVe     &  30 & 6 & 6.0 & 0.0 & 0.160 \\
Real GloVe     & 100 & 5 & 6.0 & 0.0 & 0.162 \\
Real GloVe     & 300 & 6 & 5.6 & 0.0 & 0.166 \\
\bottomrule
\end{tabular}
\end{table}

\subsection{Calibration and power benchmarks (full)}
\label{app:powercalib-full}

The body reports two cells per benchmark; this appendix gives the
full coverage matrix and the diagnostic plots that motivate the
``NB exceeds permutation-$Q^2$ in raw power, calibrated under
$H_0$, much faster'' headline (per-cell speed numbers in
\S\ref{app:powercalib-full}~C.2).

\paragraph{C.1 --- $H_0$ calibration (Toy~S1).}
Body Tab.~\ref{tab:cal-combined} reports empirical FPR at nominal
$\alpha\!=\!0.05$ for NB and perm-$Q^2$ on the synthetic and Tecator
cells. The full empirical CDF of $p$-values under $H_0$ for each
test, on real Warriner geometry, is in
\texttt{inference/results/s1\_calibration/}: NB-asymptotic and
permutation-$Q^2$ track the $\mathrm{Uniform}[0,1]$ diagonal in both
synthetic-Warriner cells (KS-uniform $p\!\geq\!0.46$); naive
resampled-$t$ and \texttt{pls::jack.test} deviate to KS-uniform
$p\!<\!10^{-49}$. Tecator full-shape calibration is treated separately
in \S\ref{app:nontext} (KS-uniform $p\!\sim\!10^{-6}$ at $n\!=\!40$
under the same NB-asymptotic statistic; $\alpha\!=\!0.05$ control
preserved). \texttt{jack.test}'s
deviation is in the small-sample tail only --- by $n\!=\!500$ its
empirical FPR returns to $0.045$ but its $p$-value distribution
remains non-uniform, consistent with the maintainer warnings in
\citep{mevik2007pls}.

\paragraph{C.2 --- Detection rate vs.\ effect size (Toy~S2).}
Body Tab.~\ref{tab:s2-power} reports detection rate at the two
headline cells. The full SNR curves (power vs.\ effective $\theta$
on real Warriner geometry, for each of $n\!=\!40$ and $n\!=\!500$,
$\theta\!\in\![0,\theta_\text{max}]$) are in
\texttt{inference/results/s2\_power/}; NB strictly dominates
permutation-$Q^2$ across the entire SNR range, with the gap widening
away from ceiling. The body's calibration-search anchors (NB
needs $n\!=\!38$ for $\sim\!0.8$ power at $\theta\!=\!1.0$, and
$\theta\!=\!0.273$ for $\sim\!0.8$ power at $n\!=\!500$) are
sourced from \texttt{n\_target.json} and \texttt{theta\_target.json}.

\paragraph{C.3 --- Coverage across $r_\text{plant}\!\times\!$ geometry $\times\!$ regime.}
Tab.~\ref{tab:s2_coverage} reports the eight-cell coverage matrix
for $r_\text{plant}\!\in\!\{1,3\} \times$ geometry
$\in\!\{\text{Warriner},\text{iid}\}$ $\times$ regime
$\in\!\{\text{small-}n / \text{hi-eff}, \text{big-}n / \text{noisy}\}$,
with $K_\text{fit}\!=\!1$ throughout. NB $\geq$ permutation-$Q^2$ in
every cell; gap widest (0.13--0.25) when both methods are away from
ceiling, narrowest (0.02) when both saturate. Two cells sit slightly
above the [0.50, 0.95] target operating band ($r_\text{plant}\!=\!1$
iid small-$n$ at NB$=0.96$; $r_\text{plant}\!=\!3$ Warriner big-$n$
at NB$=0.96$); the NB $\geq$ permutation-$Q^2$ ordering is preserved
in both.

\begin{table}[H]
\centering
\caption{S2 coverage matrix --- detection rate at $\alpha\!=\!0.05$
across $r_\text{plant} \in \{1,3\}$ $\times$ geometry
$\in \{\text{Warriner}, \text{iid}\}$ $\times$ regime
$\in \{\text{small-}n / \text{hi-eff}, \text{big-}n / \text{noisy}\}$.
$K_\text{fit}\!=\!1$ throughout. Top two rows: 400 reps (subset of
S2). Other six: 200 reps each. NB $\geq$ permutation-Q$^2$ in every
cell; gap widest (0.13--0.25) when both methods are away from
ceiling, narrowest (0.01) when both saturate.}
\label{tab:s2_coverage}
\small
\begin{tabular}{@{}cclrrr@{}}
\toprule
$r_\text{plant}$ & \textbf{Geometry} & \textbf{Regime}        & \textbf{NB} & \textbf{perm-Q$^2$} & $\Delta$ \\
\midrule
1 & Warriner & $n\!=\!40$ (valence)         & \textbf{0.81} & 0.73 & $+0.08$ \\
1 & Warriner & $n\!=\!500$ (valence)        & \textbf{0.72} & 0.69 & $+0.03$ \\
1 & iid      & small-$n$ / hi-eff           & \textbf{0.96} & 0.71 & $+0.25$ \\
1 & iid      & big-$n$ / noisy              & \textbf{0.70} & 0.63 & $+0.07$ \\
3 & Warriner & small-$n$ / hi-eff           & \textbf{0.89} & 0.76 & $+0.13$ \\
3 & Warriner & big-$n$ / noisy              & \textbf{0.94} & 0.93 & $+0.01$ \\
3 & iid      & small-$n$ / hi-eff           & \textbf{0.86} & 0.65 & $+0.21$ \\
3 & iid      & big-$n$ / noisy              & \textbf{0.71} & 0.64 & $+0.07$ \\
\bottomrule
\end{tabular}
\end{table}

\subsection{\texorpdfstring{$K^{*}$}{K*} tracks predictive rank, not \texorpdfstring{$\widetilde{r}$}{r-tilde}}
\label{app:overshoot}

The \textpkg{} default
\texttt{pls1\_find\_k\_optimal(selector="r2\_se", fwer\_method="split\_nb")}
picks the \emph{smallest} $k\!\le\!k_{\max}$ whose split-half CV-$R^2$
lies within one SE of the maximum --- the standard 1-SE rule of
\citet{hastie2009elements}, which is biased toward parsimony, not
toward inflation. The selector therefore returns $K^{*}$ equal to the
\emph{predictively justified} rank by construction: every $k\!<\!K^{*}$
has CV-$R^2$ at least one SE below the optimum, so each accepted axis
buys real held-out fit. What the selector does \emph{not} return is
$r_{\mathrm{plant}}$, the number of planted directions, or
$\widetilde r$, the dimension of a sufficient linear projection for
the conditional mean. In a scalar linear model the planted directions
combine into one index, while the number of PLS components needed to
recover its coefficient vector depends on the predictor spectrum
(App.~\ref{app:proof}). Truncating an otherwise unchanged
valid sequence at $K^{*}$ preserves FWER control; validity of its NB
tests remains a separate requirement (\S\ref{sec:methods:test:tnb}).
\texttt{split\_nb} and \texttt{raw\_perm} consume the same $K^{*}$.

On real text geometry with rank-1 truth at $R^2\!=\!0.70$
(Tab.~\ref{tab:toy5_kstar}), the default returns mode
$K^{*}\!=\!6$ on GloVe (mean $5.4$--$6.0$, per-step
$\Delta R^2_{\mathrm{CV}}$ ranging from $\!\approx\!0.166$ at
$K\!=\!2$ down to a few hundredths at $K\!=\!6$) vs.\ mode
$K^{*}\!=\!2$ on iid Gaussian $\mathbf{X}$
($\Delta R^2_{\mathrm{CV}} \approx 0.065$ at $K\!=\!2$, then
saturated). The two cells share planted rank $1$ but differ in $K^{*}$
by a factor of three; the gap is geometric in PLS deflation, not
corpus-specific, and is in both cases the predictively warranted
$K$ under the 1-SE rule (smaller $K$ are demonstrably worse
out-of-sample on GloVe, indistinguishable on iid).

Thus $K^{*}$ is a selected PLS model size, not an estimate of
$\widetilde r$. Fixing $K$ a priori from domain considerations
(the route we take in the body, $K\!=\!3$ for V/A/D) specifies a
model size, while \texttt{selector="bic"} uses a penalized-fit
criterion. Neither establishes the conditional-mean dimension.
The earlier recommendation in the
\textpkg{} v0.x docs of using sequential per-step $K^{*}$ as a
model-selection device is withdrawn.

\subsection{Detection power on real embedding geometries}
\label{app:embedding-power}

The synthetic S2 power benchmark of \S\ref{app:powercalib-full} sweeps
the NB-vs-permutation-$Q^{2}$ gap across
$r_\text{plant}\!\times\!$geometry$\!\times\!$regime cells on
Warriner-shaped iid signal. The panel below is the complementary
real-text replication: planted rank-1 SVD signal injected into the
real $\mathbf{X}$ of Configs~1, 2, and~4 (Warriner GloVe-300,
ANPW GloVe-800, EN XLM-R type-level), 200 reps per cell, NB and
perm-$Q^{2}$ at
$\alpha\!=\!0.05$. The headline question is whether the NB $\geq$
perm-$Q^{2}$ ordering of \S\ref{app:powercalib-full} (8/8 synthetic
cells) holds verbatim on the cross-embedding-family panel of
Tab.~\ref{tab:vadfit}, or whether the contextual-embedding cell
(XLM-R) shows a different power profile under the rank-1 ABTT
pipeline.

\begin{table}[H]
\centering\footnotesize
\setlength{\tabcolsep}{3pt}
\caption{Detection power on real embedding geometries with planted
rank-1 SVD signal in real $\mathbf{X}$, 200 reps per cell, $K_\text{fit}\!=\!1$,
$\alpha\!=\!0.05$. Bracketed numbers are Wilson 95\% CIs.
NB strictly beats perm-$Q^{2}$ in 4/6 cells; the two big-$n$ / noisy
cells where both methods saturate near the ceiling are within Wilson
overlap. The EN-XLM-R cell tracks EN-GloVe to within sampling.}
\label{tab:embedding-power}
\begin{tabular}{@{}llrcccc@{}}
\toprule
\textbf{Geometry} & \textbf{Regime} & $n$ &
\textbf{NB power} & \textbf{perm-$Q^{2}$ power} &
\textbf{NB med.\ $t$} & \textbf{perm med.\ $t$} \\
\midrule
EN GloVe-300 & small-$n$ / hi-eff ($R^{2}\!=\!0.40$) &  40 &
\textbf{0.965} {\scriptsize [.930, .983]} & 0.815 {\scriptsize [.755, .863]} & 0.005\,s & 0.223\,s \\
EN GloVe-300 & big-$n$ / noisy ($R^{2}\!=\!0.05$)    & 500 &
0.945 {\scriptsize [.904, .969]} & 0.950 {\scriptsize [.910, .973]} & 0.142\,s & 3.730\,s \\
PL GloVe-800 & small-$n$ / hi-eff ($R^{2}\!=\!0.40$) &  40 &
\textbf{0.865} {\scriptsize [.811, .906]} & 0.620 {\scriptsize [.551, .684]} & 0.014\,s & 0.757\,s \\
PL GloVe-800 & big-$n$ / noisy ($R^{2}\!=\!0.05$)    & 500 &
\textbf{0.870} {\scriptsize [.816, .910]} & 0.820 {\scriptsize [.761, .867]} & 0.517\,s & 28.27\,s \\
EN XLM-R     & small-$n$ / hi-eff ($R^{2}\!=\!0.40$) &  40 &
\textbf{0.965} {\scriptsize [.930, .983]} & 0.800 {\scriptsize [.739, .850]} & 0.014\,s & 0.779\,s \\
EN XLM-R     & big-$n$ / noisy ($R^{2}\!=\!0.05$)    & 500 &
0.980 {\scriptsize [.950, .992]} & 0.975 {\scriptsize [.943, .989]} & 0.501\,s & 26.97\,s \\
\bottomrule
\end{tabular}
\end{table}

\subsection{Inference: power and Tecator non-text replication}
\label{app:nontext}

\subsubsection{S2 --- Power point checks across eight cells}
Eight ($n$, $R^2$, $r_{\mathrm{plant}}$) cells on synthetic
Warriner geometry; 400 reps per cell; NB and perm-$Q^2$ at
$\alpha\!=\!0.05$. The eight-cell ordering is the headline:
NB $\geq$ perm-$Q^2$ on every cell with gap 0.01--0.25, widest in
the noisy regime. The continuous $n$-sweep visualisation is
Fig.~\ref{fig:s3-power-curves}.

\begin{table}[H]
\centering
\caption{Detection rate at $\alpha\!=\!0.05$ (400 reps $\times$ 2
cells). Naive-$t$ (both cells) and \texttt{jack.test} cell A are
non-comparable: those tests are uncalibrated under $H_0$ in those
regimes (FPR 0.3875/0.4075 for naive-$t$, 0.260 for jack-A on the S1
runs), so apparent power is inflated by Type-I leakage;
reported for sanity only.}
\label{tab:s2-power}
\small
\begin{tabular}{@{}lrrr@{}}
\toprule
\textbf{Test}    & \textbf{Cell A power} & \textbf{Cell B power} & \textbf{Median runtime A / B (s)} \\
\midrule
NB                  & \textbf{0.8125} & \textbf{0.715}  & 0.005 / 0.12   \\
Permutation-$Q^2$   & 0.7325          & 0.685           & 0.42 / 7.48    \\
\texttt{jack.test}  & 0.57            & 0.16            & 0.003 / 0.04   \\
Naive $t$           & 0.99            & 0.96            & 0.018 / 0.12   \\
\bottomrule
\end{tabular}
\end{table}

NB exceeds permutation-$Q^2$ in power: under the
normal-approximation sample-size relation
$n_{\mathrm{perm}}/n_{\mathrm{NB}} = \big((z_{1-\alpha} + z_{\pi_{\mathrm{NB}}}) / (z_{1-\alpha} + z_{\pi_{\mathrm{perm}}})\big)^2$
\citep{lehmann2005testing}, permutation-$Q^2$ needs more samples
to reach matching power (one-sided $\alpha\!=\!0.05$).%
\footnote{The sample-size formula uses one-sided $\alpha$ because the
alternative ``the supervised subspace predicts $\mathbf{y}$ better
than chance'' is genuinely one-sided in correlation. Eq.~\eqref{eq:nb}'s
two-sided $p$ is what we report operationally (it lets the test fire
on unexpected sign-flipped predictors); under the dominant one-sided
alternative the wrong-tail mass is negligible and the two-sided test
behaves like a one-sided test at $\alpha$.} NB's
wall-clock cost matches the closed-form fit-count ratio
$K_{\mathrm{CV}}\!\cdot\!P/J$ (NB does $J$ PLS fits, perm-$Q^{2}$
does $K_{\mathrm{CV}}\!\cdot\!P$ at the same $O(Knd)$ per fit, so the
ratio is independent of $n,d$ and linear in $P$): at the
chemometrics defaults ($J\!=\!50$, $K_{\mathrm{CV}}\!=\!5$,
$P\!=\!1000$) the predicted $100\times$ ratio is realized as
$\approx\!82\times$ at $n\!=\!40$ and $\approx\!61\times$ at
$n\!=\!500$ on our hardware (Tab.~\ref{tab:s2-power}). These
chemometrics defaults sit at the low end of cross-field resampling
practice: neuroimaging permutation testing routinely uses
$P\!\geq\!5000$~\citep{smith2009threshold,eklund2016cluster}, and
PLS-SEM bootstrap conventionally runs $\sim\!5000$
resamples~\citep{ray2022seminr}. At those confirmatory standards
($P\!=\!10^{4}$) the fit-count ratio rises to
$K_{\mathrm{CV}}\!\cdot\!P/J\!=\!10^{3}$, so the
$\approx\!50\!-\!100\times$ headline understates the speedup
outside chemometrics. The full
eight-cell coverage matrix
($r_{\mathrm{plant}}\!\in\!\{1, 3\}\times$ geometry
$\{\mathrm{Warriner}, \mathrm{iid}\}\times$ regime
$\{\mathrm{small\text{-}n / hi\text{-}eff},
\mathrm{big\text{-}n / noisy}\}$) is in
Tab.~\ref{tab:s2_coverage}.

\subsubsection{S3 --- Power vs.\ \texorpdfstring{$n$}{n} on synthetic geometries}
The body figure (Fig.~\ref{fig:s3-power-curves}, \S\ref{sec:exp:powercalib})
sweeps $n$ on the synthetic Warriner geometry at $\alpha\!=\!0.05$, comparing
NB to perm-$Q^{2}$ across signal regimes ($R^{2}\!\in\!\{0.05, 0.40\}$) and
planted axes ($r_{\mathrm{plant}}\!\in\!\{1,3\}$). Per-cell tabular outputs are
in \texttt{inference/results/s3\_power\_curves/runs.csv}.

\subsubsection{T3 --- Power vs.\ \texorpdfstring{$n$}{n} on Tecator NIR}
Figure~\ref{fig:t3-power-curves} repeats the S3 sweep on Tecator NIR
to confirm the ordering carries over from synthetic Warriner geometry
to a real chemometric design.
\begin{figure}[H]
\centering
\includegraphics[width=\linewidth]{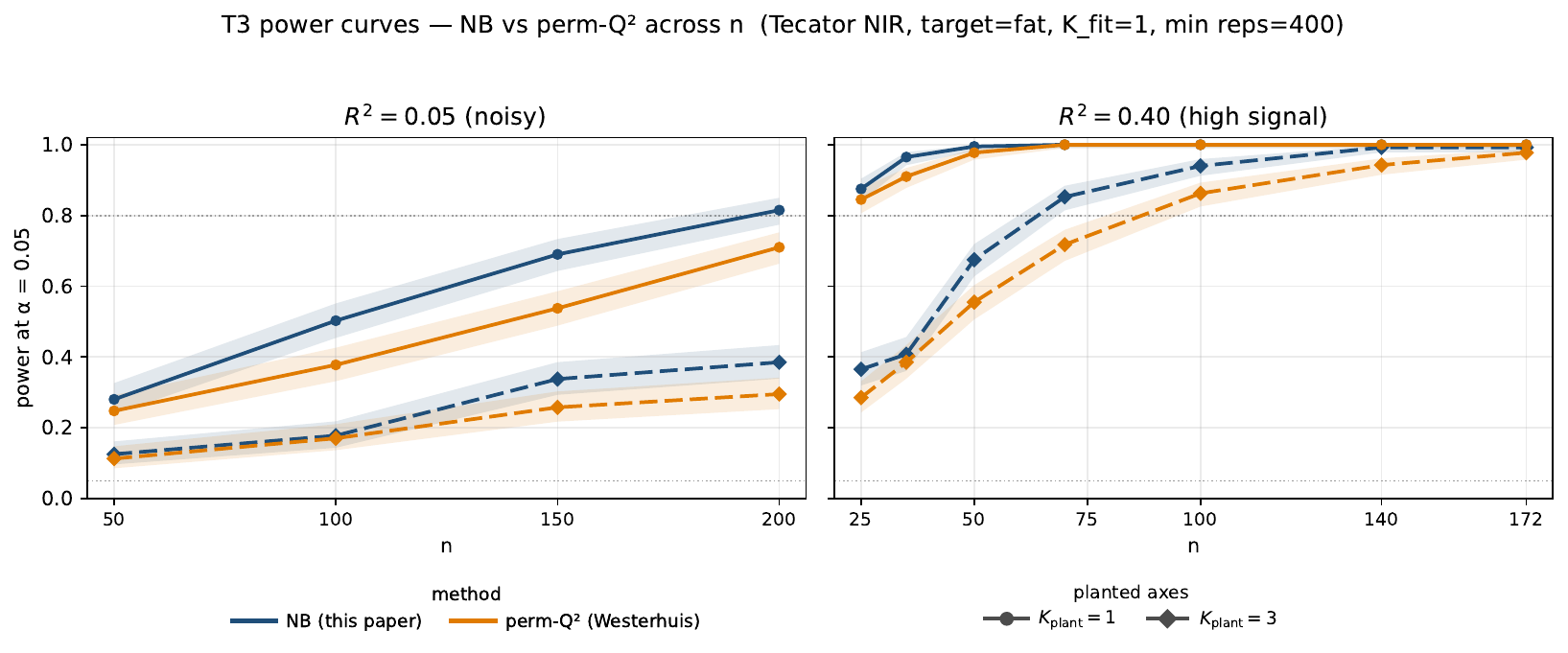}
\caption{Empirical power vs.\ $n$ on Tecator NIR; same style and
conventions as Fig.~\ref{fig:s3-power-curves}. The
$R^2\!=\!0.05$ curve plateaus below 80\% at $n\!=\!200$ for both
NB and perm-$Q^2$; this ceiling is set by the planted signal's
information content (a deliberately weak $R^2\!=\!0.05$ effect),
not by the chemometric design --- both tests track the same
ceiling, which is what this comparison is meant to read.}
\label{fig:t3-power-curves}
\end{figure}

\subsubsection{Sample-size equivalence at matched power}
\label{app:samplered}
Reading Fig.~\ref{fig:s3-power-curves} and
Fig.~\ref{fig:t3-power-curves} at fixed power
$\pi\!\in\!\{0.5, 0.7, 0.8, 0.9\}$ gives matched-power sample
sizes $n_{\mathrm{NB}}(\pi)$ and $n_{\mathrm{perm}}(\pi)$ via
linear interpolation of the empirical curves; the relative
reduction $1\!-\!n_{\mathrm{NB}}/n_{\mathrm{perm}}$ summarizes
how many fewer observations NB needs to match the gold
standard's detection rate. Across all 35 reachable
(cell $\times$ $\pi$) combinations the median reduction is
$\approx\!22\%$, with min $\approx\!5\%$ in flat-power tails of
low-effect cells and max $\approx\!37\%$ at $\pi\!=\!0.8$ in
the $r^{2}\!=\!0.4$, iid-Gaussian regime. Restricting to the
conventional $\pi\!=\!0.8$ target (9 cells where both curves
cross 80\%) the median rises to $\approx\!26\%$ (range
$8\%$--$37\%$). The abstract's $\sim\!20\%$ headline reports
the all-cell median rounded down..

\subsubsection{Calibration shape: NB-asymptotic on highly-collinear designs}
Figure~\ref{fig:pp-calibration} visualises the
FPR-passes-while-shape-fails pattern flagged in
\S\ref{sec:exp:calibration}. Both panels are PP plots of $p$
under $H_0$ at $n\!=\!40$ (400 reps): empirical CDF on the
$y$-axis against the diagonal-uniform reference. On synthetic
iid Gaussian (left) NB-asymptotic and perm-$Q^2$ both track the
diagonal. On Tecator NIR (right) perm-$Q^2$ tracks the diagonal
but NB-asymptotic sags below uniform through the bulk and lifts
above near the upper tail; the local crossing near
$\alpha\!=\!0.05$ explains the FPR pass and the bulk-and-tail
deviation explains the KS fail. The mechanism is the Fisher-$z$ asymptotic
null running into a design that is essentially rank one: on the
column-standardized Tecator $\mathbf{X}$ ($215\!\times\!100$, the same
matrix the test sees) the stable rank
$\|\mathbf{X}\|_F^{2}/\|\mathbf{X}\|_2^{2}$ is 1.014 and the first
principal component carries 98.6\% of the variance
(\texttt{inference/results/r6\_pipelines/csvs/cells.csv}). The
pre-run rule of \S\ref{sec:methods:validity} flags this design
(stable rank below 3) and sends it to the exact NB-permutation test
of \S\ref{sec:methods:exactperm}, which restores uniformity
(App.~\ref{app:rule-grid}).

\begin{figure}[h]
\centering
\includegraphics[width=\linewidth]{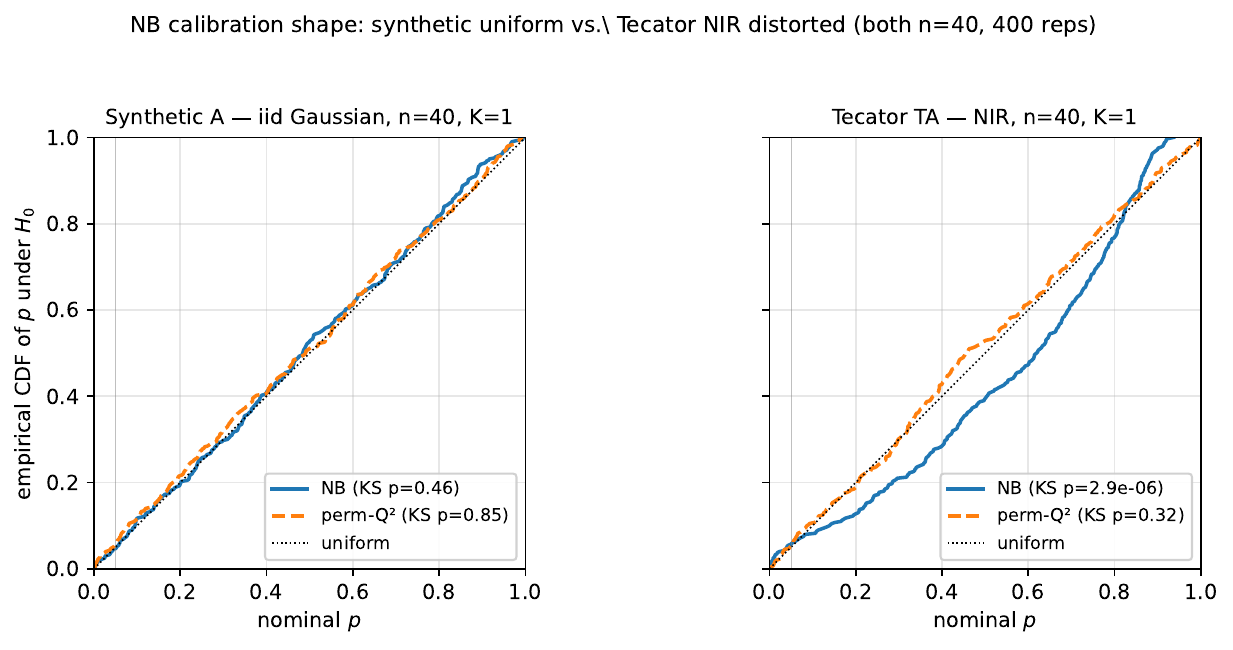}
\caption{PP plot of NB and perm-$Q^2$ $p$-values under $H_0$ on
Synthetic A (iid Gaussian, $n\!=\!40$, $K\!=\!1$; left) and
Tecator TA (NIR, $n\!=\!40$, $K\!=\!1$; right). 400 reps per
panel. Diagonal = uniform-$H_0$ reference; vertical gray line
marks the $\alpha\!=\!0.05$ rejection threshold. NB-asymptotic
(blue solid) tracks uniform on synthetic but deviates strongly on
Tecator (KS $p\!\sim\!10^{-6}$); perm-$Q^2$ (orange dashed) tracks
uniform on both.}
\label{fig:pp-calibration}
\end{figure}

\subsection{Validity rule: the full null grid, the \texttt{plsdof} construction, and cost}
\label{app:rule-grid}

\paragraph{The grid.}
Table~\ref{tab:rule-grid} lists every null cell behind the rule of
\S\ref{sec:methods:validity}; Tab.~\ref{tab:cal-combined} in the body
shows eight of them. Five designs. \emph{iid Gaussian}: a
$4{,}000\!\times\!300$ pool of independent standard normals.
\emph{Block-collinear}: five blocks of 60 columns, each block sharing
one latent factor at pairwise correlation $0.9$. \emph{Dominant
factor}: all 300 columns share one factor at pairwise correlation
$0.9$. \emph{NIR gasoline}: the real $60\!\times\!401$ octane spectra.
\emph{Warriner GloVe-300}: the real Config~1 design matrix. Every rep
draws $n$ rows from the pool; the null is $\mathbf{y}$ drawn
independently of $\mathbf{X}$ on the synthetic pools and a permutation
of the real outcome on the two real designs, so it is true by
construction in both cases. Replicates: 5{,}000 per cell (Warriner
2{,}000). Stable rank is computed by the codebase's own
\texttt{spectrum()} on the column-standardized $\mathbf{X}$ of the
pool at that $n$, i.e.\ on the matrix the test sees. Scripts:
\texttt{inference/r1\_nearsingular.py} (\texttt{run}, then
\texttt{summarise} for \texttt{r1\_calibration.csv}),
\texttt{r4\_breadth.py}, \texttt{r2\_ruler.py}, and
\texttt{r7\_signflip.py} for the sign-flip mechanism
(Tabs.~\ref{tab:signflip-bins}--\ref{tab:signflip-sweep}). Tecator is listed
for completeness: its NB-asymptotic null is the 400-replicate run of
Tab.~\ref{tab:cal-combined} in the submitted version
(\texttt{t1\_tecator/}), and its exact-test null is the 2{,}000-replicate
PLS1 arm of App.~\ref{app:beyond-pls}.

\begin{table}[H]
\centering
\caption{All null cells behind the validity rule. Stable rank
$\|\mathbf{X}\|_F^{2}/\|\mathbf{X}\|_2^{2}$ of the column-standardized
design; \emph{flag} if $n\!<\!25$ or stable rank $<\!3$ (the
\emph{$n$} column marks a flag raised by the sample-size gate alone).
Level is the empirical rejection rate at $\alpha\!=\!.05$ and $.10$;
KS is the Kolmogorov--Smirnov $p$ of the null $p$-values against
Uniform$[0,1]$. 5{,}000 replicates per cell unless marked
($^{\mathrm{w}}$ 2{,}000; $^{\mathrm{t}}$ NB-asymptotic 400, exact
2{,}000). Wilson 95\% half-width at 5{,}000 replicates: $\pm 0.006$ at
$\alpha\!=\!.05$, $\pm 0.008$ at $.10$.}
\label{tab:rule-grid}
\footnotesize
\setlength{\tabcolsep}{4pt}
\begin{tabular}{@{}lrrl rrr rrr@{}}
\toprule
 & & & & \multicolumn{3}{c}{\textbf{NB-asymptotic}} & \multicolumn{3}{c}{\textbf{Exact NB-permutation}} \\
\cmidrule(lr){5-7}\cmidrule(lr){8-10}
\textbf{Design} & $n$ & \textbf{st.\ rank} & \textbf{rule} & $\alpha\!=\!.05$ & $\alpha\!=\!.10$ & \textbf{KS} & $\alpha\!=\!.05$ & $\alpha\!=\!.10$ & \textbf{KS} \\
\midrule
iid Gaussian     &  10 &  7.08 & flag ($n$) & 0.004 & 0.020 & $10^{-70}$ & 0.056 & 0.108 & 0.08 \\
                 &  20 & 12.98 & flag ($n$) & 0.020 & 0.059 & $2{\times}10^{-13}$ & 0.049 & 0.102 & 0.60 \\
                 &  40 & 22.04 & pass & 0.038 & 0.083 & $1.3{\times}10^{-3}$ & 0.048 & 0.096 & 0.81 \\
                 &  80 & 36.07 & pass & 0.044 & 0.091 & 0.23 & 0.051 & 0.099 & 0.77 \\
                 & 160 & 55.28 & pass & 0.050 & 0.100 & 0.72 & 0.052 & 0.106 & 0.97 \\
                 & 320 & 78.83 & pass & 0.050 & 0.099 & 0.11 & 0.047 & 0.099 & 0.20 \\
\midrule
block-collinear  &  10 &  2.63 & flag & 0.017 & 0.048 & $2{\times}10^{-30}$ & 0.054 & 0.106 & 0.11 \\
                 &  20 &  3.11 & flag ($n$) & 0.038 & 0.075 & $10^{-11}$ & 0.051 & 0.096 & 0.55 \\
                 &  40 &  3.80 & pass & 0.050 & 0.095 & $7.6{\times}10^{-6}$ & 0.050 & 0.103 & 0.55 \\
                 &  80 &  4.18 & pass & 0.050 & 0.092 & $1.8{\times}10^{-6}$ & 0.049 & 0.103 & 0.17 \\
                 & 160 &  4.55 & pass & 0.054 & 0.100 & $3.2{\times}10^{-3}$ & 0.050 & 0.101 & 0.25 \\
                 & 320 &  4.78 & pass & 0.052 & 0.095 & $3.0{\times}10^{-4}$ & 0.048 & 0.093 & 0.13 \\
\midrule
dominant factor  &  10 &  1.12 & flag & 0.042 & 0.067 & $10^{-59}$ & 0.053 & 0.104 & 0.87 \\
                 &  20 &  1.13 & flag & 0.039 & 0.067 & $10^{-70}$ & 0.050 & 0.101 & 0.17 \\
                 &  40 &  1.10 & flag & 0.041 & 0.061 & $2{\times}10^{-77}$ & 0.050 & 0.100 & 0.40 \\
                 &  80 &  1.12 & flag & 0.039 & 0.066 & $2{\times}10^{-76}$ & 0.053 & 0.103 & 0.024 \\
                 & 160 &  1.12 & flag & 0.041 & 0.064 & $3{\times}10^{-64}$ & 0.052 & 0.101 & 0.21 \\
                 & 320 &  1.11 & flag & 0.037 & 0.064 & $2{\times}10^{-63}$ & 0.053 & 0.104 & 0.23 \\
\midrule
NIR gasoline     &  20 &  1.40 & flag & 0.048 & 0.079 & $2{\times}10^{-38}$ & 0.049 & 0.102 & 0.77 \\
                 &  30 &  1.41 & flag & 0.047 & 0.074 & $3{\times}10^{-40}$ & 0.052 & 0.100 & 0.18 \\
                 &  60 &  1.39 & flag & 0.049 & 0.077 & $4{\times}10^{-31}$ & 0.050 & 0.098 & 0.92 \\
\midrule
Warriner GloVe-300$^{\mathrm{w}}$ & 20 & 11.14 & flag ($n$) & 0.037 & 0.077 & $8.9{\times}10^{-6}$ & 0.057 & 0.104 & 0.40 \\
                 &  30 & 15.06 & pass & 0.036 & 0.075 & 0.024 & 0.048 & 0.098 & 0.49 \\
                 &  40 & 17.12 & pass & 0.035 & 0.085 & 0.18 & 0.042 & 0.100 & 0.88 \\
                 &  60 & 21.13 & pass & 0.042 & 0.086 & 0.12 & 0.043 & 0.093 & 0.53 \\
\midrule
Tecator NIR$^{\mathrm{t}}$ & 40 & 1.014 & flag & 0.055 & 0.080 & $2.9{\times}10^{-6}$ & 0.052 & 0.094 & 0.13 \\
                 & 172 & 1.015 & flag & 0.048 & 0.080 & $3.4{\times}10^{-3}$ & 0.050 & 0.095 & 0.34 \\
\bottomrule
\end{tabular}
\end{table}

\paragraph{What the rule catches, and what it does not.}
On the 25 cells of the rule's own grid (Tecator excluded), no flagged
cell holds both its $\alpha\!=\!.05$ and its $\alpha\!=\!.10$ level
within the Wilson interval of nominal, and every cleared cell holds
both or is conservative: cleared NB-asymptotic levels span
$0.035$--$0.054$ at $\alpha\!=\!.05$ and $0.074$--$0.100$ at $.10$,
flagged cells have KS $p$ from $9{\times}10^{-6}$ down to
$2{\times}10^{-77}$. Five cleared cells (iid $n\!=\!40, 80$;
Warriner $n\!=\!30, 40, 60$) are conservative enough that the
Wilson interval excludes nominal at $\alpha\!=\!.05$ or $.10$, which
costs power only. KS is not itself the discriminator: two cleared
block-collinear cells ($n\!=\!40, 80$) have KS $p$ of
$7.6{\times}10^{-6}$ and $1.8{\times}10^{-6}$, inside the flagged
range, while holding level at both $\alpha$. The exact test holds
level everywhere: $0.042$--$0.057$ at $\alpha\!=\!.05$,
$0.093$--$0.108$ at $.10$, $0.008$--$0.014$ at $.01$, KS never below
$0.024$. The threshold on stable rank was set on the decaying-spectrum
sweep of \texttt{r2\_ruler.py} (eigenvalues $\propto k^{-a}$,
$a\!\in\!\{0,\dots,6\}$ rotated into the columns, $n$ from 20 to 320,
400 replicates per cell) and then applied unchanged to the five
designs above, none of which was used to set it.

\paragraph{The $\alpha\!=\!.01$ concession.}
On flagged designs NB-asymptotic at $\alpha\!=\!.01$ ranges from
$0.02\times$ nominal (iid, $n\!=\!10$: 0.0002) to $1.9\times$
(dominant factor, $n\!=\!80$: 0.019); on the concentrated designs the
inflation at $.01$ and the conservativeness at $.05$ and $.10$ are one
effect, the shape failure of the next paragraph. On the cleared
block-collinear cells it inflates to
$0.013$--$0.016$ at $n\!\geq\!80$ ($1.3$--$1.6\times$). Its useful
range is therefore $\alpha\!=\!.05$ and $.10$
(\S\ref{sec:disc:limitations}).

\paragraph{Why a concentrated spectrum breaks the reference.}
The NB correction plugs in $\rho_{0}\!=\!n_{2}/(n_{1}+n_{2})\!=\!\tfrac12$ for
the correlation between overlapping splits, giving the variance
term $\rho_0/(1-\rho_0)=n_2/n_1=1$. The \texttt{jeff} arm of
\texttt{r1\_nearsingular.py} (400 replicates per cell) estimates that
correlation from the shrinkage of the across-split variance of $z$
relative to the single-split null variance $1/(n_{2}\!-\!3)$,
$\hat\rho\!=\!\max\{0,\,1-s^{2}(n_{2}\!-\!3)\}$: mean $\hat\rho$ is
$0.48$--$0.50$ on iid Gaussian at $n\!\geq\!80$, $0.45$--$0.48$ on
block-collinear at $n\!\geq\!80$, $0.38$--$0.41$ on the dominant-factor
pool at every $n$, and $0.40$--$0.41$ on NIR gasoline. The
concentrated designs move the correlation \emph{down}, not toward 1,
and \texttt{r7\_signflip.py} shows why. Write the one-factor design as
$\mathbf{X}\!\approx\!\sqrt{\lambda}\,\mathbf{u}\mathbf{v}^{\top}$. The
PLS1 weight on a training half is
$\mathbf{X}_{R}^{\top}\mathbf{y}_{R}\propto\mathrm{sign}(\mathbf{u}_{R}^{\top}\mathbf{y}_{R})\,\mathbf{v}$,
and the OLS refit slope
$\mathbf{t}_{R}^{\top}\mathbf{y}_{R}/\mathbf{t}_{R}^{\top}\mathbf{t}_{R}
=\|\mathbf{X}_{R}^{\top}\mathbf{y}_{R}\|/\mathbf{t}_{R}^{\top}\mathbf{t}_{R}$
is always positive, so the held-out correlation of a split is
$\mathrm{sign}(\mathbf{u}_{R}^{\top}\mathbf{y}_{R})\,
\mathrm{sign}(\mathbf{u}_{T}^{\top}\mathbf{y}_{T})\,
|\mathrm{cor}(\mathbf{u}_{T},\mathbf{y}_{T})|$: positive exactly when
the two halves of $\mathbf{y}$ project onto $\mathbf{u}$ with the same
sign. The two projections sum to the full-sample
$a\!=\!\mathrm{cor}(\mathbf{u}_{1},\mathbf{y})$, so $\bar z$ is a
function of one scalar. Near $a\!=\!0$ the halves must disagree and
every split is negative; at large $|a|$ they agree and every split is
positive; in between the sign flips from split to split. On the
\texttt{col1} construction at $n\!=\!40$, $J\!=\!50$ (4{,}000
replicates, splits held fixed across replicates) the sign of the
held-out correlation equals
$\mathrm{sign}(\mathbf{u}_{R}^{\top}\mathbf{y}_{R}\cdot\mathbf{u}_{T}^{\top}\mathbf{y}_{T})$
on 99.0\% of splits, and binning replicates by $|a|$
(Tab.~\ref{tab:signflip-bins}) shows the three regimes: 47\% of
replicates sit at $|a|\!<\!0.1$ with $\bar z\!<\!0$ and a quarter of
the splits positive, 41\% at $0.1$--$0.25$ with mixed signs and an
across-split spread near the single-split null variance, and 97\% of
the rejections at $\alpha\!=\!.05$ come from the 12\% at
$|a|\!\geq\!0.25$. The null of $t_{\mathrm{NB}}$ is therefore not
$t_{J-1}$ in shape (Tab.~\ref{tab:signflip-quantiles}): its median is
$-0.25$, its $0.90$ and $0.95$ quantiles fall below the reference
(conservative at $\alpha\!=\!.10$ and $.05$), its $0.99$ and $0.999$
quantiles above it (inflated at $.01$). No scalar correction to the
variance repairs a reference of the wrong shape; permuting
$\mathbf{y}$ redraws $a$ and reproduces it, which is why the exact
test holds level on these designs. Sweeping the factor strength
(Tab.~\ref{tab:signflip-sweep}) moves the effect with stable rank:
from stable rank 1.1 to 8 the across-replicate split correlation
rises from $0.33$ to $0.43$, the share of splits whose sign follows
the two projections falls from $0.99$ to $0.77$, the gap between the
absolute and signed cosine of the split weights closes, and the
$\alpha\!=\!.10$ level moves from $0.067$ to $0.084$. The
$\alpha\!=\!.01$ inflation ($0.017$--$0.021$) persists as long as a
factor is present and vanishes only at the iid endpoint, which at
$n\!=\!40$ and $p\!=\!300$ is conservative at every $\alpha$
($0.003/0.029/0.078$), as is the r1 iid cell at that $n$
(Tab.~\ref{tab:rule-grid}). At stable rank 3.3, just above the rule's
threshold, the levels are $0.047$ and $0.077$, on the conservative
side as the rule requires of a cleared cell; at stable rank 1.1 the
sweep reproduces the r1 \texttt{col1} cell at $n\!=\!40$
($0.017/0.044/0.067$ against $0.017/0.041/0.061$ at
$\alpha\!=\!.01/.05/.10$).

\begin{table}[H]
\centering
\caption{One-factor design (\texttt{col1} construction, pairwise
$r\!=\!0.9$), $n\!=\!40$, $J\!=\!50$, 4{,}000 null replicates with the
splits held fixed, binned by $|a|\!=\!|\mathrm{cor}(\mathbf{u}_{1},\mathbf{y})|$.
$s^{2}/\sigma_{0}^{2}$ is the across-split variance of $z$ relative to
the single-split null variance $1/(n_{2}\!-\!3)$; $\hat\rho$ is the
\texttt{jeff} estimate $\max\{0,1-s^{2}/\sigma_{0}^{2}\}$.}
\label{tab:signflip-bins}
\footnotesize
\setlength{\tabcolsep}{4pt}
\begin{tabular}{@{}lrrrrrrrr@{}}
\toprule
$|a|$ & \textbf{share} & $\bar z$ & $s^{2}/\sigma_{0}^{2}$ & \textbf{splits $r_j\!>\!0$} & $\hat\rho$ & \textbf{rej.\ .01} & \textbf{rej.\ .05} & \textbf{rej.\ .10} \\
\midrule
$[0, 0.10)$    & 0.47 & $-0.107$ & 0.38 & 0.25 & 0.63 & 0.000 & 0.000 & 0.000 \\
$[0.10, 0.25)$ & 0.41 & $0.036$  & 0.91 & 0.66 & 0.19 & 0.001 & 0.003 & 0.015 \\
$[0.25, 0.40)$ & 0.11 & $0.271$  & 0.92 & 0.93 & 0.26 & 0.101 & 0.275 & 0.450 \\
$[0.40, 1]$    & 0.01 & $0.485$  & 0.68 & 0.99 & 0.41 & 0.711 & 0.842 & 0.974 \\
\bottomrule
\end{tabular}
\end{table}

\begin{table}[H]
\centering
\caption{Null quantiles of $t_{\mathrm{NB}}$ on the same 4{,}000
replicates against the $t_{49}$ reference the asymptotic test uses.}
\label{tab:signflip-quantiles}
\footnotesize
\setlength{\tabcolsep}{6pt}
\begin{tabular}{@{}lrrrrrr@{}}
\toprule
\textbf{quantile} & 0.50 & 0.80 & 0.90 & 0.95 & 0.99 & 0.999 \\
\midrule
$t_{\mathrm{NB}}$ & $-0.25$ & 0.47 & 0.95 & 1.50 & 2.74 & 3.62 \\
$t_{49}$          & 0.00    & 0.85 & 1.30 & 1.68 & 2.41 & 3.27 \\
\bottomrule
\end{tabular}
\end{table}

\begin{table}[H]
\centering
\caption{Sweep over the pairwise correlation $r$ inside the factor
($n\!=\!40$, $p\!=\!300$, $J\!=\!50$, 2{,}000 null replicates per row,
splits held fixed). $\rho$ is the mean pairwise correlation between
split estimates across replicates; $\hat\rho$ the mean \texttt{jeff}
estimate; $|\cos|$ and $\cos$ the mean absolute and signed cosine
between the weight vectors of two splits; \textbf{sign} the share of
splits whose held-out sign equals
$\mathrm{sign}(\mathbf{u}_{R}^{\top}\mathbf{y}_{R}\cdot\mathbf{u}_{T}^{\top}\mathbf{y}_{T})$.}
\label{tab:signflip-sweep}
\footnotesize
\setlength{\tabcolsep}{4pt}
\begin{tabular}{@{}lrrrrrrrrr@{}}
\toprule
$r$ & \textbf{st.\ rank} & $\rho$ & $\hat\rho$ & $|\cos|$ & $\cos$ & \textbf{sign} & \textbf{lvl .01} & \textbf{lvl .05} & \textbf{lvl .10} \\
\midrule
0.9       & 1.1  & 0.33 & 0.38 & 0.67 & 0.41 & 0.99 & 0.017 & 0.044 & 0.067 \\
0.7       & 1.4  & 0.35 & 0.40 & 0.55 & 0.44 & 0.97 & 0.021 & 0.044 & 0.071 \\
0.5       & 2.0  & 0.35 & 0.39 & 0.50 & 0.46 & 0.94 & 0.017 & 0.045 & 0.077 \\
0.3       & 3.3  & 0.36 & 0.39 & 0.48 & 0.47 & 0.89 & 0.017 & 0.047 & 0.077 \\
0.1       & 8.4  & 0.43 & 0.43 & 0.47 & 0.47 & 0.77 & 0.019 & 0.054 & 0.084 \\
0 (iid)   & 22.2 & 0.43 & 0.44 & 0.47 & 0.47 & 0.58 & 0.003 & 0.029 & 0.078 \\
\bottomrule
\end{tabular}
\end{table}

\paragraph{The \texttt{plsdof} construction.}
\citet{kramer2011dof} supply a degrees-of-freedom estimate for PLS,
shipped as \texttt{plsdof}~\citep{plsdof_pkg}; they do not propose it
as a significance test and never validate its level, so any test built
on it is our construction, not theirs. We built the obvious one: an
omnibus $F$ at $K\!=\!1$ with $\mathrm{DoF}\!-\!1$ numerator and
$n\!-\!\mathrm{DoF}$ residual degrees of freedom, where the fitted DoF
counts the intercept and is capped by the package at
$\min(n\!-\!1, p\!+\!1)$ (\texttt{inference/r1\_plsdof.R}), run on the
same 5{,}000 null draws as the grid above. On the block-collinear pool it holds level
($0.044$--$0.052$ at $\alpha\!=\!.05$). On the dominant-factor pool the
numerator DoF, $\mathrm{DoF}\!-\!1$, is non-positive on 16--23\% of null draws at every $n$, so
no $F$ exists there, and among the draws with a defined $F$ the level is
$0.077$--$0.083$. On iid Gaussian at $n\!=\!10$ the fitted DoF sits at
its ceiling of $n\!-\!1$ on 91\% of draws and the level is $0.0004$; at
$n\!=\!20$ the ceiling share is 49\% and the level $0.007$. Whether
these failures indict the DoF estimate or our wrapper, neither we nor a
reader can say from this construction, which is why it is reported
here and not benchmarked in the body. Cost is the second reason: the
exact DoF pass is $O(n^{3})$~\citep{kramer2009lanczos}, measured at
$0.07$\,s at $n\!=\!500$ and $13.5$\,s at $n\!=\!4{,}000$ on Warriner
($p\!=\!300$), against
$0.03$\,s and $0.34$\,s for NB-asymptotic at $n\!=\!400$ and $3{,}200$;
the $O(n^{2})$ Lanczos approximation trades that for an accuracy that
depends on how fast the eigenvalues of $\mathbf{X}$ decay, so the
calibration question would have to be settled twice.

\paragraph{Cost of the exact test.}
Table~\ref{tab:cost} gives wall-clock per test for NB-asymptotic, the
no-refit exact NB-permutation of \S\ref{sec:methods:exactperm}
($B\!=\!1{,}000$), and perm-$Q^{2}$ ($5$-fold, $1{,}000$
permutations); each is the minimum over repeated passes, first pass
discarded, on a locked clock. With the
splits fixed and $K\!=\!1$ the exact test is two matrix products per
split, so its cost is $\tfrac{1}{4}n^{2}(p+B)$ against
$J\,n\,p$ for NB-asymptotic: the ratio falls toward one as $p$ grows
and rises with $n$. The refit implementation that the submitted
version priced, timed at a reduced budget and rescaled to
$B\!=\!1{,}000$, costs $4.3$\,s at $p\!=\!75$ and $73$\,s at
$p\!=\!1{,}200$ on the same $n\!=\!200$ cells, i.e.\ $50$--$400\times$ the
no-refit form.

\begin{table}[H]
\centering
\caption{Wall-clock seconds per test. Left: $p$ sweep at $n\!=\!200$
(synthetic pool). Right: Warriner GloVe-300 ladder at $p\!=\!300$, with
perm-$Q^{2}$ as a fourth method. Ratios are relative to NB-asymptotic.}
\label{tab:cost}
\footnotesize
\setlength{\tabcolsep}{4pt}
\begin{tabular}{@{}rrrr@{\hspace{1.4em}}rrrrrr@{}}
\toprule
$p$ & \textbf{NB} & \textbf{exact} & \textbf{ratio} & $n$ & \textbf{NB} & \textbf{exact} & \textbf{ratio} & \textbf{perm-$Q^{2}$} & \textbf{ratio} \\
\midrule
   75 & 0.0042 & 0.088 & 20.8 &   400 & 0.033 &  0.28 &  8.6 &   3.8 & 117 \\
  150 & 0.0082 & 0.094 & 11.4 &   800 & 0.068 &  0.96 & 14.1 &   7.9 & 116 \\
  300 & 0.0166 & 0.106 &  6.4 & 1{,}600 & 0.144 &  2.14 & 14.8 &  17.3 & 120 \\
  600 & 0.0339 & 0.131 &  3.9 & 3{,}200 & 0.344 &  4.32 & 12.6 &  43.0 & 125 \\
1{,}200 & 0.0705 & 0.180 &  2.5 & 6{,}400 & 0.779 & 10.8 & 13.8 &  89.9 & 115 \\
      &        &       &      & 13{,}365 & 3.22 & 24.0 &  7.5 & 387.7 & 120 \\
\bottomrule
\end{tabular}
\end{table}

\subsection{Per-component FWER: proof for one split with train-only deflation}
\label{app:fwer-proof}

The fixed-sequence argument needs validity at the first true null,
even when earlier components carry signal. We establish that property
for one independent training/test split, the data-splitting
construction of \citet{cox1975datasplitting}: all components are fitted on
the training half, and each test asks whether the next score is
positively correlated with the residual of the preceding fitted
components. Conditioning on training makes both the component order
and these hypotheses fixed. The guarantee below is finite-sample
under Gaussianity; it does not cover the repeated-split NB statistic
or full-data deflation.

\paragraph{Setup.}
Let $(\mathbf{x}_i,y_i)\sim P$, $i=1,\dots,n$. Choose a partition
$(\mathcal A,\mathcal B)$ independently of the data, with
$|\mathcal A|=n_1$ and $|\mathcal B|=n_2$, and hold it fixed.
Write $D_{\mathcal A}=\{(\mathbf x_i,y_i):i\in\mathcal A\}$ for the
training half. On $D_{\mathcal{A}}$, standardize
$\mathbf{X}$ and $\mathbf{y}$ with the training means and standard
deviations and run $K$ NIPALS steps (App.~\ref{app:nipals}), giving
$(\mathbf{w}_k,\mathbf{p}_k,q_k)_{k\le K}$. The same recurrence,
applied to a single row, defines train-built scores and residuals for
any row $(\mathbf{x},y)$: with $\tilde{\mathbf{x}}$, $\tilde y$ the row
standardized by the training statistics, set
$\mathbf{x}_{0}\!=\!\tilde{\mathbf{x}}$, $e_{0}\!=\!\tilde y$ and, for
$k\!=\!1,\dots,K$,
\begin{equation*}
  t_{k}\!=\!\mathbf{x}_{k-1}^{\!\top}\mathbf{w}_{k},\qquad
  \mathbf{x}_{k}\!=\!\mathbf{x}_{k-1}-t_{k}\,\mathbf{p}_{k},\qquad
  e_{k}\!=\!e_{k-1}-q_{k}\,t_{k}.
\end{equation*}
So $e_{h-1}\!=\!\tilde y-\sum_{k<h}q_{k}t_{k}$ is the residual of the
$(h\!-\!1)$-component prediction fitted on $D_{\mathcal{A}}$, and
$t_{h}$ is the score of component $h$. Once $D_{\mathcal{A}}$ is fixed,
both are fixed affine functions of $(\mathbf{x},y)$. For a fresh row
from $P$ define
\begin{equation*}
  \rho_{h}(D_{\mathcal{A}})
  \;:=\;
  \operatorname{Cor}_{P}\!\bigl(t_{h},\,e_{h-1}\bigr).
\end{equation*}

\paragraph{Hypotheses.}
For $h\!=\!1,\dots,K$,
\begin{equation*}
  H^{\mathcal{A}}_{0,h}:\ \rho_{h}(D_{\mathcal{A}})\le 0 .
\end{equation*}
The sign is set on training data: with
$\mathbf{w}_{h}=\mathbf{X}_{h-1}^{\!\top}\mathbf{y}_{h-1}/
\|\mathbf{X}_{h-1}^{\!\top}\mathbf{y}_{h-1}\|$, NIPALS gives
$q_h>0$. The null concerns the population correlation of this
train-fitted score with the remaining residual; it does not assert
that all population signal has been removed.
$H^{\mathcal{A}}_{0,1}$ is implied by the omnibus null
$y\perp\mathbf{x}$, under which $\rho_{1}\!=\!0$ for every
$D_{\mathcal{A}}$.

\paragraph{Per-step test.}
On the held-out rows $i\!\in\!\mathcal{B}$, let $r_{h}$ be the sample
correlation of the pairs $(t_{h}(\mathbf{x}_i),\,e_{h-1}(\mathbf{x}_i,y_i))$,
$T_{h}\!=\!r_{h}\sqrt{n_{2}-2}/\sqrt{1-r_{h}^{2}}$, and
$p_{h}\!=\!\Pr(t_{n_{2}-2}\ge T_{h})$. At a prespecified
$\alpha\in(0,1)$, the chain rejects
$H^{\mathcal{A}}_{0,1},\dots,H^{\mathcal{A}}_{0,k^{*}}$, where $k^{*}$
is the number of leading $p_{h}\!\le\!\alpha$ (stop at the first
$p_{h}\!>\!\alpha$).

\paragraph{Assumptions.}
\begin{enumerate}[label=(C\arabic*),leftmargin=2.4em,itemsep=2pt]
  \item The rows are iid from $P$.
  \item The partition is chosen independently of the data, with
        $n_{1}\!>\!K$ and $n_{2}\!\geq\!3$. Results are stated on the
        event that NIPALS on $D_{\mathcal{A}}$ runs $K$ steps with
        nonzero weights.
  \item $(\mathbf{x},y)\sim P$ is jointly Gaussian with positive
        definite covariance.
\end{enumerate}

\paragraph{Lemma (per-step validity given the training half).}
\emph{Under (C1)--(C3), for every $h$ and almost every $D_{\mathcal{A}}$ with
$\rho_{h}(D_{\mathcal{A}})\le 0$,
$\Pr(p_{h}\le\alpha\mid D_{\mathcal{A}})\le\alpha$, with equality at
$\rho_{h}\!=\!0$.}

\emph{Proof.} Given $D_{\mathcal{A}}$, the held-out rows are iid from
$P$ and independent of $D_{\mathcal{A}}$ (C1, C2), and $t_{h}$,
$e_{h-1}$ are fixed affine maps, so the $n_{2}$ pairs are iid
bivariate Gaussian with correlation $\rho_{h}$ (C3; positive
definiteness and the nonzero scores ensure nondegeneracy).
Condition further on the held-out scores. The Gaussian regression of
$e_{h-1}$ on $t_h$ has slope
$b=\operatorname{Cov}_{P}(t_h,e_{h-1})/\operatorname{Var}_{P}(t_h)$
and residual variance $\sigma^2>0$. The correlation statistic is
the usual slope $t$-statistic, with representation
\begin{equation*}
 T_h=\frac{Z+\delta}{\sqrt{U/(n_2-2)}},\qquad
 \delta=\frac{b}{\sigma}
       \sqrt{\sum_{i\in\mathcal B}(t_h(\mathbf x_i)-\bar t_h)^2},
 \quad Z\sim N(0,1),\quad U\sim\chi^2_{n_2-2},
\end{equation*}
where $\bar t_h$ is the held-out score mean and $Z$ and $U$ are
independent. Under the null, $b\le0$, hence
$\delta\le0$: its upper-tail rejection probability is no greater
than at $\delta=0$, where $T_h\sim t_{n_2-2}$ exactly.
Averaging over the held-out scores proves the claim.
\hfill$\square$

\begin{proposition}[Strong FWER, one split, train-only deflation]
\label{prop:fwer-one-split}
Under (C1)--(C3), let $V$ be the number of true
$H^{\mathcal A}_{0,h}$ rejected by the chain. Then, almost surely,
\begin{equation*}
  \Pr(V>0\mid D_{\mathcal A})\le\alpha,
  \qquad\text{and hence}\qquad \Pr(V>0)\le\alpha.
\end{equation*}
\end{proposition}

\begin{proof}
Condition on $D_{\mathcal A}$, which fixes the hypotheses and their
order. If none is true, $V=0$. Otherwise let $h_0$ index the first
true null. Because the chain stops at the first non-rejection, any
false rejection requires $p_{h_0}\le\alpha$. By the Lemma,
\begin{equation*}
 \Pr(V>0\mid D_{\mathcal A})
 \le \Pr(p_{h_0}\le\alpha\mid D_{\mathcal A})
 \le\alpha.
\end{equation*}
This bound needs no independence between tests and does not condition
on earlier rejections. Averaging over $D_{\mathcal A}$ gives the
unconditional bound: the fixed-sequence
principle~\citep{westfall2001optimally,goeman2010sequential} applied
conditionally on training.
\end{proof}

\paragraph{Without Gaussianity.}
For the different null $t_h\perp e_{h-1}$ under $P$, a permutation
test that permutes the held-out residuals has finite-sample validity
conditional on $D_{\mathcal A}$ under (C1)--(C2): independence makes
their pairing with the scores exchangeable. Using all permutations,
or the usual $(1+\#\{\text{permuted statistics}\ge\text{observed}\})/(B+1)$
formula for $B$ uniform random permutations, gives super-uniform
$p$-values, and the same fixed-sequence proof applies.
This changes the inferential target: without Gaussianity, the
permutation test does not in general test $\rho_h\le0$, and rejection
of independence need not imply positive correlation.

\paragraph{What a rejection says.}
Rejection provides evidence of positive population correlation with
the remaining residual. To connect this to prediction, define
$R_h(c)=\min_a\mathbb E_P[(e_{h-1}-a-c\,t_h)^2]$. Then
\begin{equation*}
 R_h(c)-R_h(0)
 =c^2\operatorname{Var}_P(t_h)
  -2c\operatorname{Cov}_P(t_h,e_{h-1}).
\end{equation*}
Thus positive correlation means that some positive coefficient
reduces population squared error, with the intercept reoptimized.
It does not guarantee improvement at the particular training
coefficient $q_h$. Nor does the rejection count estimate the
population signal rank $\widetilde r$: even after $\widetilde r$
fitted components, residual signal can remain and make a later
rejection correct (\S\ref{sec:methods:plsK}).

\paragraph{What is not covered.}
(i)~The $J$-split NB average of Eq.~\eqref{eq:nb}. Each split has its
own fitted components and so its own null, and the held-out rows of one
split are training rows of the others. Conditioning on every training
half conditions on essentially the whole sample, which leaves no fresh
data for the argument above. (ii)~Full-data deflation. The first $h\!-\!1$
components are then fitted on rows that are later held out, so the null
at step $h$ depends on the held-out half and the Lemma does not apply.
(iii)~The components reported in the paper are fitted on all $n$ rows;
the proposition concerns components fitted on $n_{1}$ rows.

\paragraph{Warriner valence.}
\texttt{inference/s7\_single\_split\_chain.py} runs this procedure once,
on one split drawn from a fixed seed, on the data of the per-component
example in \S\ref{sec:intro:contribution} (Warriner EN GloVe-300
valence, $n\!=\!13{,}365$, $n_{1}\!=\!6{,}682$, $n_{2}\!=\!6{,}683$,
$K\!=\!5$). The held-out correlations at steps 1--5 are $0.741$,
$0.341$, $0.205$, $0.108$, and $0.057$. The Gaussian $p_{h}$ are below
$10^{-300}$, $2\times10^{-182}$, $2\times10^{-64}$, $4\times10^{-19}$,
and $1\times10^{-6}$; the permutation $p_{h}$ ($B\!=\!9{,}999$) all sit
at their floor, $10^{-4}$. The chain keeps all five components under
either test at $\alpha\!=\!.05$. (C3) does not hold for these data; the
permutation form needs only (C1)--(C2). Drawing several splits and
reporting the best would void the guarantee, so only this one was run.

\subsection{Per-axis FWER: empirical diagnostics}
\label{app:peraxis-fwer}

\paragraph{Boundary probes (S6).}
The body table (Tab.~\ref{tab:fwer-boundary}, \S\ref{sec:exp:fwer})
comes from \texttt{inference/s6\_fwer\_boundary.py}. Setup, pinned:
$n\!=\!40$, $J\!=\!50$ splits, total planted $R^{2}\!=\!0.5$ split
equally over $m\!\in\!\{1,2,3\}$ axes, planted on right-singular
vectors of the column-standardized $\mathbf{X}$, 5{,}000 replicates per cell,
permutation budget 999. Designs: iid Gaussian ($40\!\times\!20$); one
common factor plus $0.3$ idiosyncratic noise ($p\!=\!20$, the
high-collinearity case); $p\!=\!60\!>\!n$ with eigenvalues decaying as
$j^{-2}$ rotated into the columns; the real $60\!\times\!401$ NIR
gasoline spectra, rows subsampled to 40. The three arms run on the
same generated datasets: full-data deflation with the NB $t$ reference; full-data deflation with the no-refit
permutation reference applied naively at the deflated step (not exact
there); and train-only deflation, where within each split the earlier
components are fitted on the training half and both halves are
deflated by that fit, with held-out outcomes standardized by the
training outcome's mean and standard deviation. The train-only arm is
an NB aggregate, so App.~\ref{app:fwer-proof} does not prove its
FWER control.
Planting $m$ axes does not make step $m\!+\!1$ a true conditional
null: the fitted deflation uses noisy outcomes and can leave residual
signal. We report the frequency of rejecting through that step;
this is the false-rejection event when that step is the first true null.
At one deflation the chain reaches the
boundary step in at least $99.8\%$ of replicates at $\alpha\!=\!.05$
($97.5\%$ at $\alpha\!=\!.01$) in every design. At two and three
deflations it reaches it at $\alpha\!=\!.05$ in $81\%$, $60\%$, and
$88\%$ of replicates
(dominant factor, decaying spectrum, gasoline at $m\!=\!2$), in
$0\%$, $5.2\%$, and $36\%$ at $m\!=\!3$, and in at most one
replicate in 5{,}000 in the iid design, so the iid cells at both depths
and the dominant-factor cell at three deflations carry no information
about the boundary. The
$\alpha\!=\!.01$ column and the per-arm breakdown are in
\texttt{s6\_fwer\_boundary/csvs/}; the per-cell summary, including how
often each boundary step is reached, is in
\texttt{s6\_fwer\_boundary/report.log}.

\paragraph{Planted-signal grids (S5, T4).}
These grids likewise measure rejection beyond a planted component
count. FWER would follow from valid tests at the true nulls, but
neither NB aggregation nor full-data deflation is covered by the
single-split proof (App.~\ref{app:fwer-proof}). Across
16 cells (Warriner-EN GloVe-300 valence and Tecator NIR;
$m\!\in\!\{0,1,2,3\}$ planted axes $\times\,R^{2}\!\in\!\{0.10,0.25\}$,
400 reps each), chain rejection rates show no excess over nominal
resolved by the Wilson 95\% intervals. At $m\!=\!0$ this measures
FWER under the global null (KS-uniform $p\!\ge\!0.36$ on GloVe);
at $m\!\ge\!1$ it is a diagnostic of the fitted sequence.

S5 (Warriner-EN GloVe-300, valence) and T4 (Tecator NIR) extend the
$H_0$-only calibration grids of \S\ref{app:powercalib-full} and
\S\ref{app:nontext} to planted-signal settings: $m$ signal axes planted via
top-$m$ SVD of $\mathbf{X}$ (with $m\!\in\!\{0,1,2,3\}$ and planted
$R^{2}\!\in\!\{0.10, 0.25\}$), then per-axis NB run on full-data-deflated
NIPALS scores at $K_{\mathrm{test}}\!=\!m\!+\!1$. The fixed-sequence
rejection event through step $m\!+\!1$ is
$\bigcap_{k=1}^{m+1}\{p_{k}\!\le\!\alpha\}$; the chain rate is its
empirical frequency. The marginal rate at step $m\!+\!1$ records
rejection whether or not earlier tests pass, and the minimum prior
rejection rate is the smallest rate among steps $1,\dots,m$.
These separate failure to reach a step from failure to reject there.
Per-rep $p$-values
in the long-format CSVs alongside each summary
(\texttt{inference/results/\{s5,t4\}\_peraxis\_fwer/}).

\begin{table}[H]
\centering
\caption{S5 --- marginal and chain NB rejection rates beyond $m$
planted axes on Warriner-EN GloVe-300 valence, 400 reps per cell, Wilson 95\%
CIs. M0 cells (full null) report all three $\alpha$ levels; m$\geq$1
cells report $\alpha\!=\!0.05$ only (full $\alpha$-grid in
\texttt{s5\_summary.csv}).}
\label{tab:s5-peraxis-fwer}
\footnotesize
\setlength{\tabcolsep}{3pt}
\begin{tabular}{@{}lrrrrrrr@{}}
\toprule
\textbf{cell} & $m$ & $R^{2}$ & $\alpha$ & \textbf{step $m+1$ rate (95\% CI)} & \textbf{KS-uniform $p$} & \textbf{chain rate (95\% CI)} & \textbf{min prior rate} \\
\midrule
M0R10 & 0 & 0.10 & 0.01 & 0.003 [.000, .014] & 0.82 & 0.003 [.000, .014] & --- \\
M0R10 & 0 & 0.10 & 0.05 & 0.035 [.021, .058] & 0.82 & 0.035 [.021, .058] & --- \\
M0R10 & 0 & 0.10 & 0.10 & 0.085 [.061, .116] & 0.82 & 0.085 [.061, .116] & --- \\
M0R25 & 0 & 0.25 & 0.01 & 0.015 [.007, .032] & 0.36 & 0.015 [.007, .032] & --- \\
M0R25 & 0 & 0.25 & 0.05 & 0.055 [.037, .082] & 0.36 & 0.055 [.037, .082] & --- \\
M0R25 & 0 & 0.25 & 0.10 & 0.128 [.098, .164] & 0.36 & 0.128 [.098, .164] & --- \\
\midrule
M1R10 & 1 & 0.10 & 0.05 & 0.008 [.003, .022] & --- & 0.008 [.003, .022] & 1.00 \\
M1R25 & 1 & 0.25 & 0.05 & 0.000 [.000, .010] & --- & 0.000 [.000, .010] & 1.00 \\
M2R10 & 2 & 0.10 & 0.05 & 0.000 [.000, .010] & --- & 0.000 [.000, .010] & 0.013 \\
M2R25 & 2 & 0.25 & 0.05 & 0.000 [.000, .010] & --- & 0.000 [.000, .010] & 0.000 \\
M3R10 & 3 & 0.10 & 0.05 & 0.000 [.000, .010] & --- & 0.000 [.000, .010] & 0.000 \\
M3R25 & 3 & 0.25 & 0.05 & 0.000 [.000, .010] & --- & 0.000 [.000, .010] & 0.000 \\
\bottomrule
\end{tabular}
\end{table}

\begin{table}[H]
\centering
\caption{T4 --- marginal and chain NB rejection rates beyond $m$
planted axes on Tecator NIR, 400 reps per cell, Wilson 95\% CIs (full
$\alpha$-grid in \texttt{t4\_summary.csv}).}
\label{tab:t4-peraxis-fwer}
\footnotesize
\setlength{\tabcolsep}{3pt}
\begin{tabular}{@{}lrrrrrrr@{}}
\toprule
\textbf{cell} & $m$ & $R^{2}$ & $\alpha$ & \textbf{step $m+1$ rate (95\% CI)} & \textbf{KS-uniform $p$} & \textbf{chain rate (95\% CI)} & \textbf{min prior rate} \\
\midrule
TM0R10 & 0 & 0.10 & 0.01 & 0.015 [.007, .032] & $\sim\!10^{-9}$ & 0.015 [.007, .032] & --- \\
TM0R10 & 0 & 0.10 & 0.05 & 0.033 [.019, .055] & $\sim\!10^{-9}$ & 0.033 [.019, .055] & --- \\
TM0R10 & 0 & 0.10 & 0.10 & 0.060 [.041, .088] & $\sim\!10^{-9}$ & 0.060 [.041, .088] & --- \\
TM0R25 & 0 & 0.25 & 0.01 & 0.013 [.005, .029] & $\sim\!10^{-4}$ & 0.013 [.005, .029] & --- \\
TM0R25 & 0 & 0.25 & 0.05 & 0.030 [.017, .052] & $\sim\!10^{-4}$ & 0.030 [.017, .052] & --- \\
TM0R25 & 0 & 0.25 & 0.10 & 0.068 [.047, .096] & $\sim\!10^{-4}$ & 0.068 [.047, .096] & --- \\
\midrule
TM1R10 & 1 & 0.10 & 0.05 & 0.068 [.047, .096] & 0.014 & 0.068 [.047, .096] & 0.98 \\
TM1R25 & 1 & 0.25 & 0.05 & 0.058 [.039, .085] & 0.13  & 0.058 [.039, .085] & 1.00 \\
TM2R10 & 2 & 0.10 & 0.05 & 0.020 [.010, .039] & --- & 0.015 [.007, .032] & 0.78 \\
TM2R25 & 2 & 0.25 & 0.05 & 0.008 [.003, .022] & --- & 0.008 [.003, .022] & 1.00 \\
TM3R10 & 3 & 0.10 & 0.05 & 0.000 [.000, .010] & --- & 0.000 [.000, .010] & 0.25 \\
TM3R25 & 3 & 0.25 & 0.05 & 0.005 [.001, .018] & --- & 0.003 [.000, .014] & 0.82 \\
\bottomrule
\end{tabular}
\end{table}

The low chain rates at $m\ge1$ partly reflect the chance of stopping
before step $m+1$; low chain rates alone do not establish validity
of the marginal test at that step. On the $m=0$ cells, the Tecator-side small
KS-uniform $p$ values restate the distinction between rejection rates
at the tested levels and full-shape uniformity documented in
\S\ref{sec:exp:calibration}.

\subsection{Beyond PLS: supervised PCA and a ridge probe}
\label{app:beyond-pls}

\paragraph{Protocol} (\texttt{inference/r6\_pipelines.py}).
Each replicate draws $n$ rows of Warriner GloVe-300 (valence) or
Tecator NIR. Under $H_{0}$ the real outcome is permuted (2{,}000
replicates per cell); under $H_{1}$ noise is injected into the real
outcome at $\theta\!\in\!\{0.3, 0.7\}$ (400 replicates). Three
direction learners, each fitted on the training half only: supervised
PCA~\citep{bair2006spca}, which keeps the 10\% of columns with the
largest $|\operatorname{cor}(\mathbf{x}_{j},\mathbf{y})|$ on the training
half and takes their leading principal component; a ridge probe
($\lambda\!=\!1$ on the standardized training half); and PLS1 as the
reference. The rest is the paper's protocol: OLS refit of the training
scores, held-out correlation, Fisher-$z$, mean over $J\!=\!50$ splits
drawn once and held fixed. Three tests on each: the exact
NB-permutation of \S\ref{sec:methods:exactperm} ($B\!=\!999$
permutations of $\mathbf{y}$, the direction relearned on the permuted
training half at every draw, so supervised PCA's screen pays the full
refits while PLS1 and ridge, whose direction is linear in $\mathbf{y}$,
take the no-refit form); perm-$Q^{2}$ (5-fold, direction learned inside
each training fold, 999 permutations); and NB-asymptotic on the same
$z$ values, run on Warriner only because Tecator's stable rank of
$1.01$ is the regime the rule of \S\ref{sec:methods:validity} sends to
the exact test. All arms see the same subsample, the split-half arms
share the splits, and the $Q^{2}$ arms share the folds, so every
comparison is paired.

\begin{table}[H]
\centering
\caption{R6 --- the held-out-refit machinery outside PLS. Power at
$\theta\!=\!0.7$ (400 replicates) and rejection rate under $H_{0}$ at
$\alpha\!=\!.05$ (2{,}000 replicates; Wilson half-width $\pm 0.010$)
for the exact NB-permutation test, perm-$Q^{2}$, and NB-asymptotic;
KS is the KS-uniform $p$ of the exact test's null $p$-values.
NB-asymptotic is Warriner-only (see text). At $n\!=\!200/172$ the exact
test and NB reach $\approx\!1.00$ power everywhere and only
perm-$Q^{2}$ trails.}
\label{tab:beyond-pls}
\footnotesize
\setlength{\tabcolsep}{4pt}
\begin{tabular}{@{}llr rrr rrr r@{}}
\toprule
 & & & \multicolumn{3}{c}{\textbf{power, $\theta\!=\!0.7$}} & \multicolumn{3}{c}{\textbf{level, $\alpha\!=\!.05$}} & \\
\cmidrule(lr){4-6}\cmidrule(lr){7-9}
\textbf{Data} & \textbf{Pipeline} & $n$ & \textbf{exact} & \textbf{perm-$Q^{2}$} & \textbf{NB} & \textbf{exact} & \textbf{perm-$Q^{2}$} & \textbf{NB} & \textbf{KS (exact)} \\
\midrule
Warriner & sPCA  &  40 & 0.41 & 0.28 & 0.19 & 0.050 & 0.049 & 0.010 & 0.85 \\
Warriner & ridge &  40 & 0.43 & 0.28 & 0.37 & 0.051 & 0.058 & 0.040 & 0.53 \\
Warriner & PLS1  &  40 & 0.46 & 0.34 & 0.41 & 0.048 & 0.054 & 0.044 & 0.31 \\
Warriner & sPCA  & 200 & 1.00 & 0.98 & 1.00 & 0.055 & 0.049 & 0.015 & 0.40 \\
Warriner & ridge & 200 & 1.00 & 0.70 & 0.98 & 0.051 & 0.052 & 0.016 & 0.53 \\
Warriner & PLS1  & 200 & 1.00 & 1.00 & 1.00 & 0.046 & 0.055 & 0.053 & 0.37 \\
\midrule
Tecator  & sPCA  &  40 & 0.63 & 0.40 & --- & 0.052 & 0.050 & --- & 0.11 \\
Tecator  & ridge &  40 & 0.995 & 0.89 & --- & 0.055 & 0.056 & --- & 0.076 \\
Tecator  & PLS1  &  40 & 0.54 & 0.33 & --- & 0.052 & 0.047 & --- & 0.13 \\
Tecator  & sPCA  & 172 & 1.00 & 0.97 & --- & 0.049 & 0.050 & --- & 0.60 \\
Tecator  & ridge & 172 & 1.00 & 1.00 & --- & 0.046 & 0.047 & --- & 0.79 \\
Tecator  & PLS1  & 172 & 0.995 & 0.92 & --- & 0.050 & 0.051 & --- & 0.34 \\
\bottomrule
\end{tabular}
\end{table}

\paragraph{Reading.}
The exact test holds level in all eight non-PLS $H_{0}$ cells at
$\alpha\!=\!.05$ and $.10$ (Wilson intervals cover nominal; KS never
below $0.076$); at $\alpha\!=\!.01$ seven of the eight are within the
Wilson interval and Tecator supervised PCA at $n\!=\!40$ sits at
$0.017$ $[0.012, 0.023]$. Its power is at or above perm-$Q^{2}$'s in
all twelve cells at both $\theta$. The NB correction, calibrated on
PLS1 ($0.044$ and $0.053$), turns conservative off it: supervised PCA
at $0.010$ and $0.015$, ridge at $0.040$ and $0.016$, with the power
loss that implies ($0.19$ against $0.41$ for supervised PCA at
$n\!=\!40$). Two controls from the same draws: the textbook one-sample
$t$ on the same 50 Fisher-$z$ values rejects 27--40\% of nulls at
$\alpha\!=\!.05$ in every cell, and supervised PCA with its screen fixed
once on all rows before folding and permuting (the selection leak of
screening on the full data and then permuting inside the reduced
matrix) rejects 99.6\% and 92\% of Warriner nulls under perm-$Q^{2}$
($n\!=\!40$, $200$). The full grid, including $\theta\!=\!0.3$ and the
$\alpha\!=\!.01$ and $.10$ levels, is in
\texttt{inference/results/r6\_pipelines/csvs/cells.csv}.

\subsection{Embedding pipelines and vocabulary cleaning}
\label{app:embeddings}

All four configurations share the document-vector recipe of
\S\ref{sec:methods:representations}: mean-pool token embeddings,
$\ell_2$-normalize rows of $\mathbf{E}$ before pooling, and intersect
the pre-pooling vocabulary with the rated lexicon (Warriner-EN,
ANPW-PL, or Stadthagen-Gonz\'alez-ES). Per-config specifics:

\paragraph{Config 1 --- EN GloVe-300.}
\citet{warriner2013norms}, 13{,}915 word types. We use the public
GloVe-300 vectors~\citep{pennington2014glove} trained on Common Crawl
(42B tokens; Stanford NLP \texttt{glove.42B.300d}), apply
$\ell_2$-normalization row-wise, intersect with the Warriner type
list, and drop entries with no vector. Effective $n\!=\!13{,}365$.

\paragraph{Config 2 --- PL GloVe-800.}
\citet{imbir2016anpw}, 4{,}905 ANPW types with V/A/D ratings. The
embedding is the Polish GloVe-800 release distributed via
\citet{dadas2019polishnlp} (trained on Polish Wikipedia, books, and
articles, $\sim$1.5B tokens), $\ell_2$-normalized. After intersection
with non-missing ANPW ratings, $n\!=\!4{,}526$.

\paragraph{Config 3 --- ES GloVe-300 (SBWCE).}
\citet{stadthagen2017spanishnorms}, 14{,}031 Spanish word types with
valence and arousal ratings (no rated dominance). The embedding is
the GloVe-300 release distributed with the Spanish Billion Words
Corpus and Embeddings (SBWCE; \citealp{cardellino2019sbwce}), trained
on $\sim$1.5B tokens of cleaned Spanish text from Wikipedia,
ParaCrawl, OpenSubtitles, and other sources. We $\ell_2$-normalize
the embedding rows, intersect with the rated lexicon, and drop
entries with no vector. Effective $n\!=\!12{,}639$.

\paragraph{Config 4 --- EN XLM-R type-level.}
Type-vectors built from \texttt{xlm-roberta-base}~\citep{conneau2020xlmr}
contextual embeddings: for each Warriner type we collected up to 20
sentences from the AllenAI C4 corpus~\citep{raffel2020t5} (EN,
streamed; sentence length filter
$30 \leq \ell \leq 300$ characters), located the sub-word span via
\texttt{offset\_mapping}, and mean-pooled the BPE token vectors
within the span. The 20 contextual vectors are then averaged into a
single per-word type vector. Post-processing: $\ell_2$-normalize,
followed by rank-1 ABTT (all-but-the-top dominant-direction removal,
\citealp{mu2018allbutthetop}) to mitigate the XLM-R anisotropy documented by
\citet{ethayarajh2019contextual}; the same normalize-then-ABTT step is applied to
the GloVe-300 baseline so cross-embedding comparisons in
\S\ref{sec:exp:vadfit} differ only in the underlying type vectors.
Coverage: 13{,}719 of 13{,}915 Warriner types had at least one C4
sentence (98.6\%); after intersection with non-missing ratings and
the 20-sentence-coverage floor, effective $n\!=\!13{,}540$.

\paragraph{Varimax loading-side target.}
The varimax target $\mathcal{V}_{\!\mathrm{tgt}}$
(\S\ref{sec:methods:rotation}) for the v1\_multipls\_vad run is the
full embedding vocabulary intersected with the rated lexicon, no
external POS filter applied. Concretely,
\textpkg's main fit interface (\texttt{fit\_multipls(rotate="varimax")}) uses the standardized
embedding matrix
$\mathbf{L}\!=\!(\mathbf{E}-\boldsymbol\mu_X)/\boldsymbol\sigma_X\!\cdot\!\mathbf{W}$
on the full vocabulary as the loading target, so the target is fixed
by the data path of \S\ref{sec:methods:representations} (rated lexicon
$\cap$ embedding vocabulary), not by a separate curated word list.
Source files for the script: \texttt{applied/v1\_multipls\_vad.py}; for
the rotation backend in \textpkg{}: \texttt{backends/pls.py::mpls\_fit}.

\subsection{Sparse-PLS --- per-\texorpdfstring{$(K, s)$}{(K, s)} sweep on Warriner-EN valence}
\label{app:sparse-pls}

\begin{table}[H]
\centering
\caption{Hard top-$s$ sparse-PLS sweep on Warriner-EN valence
(Config~1, GloVe-300). Each row is an independent sparse-PLS fit;
$r^2$ values are in-sample on the standardized response and
represent the intrinsic fit at the given cardinality. The
$s\!=\!D\!=\!300$ row coincides with NIPALS PLS1 by construction
(it is the natural endpoint of the cardinality grid, not a
comparison baseline). Per-cell top-15 word leaderboards on
$\mathbf{w}_1$ in Tab.~\ref{tab:sparse-pls-words}.}
\label{tab:sparse-pls-sweep}
\small
\begin{tabular}{@{}lrrrr@{}}
\toprule
\textbf{Cell} & $K$ & $s_k$ & \textbf{Total $\mathbf{r}^2$} & \textbf{Unique features used} \\
\midrule
dense ($s\!=\!D$, $\equiv$ NIPALS PLS1) & 3 & 300 & 0.631 & 300 \\
\midrule
cardinality sweep, $K\!=\!3$            & 3 & 200 & 0.630 & 287 \\
                                        & 3 & 100 & 0.621 & 218 \\
                                        & 3 &  50 & 0.605 & 133 \\
                                        & 3 &  25 & 0.575 &  71 \\
\midrule
more axes, sparser each                 & 4 & 100 & 0.634 & 250 \\
                                        & 5 &  50 & 0.634 & 194 \\
\bottomrule
\end{tabular}
\end{table}

We report a per-$(K, s)$ sweep on Warriner-EN valence
(Tab.~\ref{tab:sparse-pls-sweep}). The motivation for sparse-PLS in
this work is \emph{not} to recover the dense PLS1 fit at lower
cardinality, but to make a higher number of components individually
readable: each axis is literally its $\le\!s_k$ supporting features,
so increasing $K$ produces more interpretable directions rather than
dense mixtures. Tab.~\ref{tab:sparse-pls-words} demonstrates this on
the $K\!=\!5, s\!=\!50$ cell --- five qualitatively distinct readable
axes coexist in a single fit, each defined by 50 features. The $r^2$
values in Tab.~\ref{tab:sparse-pls-sweep} are intrinsic sparse-PLS
performance per cardinality; consistent with the paper's overall
presentation policy, we make no cross-model $R^2$ comparison with
the dense PLS1 fits in \S\ref{sec:exp:vadfit} / \S\ref{sec:exp:peraxis-glove}.
The dense ($s\!=\!D\!=\!300$) row is the natural endpoint of the
cardinality grid (sparse-PLS reduces to NIPALS PLS1 there) and is
included for that reason, not as a baseline against which the
sparser cells are measured.

\begin{table}[H]
  \centering
  \scriptsize
  \setlength{\tabcolsep}{3pt}
  \caption{Top-5 and bottom-5 words on each of five sparse axes
  for the $K\!=\!5, s\!=\!50$ cell on Warriner-EN valence (see
  Tab.~\ref{tab:sparse-pls-sweep}). Each axis is supported on $\le 50$
  features; the columns are read directly off $\mathbf{w}_k$, not off
  a varimax-rotated loading. Dim-5 negative-end \textless{}slur\textgreater{}
  placeholders mask racial, homophobic, or strong-expletive vocabulary
  present in the Warriner ratings.
 }
  \label{tab:sparse-pls-words}
  \begin{tabular}{rl ccccc}
    \toprule
    & & \textbf{dim-1} & \textbf{dim-2} & \textbf{dim-3} & \textbf{dim-4} & \textbf{dim-5} \\
    Rank & Side & \emph{R²=0.506} & \emph{R²=0.064} & \emph{R²=0.036} & \emph{R²=0.018} & \emph{R²=0.010} \\
    \midrule
    1 & pos & wonderful   & easygoing     & significant  & madden          & localized   \\
    2 & pos & showcase    & affectionate  & growth       & shovel          & eminence    \\
    3 & pos & superb      & togetherness  & stimulation  & sequel          & reinstate   \\
    4 & pos & fabulous    & goofball      & facilitate   & donkey          & loveliness  \\
    5 & pos & fantastic   & spunky        & employment   & jolly           & sovereign   \\
    \midrule
    1 & neg & unjust      & industrial    & pompous      & disorientation  & suck                          \\
    2 & neg & senseless   & carbon        & pretentious  & negligence      & \textless{}slur\textgreater{} \\
    3 & neg & desertion   & heavy metal   & snooty       & carelessness    & \textless{}slur\textgreater{} \\
    4 & neg & inhumane    & steel         & ornery       & decedent        & \textless{}slur\textgreater{} \\
    5 & neg & unprovoked  & machinery     & wannabe      & disturbance     & \textless{}slur\textgreater{} \\
    \bottomrule
  \end{tabular}
\end{table}

Per-cell top-15 word leaderboards live alongside the sweep CSV at
\texttt{applied/results/v3\_sparse\_pls\_valence/sparse\_pls\_axes\_K\{K\}\_s\{s\}.md}.

\subsection{Broader impacts}
\label{app:impacts}

The supervised text pipeline is a small statistical-inference layer
on top of pre-existing public embeddings: it runs on CPU, requires no
LLM in the inference path, and returns word-readable axes.
Its broader effect is on \emph{access} rather
than capability. Corpus-scale text analysis currently favours groups
that can run LLM-scale pipelines; the supervised text pipeline is a
CPU-only alternative for groups without that compute. It does not change what
well-resourced actors can do, since they already have stronger
machinery. The standing risk is the one shared by any text-on-outcome
tool: the recovered axes are correlational and can surface or amplify
spurious associations if treated as causal. We document this in the
body's Limitations subsection (\S\ref{sec:disc:limitations}) and in
the package documentation; the released software (\textpkg{},
GPL-3.0-or-later) is a statistical library, not a generative model or
scraped dataset, so it requires no additional release-time safeguards
(checklist item~11).

\subsection{Reproducibility}
\label{app:repro}

\paragraph{Code.}
The pipeline is packaged as the \textpkg{} library on PyPI
(\texttt{pip install} \textpkg). The replication scripts that produced
every number in this paper, except the wall-clock timings, live in
\repoanchor: \texttt{applied/} (Configs 1--4), \texttt{toys/}
(Toys 1--5), \texttt{inference/} (the S1--S7 and T1--T4 calibration,
power, and FWER benchmarks; the R1, R2 and R4 null grids behind the
validity rule; R6 outside-PLS pipelines), and \texttt{embeddings.md}
(sources of the embeddings and the type-vector construction for
Config 4).

\paragraph{Public assets.}
Embeddings: GloVe-300 Common Crawl 42B~\citep{pennington2014glove} for
Config~1, Polish GloVe-800~\citep{dadas2019polishnlp} for Config~2,
\texttt{xlm-roberta-base} (Hugging Face) over AllenAI C4 (EN,
streamed) for Config~4. Rated lexicons: Warriner et al.\ (2013) for
EN, ANPW \citep{imbir2016anpw} for PL.

\paragraph{Seeds.}
Configs 1--4 fits use a single global seed (\texttt{seed}=2137) for
the NB split-half schedule; results are insensitive to the seed
choice as the NB statistic averages over $J\!=\!50$ splits. Toys
1--5 and S1/S2 use per-cell seed schedules: Toy~1 50 seeds per cell
$\times$ 4 cells $\times$ 7 $R^2$ values; Toy~2 50 seeds per cell
$\times$ 8 cells; Toy~3 20 seeds per arm; Toy~4 50 random
orthogonal rotations on a single fit; Toy~5 with 5 seeds per cell
$\times$ 8 cells; S1 and S2 400 reps per cell.

\paragraph{Compute.}
End-to-end reproduction runs comfortably on a single workstation
CPU; the heaviest single step is the C4 streaming pass for Config~4
($\approx$10 hours single-machine for 13{,}915 words $\times$ up to
20 sentences). The full toys + S1/S2 + Configs 1--3 sweep completes
in $<\!2$ hours. No GPU is used in the reported numbers; XLM-R
encoding for Config~4 used a single consumer GPU at FP16 inference.

\subsection{Algorithm: PLS-\texorpdfstring{$K$}{K} pseudocode for the supervised text pipeline}
\label{app:algorithm}

\begin{algorithm}[h]
\caption{PLS-$K$ for the supervised text pipeline, with loading-side
varimax and NB split-half inference. Inputs: mean-centered document
matrix \texttt{X} ($n\!\times\!d$); mean-centered outcome \texttt{y}
($n$); embedding matrix \texttt{E} ($V\!\times\!d$, $\ell_{2}$-normed
rows); target vocabulary index set \texttt{V\_tgt}; preselected
component count \texttt{K} ($K^{*}$ in the text);
number of split-half resamples \texttt{J}; level $\alpha$.}
\label{alg:textpipeplsk}
\begin{lstlisting}[style=pypseudo]
def fit(X, y, E, V_tgt, K, J, alpha=0.05):
    n  = X.shape[0]
    n1 = n2 = n // 2

    # ---- PLS-K via NIPALS deflation (Eq. 4) ----
    W, T, P, Q = nipals_pls_k(X, y, K)   # W,P: (d,K); T: (n,K); Q: (K,)

    # ---- Per-axis NB split-half inference (Eq. 3) on deflated pairs ----
    # For axis h, deflate (X, y) by the leading h-1 full-data NIPALS
    # components, then run NB at K=1 on the deflated pair. Matches
    # inference library: p_for_incremental.
    p = np.empty(K)
    for h in range(1, K + 1):
        if h == 1:
            X_h, y_h = X, y
        else:
            T_h, P_h, Q_h = nipals_pls_k(X, y, h - 1)[1:]  # full-data deflation
            X_h = X - T_h @ P_h.T
            y_h = y - T_h @ Q_h
        z = np.empty(J)
        for j in range(J):
            tr, te = random_partition(n, n1, n2)
            w_j     = nipals_pls_k(X_h[tr], y_h[tr], 1)[0]   # K=1 weight
            t_tr    = X_h[tr] @ w_j
            beta_j  = (t_tr @ y_h[tr]) / (t_tr @ t_tr)
            t_te    = X_h[te] @ w_j
            r       = pearson(beta_j * t_te, y_h[te])
            z[j]    = np.arctanh(r)
        t_NB_h = z.mean() / (z.std(ddof=1) * np.sqrt(1/J + n2/n1))
        p[h-1] = t_sf(t_NB_h, df=J - 1)  # one-sided upper tail

    # ---- Fixed-sequence FWER on canonical NIPALS order ----
    # (Westfall 2001; sequential rejection of Goeman & Solari 2010
    # requires valid per-step tests for the ordered family.)
    k_star          = number_of_leading_rejections(p, alpha)
    alpha_star_FWER = p.max()  # summary for the preselected K components

    # ---- Loading-side varimax: post-inferential rebasing ----
    # Within-Krylov; held-out predictions and alpha_star_FWER unchanged.
    L      = E[V_tgt] @ W               # (|V_tgt|, K)
    R_star = varimax(L)                 # arg max over O(K), Eq. 5
    return (W @ R_star,                 # W*: rotated weights
            L @ R_star,                 # L*: rotated word loadings
            p, alpha_star_FWER)
\end{lstlisting}
\end{algorithm}


\end{document}